\documentclass[11pt]{article}
\usepackage[T1]{fontenc}         
\usepackage{amsfonts}      
\usepackage{nicefrac}      
\usepackage{microtype}   
\usepackage[colorlinks,
            linkcolor=red,
            anchorcolor=blue,
            citecolor=blue, 
            pagebackref=true
           ]{hyperref}

\usepackage{fullpage}
\usepackage{tabularx}
\usepackage{float}
\usepackage{enumitem}
\usepackage{wrapfig}

\usepackage{mmll}
\usepackage[colorinlistoftodos, textsize=tiny, textwidth=2.2cm, backgroundcolor=blue!10, linecolor=magenta, bordercolor=magenta]{todonotes}
\usepackage{enumitem}

\crefname{assumption}{assumption}{assumptions}
\usepackage{selectp}
\usepackage{tikz}
\usetikzlibrary{arrows,automata,positioning}

\newcommand{\inner}[1]{\left\langle #1\right\rangle}
\newcommand{\innerA}[1]{\left\langle #1\right\rangle_{\cA}}
\newcommand{\sahk}{(s_h^k,a_h^k)}
\newcommand{\infPh}[1]{\inf_{P_{#1}(\cdot|s,a)\in \cU_{#1}^\rho(s,a;\bmu_{#1}^0)}}

\newcommand{\infPP}{\inf_{P\in \cU^\rho(P^0)}}
\newcommand{\vA}{\text{vol}(\cA)}
\newcommand{\piref}{\pi^{\text{ref}}}
\newcommand{\alg}{DR-RPO }
\newcommand{\algo}{DR-RPO}
 
\DeclareMathOperator*{\arginf}{arg\,inf}

\allowdisplaybreaks
\begin{document}

\title{Policy Regularized Distributionally Robust Markov Decision Processes with Linear Function Approximation}

\author{
Jingwen Gu\thanks{Cornell University}~\footnotemark[4] \qquad
Yiting He\thanks{University of Science and Technology of China}~\footnotemark[4] \qquad
Zhishuai Liu\thanks{Duke University}~\thanks{Equal contribution} \qquad
Pan Xu\footnotemark[3]~\thanks{Correspondence to \texttt{pan.xu@duke.edu}} 
}

\maketitle

\begin{abstract}
Decision-making under distribution shift is a central challenge in reinforcement learning (RL), where training and deployment environments differ. We study this problem through the lens of robust Markov decision processes (RMDPs), which optimize performance against adversarial transition dynamics. Our focus is the online setting, where the agent has only limited interaction with the environment, making sample efficiency and exploration especially critical. Policy optimization, despite its success in standard RL, remains theoretically and empirically underexplored in robust RL. To bridge this gap, we propose \textbf{D}istributionally \textbf{R}obust \textbf{R}egularized \textbf{P}olicy \textbf{O}ptimization algorithm (DR-RPO), a model-free online policy optimization method that learns robust policies with sublinear regret. To enable tractable optimization within the softmax policy class, DR-RPO incorporates reference-policy regularization, yielding RMDP variants that are doubly constrained in both transitions and policies. To scale to large state-action spaces, we adopt the $d$-rectangular linear MDP formulation and combine linear function approximation with an upper confidence bonus for optimistic exploration. We provide theoretical guarantees showing that policy optimization can achieve polynomial suboptimality bounds and sample efficiency in robust RL, matching the performance of value-based approaches. Finally, empirical results across diverse domains corroborate our theory and demonstrate the robustness of DR-RPO.

\end{abstract}

\section{INTRODUCTION}
The deployment of reinforcement learning (RL) \citep{Sutton1998} in real-world systems faces a critical challenge: the lack of knowledge about the transition kernel in the true environment. To address this problem, off-dynamics RL \citep{koos2012transferability,eysenbach2021offdynamicsreinforcementlearningtraining, wang2024return} formulates a paradigm which trains the policy in a source domain and then deploy it to the target domain. The source domain can be chosen to be a simulator, such that the transition kernel is known and accessible. However, subtle discrepancies in dynamics between the source and target domains can significantly undermine the performance of policies when deployed for inference. This problem is known as the sim-to-real gap \citep{Jakobi1995NoiseRealityGap,mouret2013crossingrealitygapshort}. To overcome this hurdle, policies trained from the source domain need to be robust against perturbations in transitions so as to generalize over a set of possible target domains.

The framework of robust Markov decision processes (RMDPs) has gained increasing momentum as a potential solution to the sim-to-real gap. RMDPs account for the value under various target transitions close to the nominal transition, and come in two major formulations: the distributionally robust Markov decision process (DRMDP) \citep{satia1973MDP,nilim2005robustcontrol,iyengar2005robust} and the robust regularized Markov decision process (RRMDP) \citep{yang2023robustmarkovdecisionprocesses,zhang2024softrobustmdpsrisksensitive}. 
DRMDP constructs an uncertainty set of distinct target domain transitions around the nominal transition, and ensures robustness by optimizing policies against the worst-case transition in this uncertainty set. In contrast, RRMDP replaces the uncertainty set with a regularization term on the discrepancy between the nominal and perturbed transitions. To tackle large state-action spaces in many real-world applications, recent advances in robust online reinforcement learning  \citep{liu2024distributionally,liu2024upperlowerboundsdistributionally} have established the first provably efficient algorithms for DRMDPs with linear function approximation. A similar algorithm has also been devised for RRMDPs, but for the offline setting \citep{tang2025robust}. The essence of these algorithms is to maintain an optimistic estimation of the optimal robust Q-function, and take its greedy policy as the optimal policy estimation. Although these approaches enjoy polynomial sample complexity and robust guarantees, they fail in cases with high-dimensional or continuous action spaces, as getting the greedy policy is computationally intractable.

Policy optimization is one of the common approaches to overcome this hurdle. Distinct from the aforementioned greedy value iteration methods, policy optimization is another dominant paradigm that has achieved empirical success and extensive theoretical analysis in non-robust RL settings. In practice, policy optimization methods such as policy gradient \citep{agarwal2020theorypolicygradientmethods,sham2001npg}, actor–critic \citep{barto1983neuronlike,konda1999actorcritic}, and trust-region algorithms \citep{schulman2017trustregionpolicyoptimization,schulman2017proximalpolicyoptimizationalgorithms} are widely deployed. In particular, policy optimization algorithms are well suited for settings with large and continuous action spaces. Furthermore, unlike greedy updates that select deterministic actions, many practical domains require stochasticity, either to facilitate exploration, to enable continuous control, or to resist adversarial exploitation. Robust RL frameworks that naturally incorporate stochastic policies are therefore an important complement to existing deterministic methods. In addition, many policy optimization methods impose regularization against a reference policy, a design that has substantial practical appeal; policy regularization enables warm-starting and fine-tuning from existing policies and is also suitable for safety-relevant application settings such as healthcare \citep{gottesman2019healthcare,pmlr-v136-killian20a} and autonomous driving \citep{shalevshwartz2016safemultiagentreinforcementlearning,kahn2017uncertaintyawarereinforcementlearningcollision}, where it is advisable to remain close to established safe baselines. These benefits are not only pronounced in standard RL settings, but are also ideal for robust RL, leading to the open question:

\emph{Is it possible to design provably efficient policy optimization algorithms for online RMDPs with linear function approximation?}

In this work, we provide an affirmative answer by designing the first policy optimization algorithm for two formulations of online RMDPs, the policy-regularized $d$-rectangular DRMDP and RRMDP, and providing corresponding theoretical guarantees.

\textbf{Our main contributions} are summarized as follows:
\begin{itemize}[nosep,leftmargin=*]
\item We introduce the policy-regularized $d$-rectangular DRMDP and RRMDP frameworks, which ensure robustness by constraining both transition and policy. We verify that dynamic programming principles hold under these settings, facilitating algorithm design and theoretical analysis.

\item We develop an online algorithm, \textbf{D}istributionally \textbf{R}obust \textbf{R}egularized  \textbf{P}olicy \textbf{O}ptimization algorithm (DR-RPO), for RMDPs with function approximation. \alg is a model-free method with soft policy update, which imposes an Upper Confidence Bound (UCB) bonus to facilitate exploration. We provide specific implementations for policy-regularized $d$-rectangular DRMDPs and RRMDPs. To our knowledge, this is the first such result.

\item We prove an average suboptimality upper bound for \alg in the order of $\tilde{\cO}(d^2H^2/\sqrt{K})$, which is consistent with existing greedy value-iteration based methods in RMDPs.

\item We conduct numerical experiments to demonstrate the robustness performance of \alg on a simulated linear MDP environment.
\end{itemize}

\paragraph{Notations:} We denote $\Delta(\cS)$ as the set of probability measures over some set $\cS$. For any number $H\in\ZZ_{+}$, we denote $[H]$ as the set of $\{1,2,\cdots, H\}$. For any function $V:\cS\rightarrow \RR$, we denote $[\PP_hV](s,a) = \EE_{s'\sim P_h(\cdot|s,a)}[V(s')]$ as the expectation of $V$ with respect to the transition kernel $P_h$, and $[V(s)]_{\alpha}=\min\{V(s),\alpha\}$, given a scalar $\alpha>0$, as the truncated value of $V$. For a vector $\bx$, we denote $x_j$ as its $j$-th entry. And we denote $[x_i]_{i\in [d]}$ as a vector with the $i$-th entry being $x_i$. For a matrix $A$, denote $\lambda_i(A)$ as the $i$-th eigenvalue of $A$. For two matrices $A$ and $B$, we denote $A\preceq B$ as the fact that $B-A$ is a positive semidefinite matrix.
For any function $f:\cS\rightarrow\RR$, we denote $\|f\|_{\infty}=\sup_{s\in\cS}f(s)$. 
Given %
$P,Q\in\Delta(\cS)$, %
the total variation divergence of $P$ and $Q$ is defined as $D(P\Vert Q)=1/2\int_{\cS}|P(s)-Q(s)|ds$.

\section{RELATED WORK}

\paragraph{Robust MDPs}
DRMDP and RRMDP are two specific instantiations of the RMDP framework. DRMDP was first proposed by the foundational works of \citet{iyengar2005robust} and \citet{nilim2005robustcontrol}. Subsequent works \citep{zhou2021finitesample,yang2022theoreticalunderstandingsrobustmarkov,panaganti2022robustreinforcementlearningusing,shi2023curious,xu2023improvedsamplecomplexity} studied the DRMDP with the assumption of a generative model. Additional structures have since been added to the DRMDP framework, the $s$-rectangular uncertainty set \citep{behzadian2021srectangular}, the $d$-rectangular uncertainty set \citep{ma2023distributionallyrobustofflinereinforcement} and the linear mixture uncertainty set \citep{liu2025linear}. These assumptions on the uncertainty set structure enabled recent progress in online DRMDP \citep{liu2024distributionally,lu2024distributionallyrobustreinforcementlearning,liu2024upperlowerboundsdistributionally}. In particular, \citet{liu2024distributionally} achieved provably efficient exploration in online $d$-rectangular DRMDP with sublinear regret, and \citet{liu2024upperlowerboundsdistributionally} further proved upper and lower regret bounds for online learning in DRMDP by utilizing variance-weighted regressions. Several recent works \citep{wang2024samplecomplexityofflinedistributionally,liu2024minimax} have also studied offline $d$-rectangular DRMDP under various coverage assumptions and achieved comparable results in tractable and efficient algorithms.

A separate line of work studies RRMDP \citep{yang2023robustmarkovdecisionprocesses,zhang2024softrobustmdpsrisksensitive}, which replaces the uncertainty set in the DRMDP framework with a penalty term, thus imposing a softer regularization on distribution shift. Similarly to $d$-rectangular DRMDP, the $d$-rectangular RRMDP formulation has also been studied by \citet{tang2025robust}, which proposed a near-optimal algorithm for offline RRMDP with linear function approximation. As for the online setting, \citet{he2025sample} achieved sublinear regret for tabular RRMDP with general $f$-divergence, while \citet{panaganti2024modelfreerobustphidivergencereinforcement} proposed a hybrid algorithm to solve RRMDP with both online and offline data.
    
\paragraph{Exploration in Linear MDP}
Our work falls within the category of optimistic exploration in linear MDPs  \citep{jin2020,wang2021provably,wang2022sampleefficientreinforcementlearninglinearlyparameterized,sherman2024rateoptimalpolicyoptimizationlinear,cassel2024warmup}. The seminal work of \citet{jin2020} proposed the LSVI-UCB algorithm, which utilizes linear function approximation and an upper confidence bonus to achieve online learning on the linear MDP with a $\tilde{\cO}(\sqrt{d^3H^4K})$ regret bound. There has also been a surge of research in provably efficient exploration with policy optimization on linear MDPs. In particular, \citet{agarwal2020pcpgpolicycoverdirected} introduces a policy cover and a binary exploration bonus to enable optimistic exploration in an NPG-style algorithm. Subsequently, \citet{liu2023optimisticnaturalpolicygradient} developed the Optimistic NPG algorithm, which bears strong similarity to LSVI-UCB and adopts the same Hoeffding-style upper confidence bonus, meanwhile replacing the value iteration update with a soft policy update. In addition, \citet{cai2020provably} studied a similar NPG-based online algorithm on the linear mixture MDP, a slightly different setting than the linear MDP. So far, work in policy optimization with optimistic exploration has been limited in the non-robust setting, and it remains unknown whether similar algorithms can be derived for RMDPs, and whether comparable regret bounds can be attained.

\section{PRELIMINARIES}

We consider a finite horizon Markov decision process $\text{MDP}(\cS,\cA,H,P,r)$ with state and action spaces $\cS,\cA$, horizon $H$, transition kernel $P=\{P_h\}_{h=1}^H$, and reward functions $r=\{r_h\}_{h=1}^H$ where each $r_h:\cS\times\cA\rightarrow[0,1]$.
Without loss of generality, we assume that the action space $\cA$ is measurable and bounded, in the sense that $\int_{\cA}1 da = \vA < \infty$, where $da$ is the counting measure for discrete actions space and Lebesgue measure for continuous space.
Further, we assume a linear structure on transition kernels. 

\begin{assumption}\label{assumption:linearMDP} 
\citep{jin2020} Let $\bm{\phi}:\cS\times\cA\rightarrow\RR^d$ be a known feature mapping that satisfies $\phi_i(s,a)\geq0,\sum_{i=1}^d\phi_i(s,a)=1,\forall (s,a)\in\cS\times\cA$. We assume transition kernels and reward functions are both linear in this mapping, i.e. there exist unknown probability measures $\{\bmu_h(\cdot)\}_{h=1}^H\in\bigotimes_{h=1}^H\Delta(\cS)^d$ and known vectors $\{\bm{\theta}_h\}_{h=1}^H$ such that $\forall (s,a,h)\in\cS\times\cA\times[H],P_h(\cdot|s,a)=\inner{\bm{\phi}(s,a),\bmu_h(\cdot)}, r_h(s,a)=\inner{\bm{\phi}(s,a),\bm{\theta}_h}$. We also assume $\|\bm{\theta}_h\|\leq\sqrt{d},\forall h\in[H]$. 
\end{assumption}

\subsection{The $d$-Rectangular Linear DRMDP} \label{sec:pre_linearDRMDP}
To define the finite horizon $d$-rectangular DRMDP, we first introduce the $d$-rectangular uncertainty set following the definition in \citet{liu2024distributionally}.
\begin{definition}\label{def:uncertainty_set}
($d$-rectangular uncertainty set) Given the factor distributions $\{\bmu_h\}_{h=1}^H$ and feature $\bphi(\cdot,\cdot)$ of the linear MDP and uncertainty level $\rho>0$, for each $\mu_h$ we construct an uncertainty ball $\cU^\rho_{h,i}(\mu_{h,i}^0)=\{\mu\in\Delta(\mathcal{S}):D(\mu\|\mu_{h,i}^0)\leq\rho\}$, where $D(\cdot\|\cdot)$ denotes the Total Variation divergence. Then for all $(s,a)\in\cS\times\cA$, we define $\cU^\rho_{h}(s,a;\bmu_h^0)=\{\inner{\bphi(s,a),\bmu}:\mu_i\in\cU^\rho_{h,i}(\mu_{h,i}^0),\forall i\in[d]\}$. Subsequently, we denote $\cU^\rho_{h}(P_h^0)=\bigotimes_{(s,a)\in\cS\times\cA}\cU^\rho_{h}(s,a;\bmu_h^0)$ and $\cU^\rho(P^0)=\bigotimes_{h\in[H]}\cU^\rho_{h}(P_h^0)$.
\end{definition}
Then a  finite horizon $d$-rectangular DRMDP is defined by the tuple $\text{DRMDP}(\cS, \cA, H, P^0, \cU^\rho(P^0), r)$, where $P^0=\{P_h^0\}_{h=1}^H$ is the nominal transition kernel with $d$-rectangular uncertainty set $\cU^\rho(P^0)$.
The nominal transition $P^0$ and reward function $r$ satisfy the linear structure defined in \Cref{assumption:linearMDP}.

Under the $d$-rectangular DRMDP setting, let $V_h^{\pi,P}$ and $Q_h^{\pi,P}$ denote the value function and Q-function under transition kernel $P$. For any policy $\pi$ and time step $h\in[H]$, we define the robust value function $V_h^{\pi,\rho}:\cS\rightarrow\RR$ as $V_h^{\pi,\rho}(s)=\infPP V_h^{\pi,P}(s)$
and robust Q-function $Q_h^{\pi,\rho}:\cS\times\cA\rightarrow\RR$ as $Q_h^{\pi,\rho}(s,a)=\infPP Q_h^{\pi,P}(s,a)$. 
Furthermore, \citet{liu2024distributionally} showed that the robust value and $Q$-functions satisfy the following robust Bellman equation:
\begin{align}
Q_h^{\pi,\rho}(s,a)&=r_h(s,a)+\infPh{h}[\PP_hV_{h+1}^{\pi,\rho}](s,a) ,\label{eq:bellman_drmdp_q0}\\
V_h^{\pi,\rho}(s)&=\EE_{a\sim\pi(\cdot|s)}\big[Q_h^{\pi,\rho}(s,a)\big].\label{eq:bellman_drmdp_v0}
\end{align}%

\subsection{The $d$-Rectangular Linear RRMDP}

A finite horizon $d$-rectangular RRMDP is defined by the tuple $\text{RRMDP}(\cS, \cA, H, P^0, r, \sigma, R)$, where $\sigma$ is the regularization parameter and $R$ is the penalty on distribution shift. In this paper, we set $R$ to be the Total Variation distance $D(\cdot\|\cdot)$. The nominal transition $P^0$ and reward function $r$ satisfy the linear structure defined in \Cref{assumption:linearMDP}. Contrary to the DRMDP, which specifies an uncertainty set for target transitions, the RRMDP accounts for all probability measures over the state space as possible perturbed transitions, and regularize on the distance between the nominal and perturbed transitions with a penalty term. In particular, for any policy $\pi$ and time step $h\in[H]$, we can define the robust value function 
$V_h^{\pi,\sigma}:\cS\rightarrow\RR$ and robust $Q$-function $Q_h^{\pi,\sigma}:\cS\times\cA\rightarrow\RR$:
\begin{align*}
V_h^{\pi,\sigma}&(s)=\inf_{\bmu_t\in\Delta(\cS)^d,P_t=\inner{\bphi,\bmu_t},h\leq t\leq H}\EE_{P}\Bigg[\sum_{t=h}^H r_t(s_t,a_t)+\sigma\inner{\bm{\bphi}(s_t,a_t),D(\bmu_t\|\bmu_t^0)} \Bigg|s_h=s,\pi\Bigg],\\
Q_h^{\pi,\sigma}&(s,a)=\inf_{\bmu_t\in\Delta(\cS)^d,P_t=\inner{\bphi,\bmu_t},h\leq t\leq H}\EE_{P}\Bigg[\sum_{t=h}^H
r_t(s_t,a_t)+\sigma\inner{\bm{\phi}(s_t,a_t),D(\bmu_t\|\bmu_t^0)}\Bigg|s_h=s,a_h=a,\pi\Bigg].
\end{align*}
\citet{tang2025robust} show that the robust value and Q-functions satisfy the robust Bellman equations below:
\begin{align*}
Q_h^{\pi,\sigma}&(s,a)= r_h(s,a)+\inf_{\bmu_h\in\Delta(\cS)^d,P_h=\inner{\bm{\phi},\bmu_h}} \Big[\EE_{s'\sim P_h(\cdot|s,a)}[V_{h+1}^{\pi,\sigma}(s')]+\sigma\inner{\bm{\phi}(s,a),D(\bmu_h\|\bmu_h^0)}\Big],\\
V_h^{\pi,\sigma}&(s)=\EE_{a\sim\pi(\cdot|s)}\big[Q_h^{\pi,\sigma}(s,a)\big].
\end{align*}

\subsection{Policy-Regularized $d$-Rectangular Linear DRMDPs and RRMDPs}

On top of the constraint on transition kernels in $d$-rectangular DRMDP and RRMDP, we impose a second regularization on the policy such that it stays close to a reference policy $\piref$. This is inspired by the policy-regularized MDP framework, which modifies the optimality equations themselves, producing soft Bellman operators in which the log-sum-exp over Q-values replaces the hard max. Numerous existing literature have explored policy-regularized value functions with similar formulations. For instance, \citet{schulman2018equivalencepolicygradientssoft} study the Boltzmann policy with value function $v_{\bm{\theta}}=\EE_{a\sim\pi_{q_{\bm{\theta}}}^{\mathcal{B}}}[q_{\bm{\theta}}(a)]-\tau D_{KL}[\pi_{q_{\bm{\theta}}}^{\mathcal{B}}\Vert\bar{\pi}]$ where $\bar{\pi}$ is a static reference policy. \citet{haarnoja2017reinforcementlearningdeepenergybased} and \citet{fox2017tamingnoisereinforcementlearning} also discuss similar Boltzmann policy updates with the form $\pi(a|s)\propto\exp(\beta f(s,a))$ where $f$ is some function designed to approximate the Q-value. 

Policy regularization has several profound consequences. First, it produces stochastic optimal policies of the form $\pi(a|s)\propto\piref(a|s)\exp(\eta Q(s,a))$, which balance the exploitation of high-value actions with adherence to a prior policy $\piref$. Moreover, soft policies are more suitable for applications with continuous action space.
Second, it connects RL to probabilistic inference: the KL term can be seen as a log-prior over trajectories, and the reward as a log-likelihood \citep{Kappen_2012,levine2018reinforcementlearningcontrolprobabilistic}. This makes it possible to apply inference algorithms to control problems and to incorporate domain knowledge via the prior $\pi_0$. Third, the softened value function has better contraction properties in the presence of approximation errors, improving stability in practice \citep{fox2017tamingnoisereinforcementlearning,neu2017unifiedviewentropyregularizedmarkov,chaudhuri2019RMDP,vieillard2020munchausenreinforcementlearning,zhang2023global}. 
Based on the $d$-rectangular linear DRMDP framework, we define the robust policy-regularized value function and $Q$-function as 
\begin{align}
&\tilde{V}_h^{\pi,\rho}(s)=\infPP\EE_{P}\Bigg[\sum_{t=h}^H\big(r_t(s_t,a_t)-\frac{1}{\eta}D_{KL}[\pi_t(\cdot|s_t)\Vert\piref_t(\cdot|s_t)]\big)\Bigg|s_h=s,\pi\Bigg],\label{eq:drmdp_v}\\
&\tilde{Q}_h^{\pi,\rho}(s,a)=r_h(s,a)+\infPP\EE_{P}\Bigg[\sum_{t=h+1}^Hr_t(s_t,a_t)-\frac{1}{\eta}D_{KL}[\pi_t(\cdot|s_t)\Vert\piref_t(\cdot|s_t)]\Bigg|(s_h,a_h)=(s,a),\pi\Bigg].\label{eq:drmdp_q}
\end{align}%
Similarly, we define the robust policy-regularized value function and $Q$ function under the $d$-rectangular linear RRMDP framework as 
\begin{align}
&\tilde{V}_h^{\pi,\sigma}(s)=\inf_{\bmu_t\in\Delta(\cS)^d,P_t=\inner{\bphi,\bmu_t},h\leq t\leq H}\bigg[\sum_{t=h}^H\big(r_t(s_t,a_t)+\sigma\inner{\bm{\phi}(s_t,a_t),D(\bmu_t\|\bmu_t^0)}\label{eq:rrmdp_v}\\
&\quad-\frac{1}{\eta}D_{KL}[\pi_t(\cdot|s_t)\Vert\piref_t(\cdot|s_t)]\big)\bigg|s_h=s,\pi\bigg],\notag\\
&\tilde{Q}_h^{\pi,\sigma}(s,a)=r_h(s,a) 
\quad+\inf_{\bmu_t\in\Delta(\cS)^d,P_t=\inner{\bphi,\bmu_h},h+1\leq t\leq H}\Bigg[\sigma\inner{\bphi(s,a),D(\bmu_h\|\bmu_h^0)} \label{eq:rrmdp_q}\\
&\quad+\EE_{P}\bigg[\sum_{t=h+1}^H\Big(r_t(s_t,a_t)+\sigma\inner{\bm{\phi}(s_t,a_t),D(\bmu_t\|\bmu_t^0)}-\frac{1}{\eta}D_{KL}[\pi_t(\cdot|s_t)\Vert\piref_t(\cdot|s_t)]\Big)\bigg|s_h=s,a_h=a,\pi\bigg]\Bigg].\notag
\end{align}
Next we show that the robust policy-regularized Bellman equations hold for policy-regularized robust value functions defined above.
\begin{proposition}\label{proposition:robust_bellman_kl_drmdp}
(DRMDP) For any policy $\pi$ and any $(s,a,h)\in\cS\times\cA\times[H]$, we have
\begin{align*}
\tilde{Q}_h^{\pi,\rho}(s)=&r_h(s,a)+\infPh{h}\EE_{s'\sim P_h(\cdot|s,a)}\tilde{V}_{h+1}^{\pi,\rho}(s'),\\
\tilde{V}_h^{\pi,\rho}(s)=&\inner{\tilde{Q}_h^{\pi,\rho}(s,\cdot), \pi_h(\cdot|s)}-\frac{1}{\eta}D_{KL}[\pi_h(\cdot|s)\Vert\piref_h(\cdot|s)].
\end{align*}
\end{proposition}
\begin{proposition}\label{proposition:robust_bellman_kl_rrmdp}
(RRMDP) For any policy $\pi$ and any $(s,a,h)\in\cS\times\cA\times[H]$, we have that
\begin{align*}
&\tilde{Q}_h^{\pi,\sigma}(s,a)=r_h(s,a)+\inf_{\bmu_h\in\Delta(\cS)^d,P_h=\inner{\bm{\phi},\bmu_h}}\Big[\EE_{s'\sim P_h(\cdot|s,a)}[\tilde{V}_{h+1}^{\pi,\sigma}(s')]+\sigma\inner{\bm{\phi}(s,a),D(\bmu_h\|\bmu_h^0)}\Big],\\
&\tilde{V}_h^{\pi,\sigma}(s)=\EE_{a\sim\pi(\cdot|s)}\big[\tilde{Q}_h^{\pi,\sigma}(s,a)\big]-\frac{1}{\eta}D_{KL}[\pi_h(\cdot|s)\Vert\piref_h(\cdot|s)].
\end{align*}
\end{proposition}
For all $(s,a,h)\in\cS\times\cA\times[H]$, we define the policy-regularized optimal value function as $\tilde{V}_h^{*,\rho}(s)=\sup_{\pi\in\Pi}\tilde{V}_h^{\pi,\rho}(s)$ for DRMDP and $\tilde{V}_h^{*,\sigma}(s)=\sup_{\pi\in\Pi}\tilde{V}_h^{\pi,\sigma}(s)$ for RRMDP; similarly, we define $\tilde{Q}_h^{*,\rho}(s,a)=\sup_{\pi\in\Pi}\tilde{Q}_h^{\pi,\rho}(s,a)$ for DRMDP and $\tilde{Q}_h^{*,\sigma}(s,a)=\sup_{\pi\in\Pi}\tilde{Q}_h^{\pi,\sigma}(s,a)$ for RRMDP. It is well-known that there exists a stationary and deterministic policy that achieves the optimal value function for a standard MDP \citep{Sutton1998}; likewise, a stationary and deterministic optimal policy also exists for $d$-rectangular linear DRMDP \citep{liu2024distributionally} and $d$-rectangular linear RRMDP \citep{tang2025robust}. In contrast, we show that under policy regularization, the optimal policies have a particular closed form. 

\begin{proposition}\label{proposition:optimal_policy_drmdp}
(Optimal policy under DRMDP) Under the policy-regularized $d$-rectangular linear DRMDP, there exists a policy $\pi^*$ 
such that $\tilde{V}^{\pi^*,\rho}_h(s)=\tilde{V}^{*,\rho}_h(s)$ and $\tilde{Q}^{\pi^*,\rho}_h(s,a)=\tilde{Q}^{*,\rho}_h(s,a), \forall (s,a,h)\in\cS\times\cA\times[H]$, where $\pi^*$ is defined by
\begin{align*}
\pi_h^*(a|s)=\frac{1}{Z_h(s)}\piref_h(a|s)\exp\big(\eta\tilde{Q}^{*,\rho}_{h}(s,a)\big),
\end{align*}
where $Z_h(s)=\int_\cA\piref_h(a|s)\exp\big(\eta\tilde{Q}^{*,\rho}_{h}(s,a)\big)da$ is the partition function.
\end{proposition}

\begin{proposition}\label{proposition:optimal_policy_rrmdp}
(Optimal policy under  RRMDP) Under the policy-regularized $d$-rectangular linear RRMDP, there exists a policy $\pi^*$ such that $\tilde{V}^{\pi^*,\sigma}_h(s)=\tilde{V}^{*,\sigma}_h(s)$ and $\tilde{Q}^{\pi^*,\sigma}_h(s,a)=\tilde{Q}^{*,\sigma}_h(s,a), \forall (s,a,h)\in\cS\times\cA\times[H]$, where $\pi^*$ is defined by
\begin{align*}
\pi_h^*(a|s)=\frac{1}{Z_h(s)}\piref_h(a|s)\exp\big(\eta\tilde{Q}^{*,\sigma}_{h}(s,a)\big),
\end{align*}
where $Z_h(s)=\int_\cA\piref_h(a|s)\exp\big(\eta\tilde{Q}^{*,\rho}_{h}(s,a)\big)da$.
\end{proposition}

\begin{remark}
 Due to the regularization against a reference policy, $\pi^*$ may not be deterministic under the policy-regularized setting. In particular, unless $\piref$ is deterministic, the closed forms in \Cref{proposition:optimal_policy_drmdp} and \Cref{proposition:optimal_policy_rrmdp} imply that $\pi^*$ is necessarily stochastic.
\end{remark}
The agent's goal is to learn the optimal policy within $K$ episodes of online interaction with the source domain, minimizing average suboptimality. In episode $k$, it starts at initial state $s_1^k$ and follows policy $\pi^k$. The average suboptimality is defined as $\text{AveSubopt}(K)=1/K\sum_{k=1}^K(\tilde{V}_1^{*,\rho}-\tilde{V}_1^{\pi^k,\rho})$ for DRMDP and $\text{AveSubopt}(K)=1/K\sum_{k=1}^K(\tilde{V}_1^{*,\sigma}-\tilde{V}_1^{\pi^k,\sigma})$ for RRMDP.

\section{ALGORITHM DESIGN}

In this section, we develop \algo, a value iteration based algorithm for online policy-regularized RMDPs. On a high level, \alg maintains an estimate of the robust value and $Q$-functions via linear approximation with online data, and ensures that the estimate is optimistic by incorporating a Upper Confidence Bound (UCB)-style bonus. In each episode, it updates the current policy using a softmax update rule, constraining the policy to stay near the reference policy. To motivate the algorithm design, we first discuss the linearity of the robust $Q$-function. 

\subsection{Linearity of the Robust $Q$-Function under the $d$-Rectangular Linear DRMDP}

It is known that the optimization problem under the $d$-rectangular linear DRMDP and the TV uncertainty set has a dual formulation as follows:
\begin{proposition}\label{proposition:strong_duality}(\citet{shi2023curious}, Lemma 4) 
Consider any probability measure $\mu^0\in\Delta{(\cS)}$, any fixed uncertainty level $\rho$, the uncertainty set $\cU^\rho(\mu^0)=\{\mu:\mu\in\Delta{(\cS)},D_{TV}(\mu\|\mu^0)\leq\rho\}$, and any function $V:\cS\rightarrow[0,H]$, one has that
\begin{align*}
\inf_{\mu\in\cU^\rho(\mu^0)}\EE_{s\sim\mu}V(s)=&\max_{\alpha\in[V_{\min},V_{\max}]}\Big\{\EE_{s\sim\mu^0}[V(s)]_\alpha-\rho(\alpha-\min_{s'}[V(s')]_\alpha)\Big\},
\end{align*}
where $[V(s)]_\alpha=\min\{V(s),\alpha\}$,$V_{\min}=\min_sV(s)$, and $V_{\max}=\max_sV(s)$.
\end{proposition}

Solving $\min_{s'}[V(s')]_\alpha$ is computationally expensive when $\cS$ is large, especially when $\min_{s'}[V(s')]_\alpha$ is not convex with respect to $\cS$. To circumvent this issue, we introduce the fail-state assumption to zero-out $\min_{s'}[V(s')]_\alpha$ below, following the same definition from \citet{panaganti2022robustreinforcementlearningusing}.
\begin{assumption}\label{assumption:failstate}
\cite[Assumption 3]{panaganti2022robustreinforcementlearningusing}
There exists a fail state $s^\dagger\in\cS$ such that $\forall(a,h)\in\cA\times[H], r_h(s^\dagger,a)=0, P_h^0(s^\dagger|s^\dagger,a)=1$.
\end{assumption} 

\begin{remark}\label{remark:failstate}
\Cref{assumption:failstate} implies that for all $\alpha\in[0,H]$, $\min_{s'}[V(s')]_\alpha=[V(s^\dagger)]_\alpha=0$, which allows us to simplify \Cref{proposition:strong_duality}: for any probability measure $\mu^0\in\Delta(\cS)$, any fixed uncertainty level $\rho$, the uncertainty set $\cU^\rho(\mu^0)=\{\mu:\mu\in\Delta(\cS),D_{TV}(\mu\|\mu^0)\leq\rho\}$, and any function $V:\cS\rightarrow[0,H]$ such that $\min_{s\in\cS}V(s)=0$, one has that
\begin{align}\label{eq:dual_drmdp_simp}
&\inf_{\mu\in\cU^\rho(\mu^0)}\EE_{s\sim\mu}V(s)=\max_{\alpha\in[V_{\min},V_{\max}]}\Big\{\EE_{s\sim\mu^0}[V(s)]_\alpha-\rho\alpha\Big\}.
\end{align}
\end{remark}
We note that \Cref{assumption:failstate} is not overly restrictive since fail-states are common in practical scenarios, such as autonomous car crashes or robot falls. \citet{liu2024distributionally} also showed that \Cref{assumption:failstate} is compatible with the linear MDP structure.

With this simplification, the robust $Q$-function $\tilde{Q}_h^{\pi,\rho}(\cdot,\cdot)$ is now linear in $\bm{\phi}(\cdot,\cdot)$ for any policy $\pi$.
\begin{proposition}
Under \Cref{assumption:linearMDP} and \Cref{assumption:failstate}, for all $(s,a,h)\in\cS\times\cA\times[H]$ and $\pi\in\Pi$, the robust $Q$-function has the linear form
\begin{align*}
\tilde{Q}_h^{\pi,\rho}(s,a)=\inner{\bm{\phi}(s,a),\bm{\theta}_h+\bm{\nu}_h^{\pi,\rho}}\mathbf{1}\{s\neq s^\dagger\},
\end{align*}
where $\bm{\nu}_h^{\pi,\rho}=(\nu_{h,i}^{\pi,\rho})_{i\in[d]}$, $\nu_{h,i}^{\pi,\rho}=\max_{\alpha\in[0,H]}\{z_{h,i}^\pi(\alpha)-\rho\alpha\}$, and $z_{h,i}^\pi(\alpha)=\EE_{s\sim\mu_{h,i}^0}[\tilde{V}_{h+1}^{\pi,\rho}(s')]_\alpha$.
\end{proposition}
The proof of the above proposition is identical to the original proof supplied by \citet{liu2024distributionally} despite the introduction of policy regularization, because the proof only utilizes \eqref{eq:bellman_drmdp_q0} and not \eqref{eq:bellman_drmdp_v0} among the robust Bellman equations. When policy regularization is imposed in \Cref{proposition:robust_bellman_kl_drmdp}, \eqref{eq:bellman_drmdp_q0} remains unchanged.

Then given some robust value function $\tilde{V}_{h+1}^{k,\rho}(\cdot)$, we can also compute $\tilde{Q}_{h}^{k,\rho}(s,a)$ by
\begin{align}\label{eq:q_k_rho_linear_form}
\tilde{Q}_h^{k,\rho}(s,a)=\inner{\bphi(s,a),\bm{\theta}_h+\bm{\nu}_h^{\pi,\rho}}\mathbf{1}\{s\neq s^\dagger\},
\end{align}
where $\bnu_h^{k,\rho}=(\bnu_{h,i}^{k,\rho})_{i\in[d]}$, $\bnu_{h,i}^{k,\rho}=\max_{\alpha\in[0,H]}\{z_{h,i}^k(\alpha)-\rho\alpha\}$, and $z_{h,i}^k(\alpha)=\EE_{s\sim\bmu_{h,i}^0}[\tilde{V}_{h+1}^{k,\rho}(s')]_\alpha$. This constitutes a Bellman backup, which enables us to recursively approximate the optimal value function.  By collecting near-optimal trajectories $\{(s_h^\tau, a_h^\tau, r_h^\tau)\}_{(h,\tau)\in[H]\times[k-1]}$, we obtain estimates $\tilde{V}_{h+1}^{k,\rho}$ of the optimal value function. From the definition of $z_{h,i}^k(\alpha)$, we further have that $\la{\bphi(s,a),z_{h,i}^k(\alpha)}\ra=[\PP_h^0[\tilde{V}_{h+1}^{k,\rho}]_\alpha](s,a)$, which allows us to approximate $\bm{z}_{h}^k(\alpha)$ by solving the following regression problem:
\begin{align*}
\bm{z}_{h}^k(\alpha)=& \argmin_{\bm{z}\in\RR^d}\sum_{\tau=1}^{k-1}\Big([\tilde{V}_{h+1}^{k,\rho}(s_{h+1}^\tau)]_\alpha -\big(\bm{\phi}(s_{h}^\tau,a_{h}^\tau)\big)^\top \bm{z}\Big)+\lambda\|\bm{z}\|_2^2,
\end{align*}
which has the closed-form solution %
\begin{align}\label{eq:regression_drmdp_z}
  \bm{z}^k_h(\alpha)=(\Lambda_h^k)^{-1}\bigg[\sum_{\tau=1}^{k-1}\bm{\phi}(s_{h}^\tau,a_{h}^\tau)[\tilde{V}_{h+1}^{k,\rho}(s_{h+1})]_\alpha\bigg],
\end{align}
where $\Lambda_h^k=\sum_{\tau=1}^{k-1}\bm{\phi}(s_{h}^\tau,a_{h}^\tau)(\bm{\phi}(s_{h}^\tau,a_{h}^\tau))^\top+\lambda I$ is the Gram matrix. 

With the estimate of $\bm{z}_h^k(\alpha)$, we can compute $\bnu_{h,i}^{k,\rho}$ by
\begin{align}\label{eq:regression_drmdp_nu}
\nu_{h,i}^{k,\rho}=\max_{\alpha\in[0,H]}\{z_{h,i}^{k,\rho}(\alpha)-\rho\alpha\},i\in[d].
\end{align}
and finally obtain the optimal robust $Q$-function estimate $\tilde{Q}_h^{k,\rho}$ by \eqref{eq:q_k_rho_linear_form}.

\subsection{Linearity of the Robust $Q$-Function under the $d$-rectangular linear RRMDP}

For the regularized optimization problem, we have a similar dual formulation as \Cref{proposition:strong_duality}:
\begin{proposition} \label{proposition:duality_rrmdp}\cite[Proposition 4.3]{tang2025robust} Consider any probability measure $\bmu^0\in\Delta({\cS})^d$, any regularization constant $\sigma>0$, and any function $V:\cS\rightarrow[0,H]$, one has that $\inf_{\bmu\in\Delta(\cS)^d}\EE_{s\sim\bmu}V(s)+\lambda  D_{TV}(\bmu\|\bmu^0)=\EE_{s\sim\bmu^0}[V(s)]_{V_{\min}+\sigma}$.
Note that the existence of the fail-state defined in \Cref{assumption:failstate} implies that $V_{\min}=0$, which further simplifies the dual form to $\inf_{\bmu\in\Delta(\cS)^d}\EE_{s\sim\bmu}V(s)+\lambda  D_{TV}(\bmu\|\bmu^0)=\EE_{s\sim\bmu^0}[V(s)]_{\sigma}.$
\end{proposition}

The robust $Q$-function $\tilde{Q}_h^{\pi,\sigma}$ is also linear in the features $\bm{\phi}(\cdot,\cdot)$ for any policy $\pi$.
\begin{proposition}\cite[Proposition 4.1]{tang2025robust} Under \Cref{assumption:linearMDP} and \Cref{assumption:failstate}, for all $(s,a,h)\in\cS\times\cA\times[H]$ and $\pi\in\Pi$, the robust $Q$-function has the linear form
\begin{align*}
\tilde{Q}_h^{\pi,\sigma}(s,a)=\inner{\bm{\phi}(s,a),\bm{\theta}_h+\bm{\nu}_h^{\pi,\sigma}}\mathbf{1}\{s\neq s^\dagger\},
\end{align*}
where $\bm{\nu}_h^{\pi,\sigma}=(\nu_{h,i}^{\pi,\sigma})_{i\in[d]}$, and $\nu_{h,i}^{\pi,\sigma}=\inf_{\mu_{h,i}\in\Delta(\cS)}\EE_{s\sim\mu_{h,i}^0}[\tilde{V}_{h+1}^{\pi,\sigma}(s')+\lambda D_{TV}(\mu_{h,i}\|\mu_{h,i}^0)]$.
\end{proposition}
Combining \Cref{proposition:duality_rrmdp} and the definition of $\bm{\nu}_h^{\pi,\sigma}$, we have $\bm{\nu}_h^{\pi,\sigma}=\EE_{s'\sim\bmu^0}[\tilde{V}_{h+1}^{\pi,\sigma}(s')]$. Applying the linearity of the transition in \Cref{assumption:linearMDP}, we further have
$\inner{\bm{\phi}(s,a),\bm{\nu}_h^{\pi,\sigma}}=[\PP_h^0\tilde{V}_{h+1}^{\pi,\sigma}](s,a)$.
Similarly to the DRMDP case, we can now approximate the robust $Q$-function by solving the ridge linear regression problem
\begin{align*}
\bm{\nu}_{h}^{k,\sigma}&=\argmin_{\bm{\nu}\in\RR^d}\sum_{\tau=1}^{k-1}\Big([\tilde{V}_{h+1}^{k,\sigma}(s_{h+1}^\tau)]_\sigma-(\bm{\phi}(s_{h}^\tau,a_{h}^\tau))^\top \bm{\nu}\Big)+\lambda\|\bm{\nu}\|_2^2,
\end{align*}
with the closed-form solution being
\begin{align}\label{eq:regression_rrmdp_nu}
\bm{\nu}^{k,\sigma}_h=(\Lambda_h^k)^{-1}\bigg[\sum_{\tau=1}^{k-1}\bm{\phi}(s_{h}^\tau,a_{h}^\tau)[\tilde{V}_{h+1}^{k,\sigma}(s_{h+1}^\tau)]_\sigma\bigg],
\end{align}
where the data $(s_h^\tau, a_h^\tau, r_h^\tau)$ and Gram matrix $\Lambda_h^k$ are identically defined as in the DRMDP case discussed in the previous subsection. 

\subsection{Softmax Policy Update}

With an estimated $Q$-function, the next step towards formulating an online policy optimization algorithm is to update the policy $\pi$ to maximize the learning objective under the current function approximation. In settings without policy regularization such as \citet{jin2020,liu2024distributionally}, this can be simply done by setting a greedy policy where $\pi_h^k(s)=\arg\max_{a\in\cA}\tilde{Q}_h^k(s,a)$. With the additional regularization term, however, the objective function changes, as does the corresponding policy which maximizes the objective. We introduce the following results for deriving the optimal policy.

\begin{proposition}\label{propositon:optimization_kl}
\cite[Lemma 3.4]{zhao2025logarithmicregretonlineklregularized}
For any fixed $s\in\cS$ and $Q:\cS\times\cA\rightarrow[0,H]$, we have
\begin{align*}
&\max_\pi\inner{Q(s,\cdot),\pi(\cdot|s)}-1/{\eta}D_{KL}[\pi(\cdot|s)\|\piref(\cdot|s)] =\frac{1}{\eta}\log\EE_{a\sim\piref(\cdot|s)}\exp(\eta Q(s,a)),
\end{align*}
and the maximizer of the objective is
\begin{align*}
\pi(a|s)=\frac{\piref(a|s)\exp\big(\eta Q(s,a)\big)}{\int_A \piref(a|s)\exp\big(\eta Q(s,a)\big)da}.
\end{align*}
\end{proposition} 

Hence, the updated policy takes the form of $\pi_h^k(\cdot|s)\propto\piref_h(\cdot|s)\exp\big(\eta \tilde{Q}_h^k(s,\cdot)\big)$ under policy-regularization.

\subsection{UCB Exploration}

Thanks to the linear function approximation of $\tilde{Q}_h^k$ and the update rule for $\pi_h^k$, we can design an online policy optimization algorithm that iteratively rolls out $\pi_h^k$ to collect data, estimates $\tilde{Q}_h^k$ with ridge linear regression, and improves $\pi_h^k$ with the softmax update rule. However, the policy solved from $\tilde{Q}_h^k$ lacks the incentive to explore beyond the seen state-action pairs, leading to suboptimal performance. To address this, we impose an Upper Confidence Bound (UCB) bonus to the estimated $\tilde{Q}_h^k$, motivating the algorithm to explore less frequently-visited state-action pairs. %

\begin{algorithm}[ht]
    \caption{\alg}
    \begin{algorithmic}[1]\label{algorithm:alg}
    \REQUIRE{
        uncertainty level $\rho > 0$ (for DRMDPs), or regularizer $\sigma > 0$ (for RRMDPs).
        }
    \STATE Initialize $\pi_h^0=\piref_h$ and $\tilde{Q}_h^0=0$
    \FOR {step $k=1,\cdots,K$}
        \STATE Receive the initial state $s_1^k$
        \FOR {step $h=1,\cdots,H$}
            \STATE $\pi_h^k(a|s)\propto\piref_h(a|s)\exp\big(\eta \tilde{Q}_h^{k-1}(s,a)\big)$ \label{algline:update_policy}
            \STATE Sample action $a_{h}^k\sim\pi_h^k(\cdot|s_h^k)$ and receive reward $r_h^k$ and next state $s_{h+1}^k$ \label{algline:rollout}
        \ENDFOR
        \FOR {step $h=H,\cdots,1$} \label{algline:start_of_stage2}
            \STATE $\Lambda_h^k\leftarrow\sum_{\tau=1}^{k-1}\bm{\phi}(s_h^\tau,a_h^\tau)\bm{\phi}(s_h^\tau,a_h^\tau)^\top+\lambda I$ \label{algline:gram_matrix}
            \IF {$h=H$}
                \STATE $\bm{\nu}_h^{k}\leftarrow0$
            \ELSE
                \STATE  
                Update $\bm{\nu}_{h}^{k}$ via \eqref{eq:regression_drmdp_z} and \eqref{eq:regression_drmdp_nu} for DRMDP, or \eqref{eq:regression_rrmdp_nu} for RRMDP
            \ENDIF \label{algline:end_of_regression}
            \STATE $\Gamma_h^k(s,a)\leftarrow\beta\sum_{i=1}^d\phi_i(s,a)\sqrt{\mathbf{1}_i^\top(\Lambda_h^k)^{-1}\mathbf{1}_i}$ \label{algline:bonus}
            \STATE $\tilde{Q}_h^{k}(s,a)\leftarrow\min\big\{\bm{\phi}(s,a)^\top(\bm{\theta}_h+\bm{\nu}_h^{k})+\Gamma_h^k(s,a), H-h+1\big\}\mathbf{1}\{s\neq s_f\}$ \label{algline:q}
             \STATE $\tilde{V}_h^{k}(s)\leftarrow\la\tilde{Q}_h^{k}(s,\cdot), \pi_h^k(\cdot|s)\ra_{\cA}-1/\eta D_{KL}[\pi_h^k(\cdot|s)\Vert\piref_h(\cdot|s)]$ \label{algline:v}
        \ENDFOR
    \ENDFOR
    \end{algorithmic}
\end{algorithm}

 We present our algorithm in \Cref{algorithm:alg}. Note that $\tilde{V}_h^k,\tilde{Q}_h^k,\bm{\nu}_h^k$ in \algo\ are notations that represent different meanings under policy-regularized $d$-rectangular linear DRMDP and RRMDP. For DRMDP, we denote $\tilde{V}_h^{k,\rho}=\tilde{V}_h^k$, $\tilde{Q}_h^{k,\rho}=\tilde{Q}_h^k$, $\bm{\nu}_h^{\rho,k}=\bm{\nu}_h^k$; for RRMDP, we denote $\tilde{V}_h^{k,\sigma}=\tilde{V}_h^k,\tilde{Q}_h^{k,\sigma}=\tilde{Q}_h^k,\bm{\nu}_h^{\sigma,k}=\bm{\nu}_h^k$.
In each episode, \alg first updates the current policy $\pi_h^k$ according to the softmax update rule in \Cref{propositon:optimization_kl} (Line \ref{algline:update_policy}) and rolls out this policy in the nominal environment to collect a new  trajectory (Line \ref{algline:rollout}). Then it performs ridge linear regression to obtain an estimate for the $Q$-function (Lines \ref{algline:start_of_stage2}-\ref{algline:end_of_regression}). For the DRMDP setting, the algorithm follows \eqref{eq:regression_drmdp_z} to estimate $\bm{z}_h^k(\alpha)$ and $\bm{\nu}_h^{k,\rho}$; for the RRMDP setting, the algorithm follows \eqref{eq:regression_rrmdp_nu} to estimate $\bm{\nu}_h^{k,\sigma}$. We then add a bonus $\Gamma_h^k$ to the approximated $Q$-function (Line \ref{algline:bonus}) following \cite{liu2024distributionally} to facilitate exploration. Finally, the algorithm completes the estimation for $\tilde{Q}_h^k$ and computes $\tilde{V}_h^k$ according to the robust Bellman equations in \Cref{proposition:robust_bellman_kl_drmdp} and \Cref{proposition:robust_bellman_kl_rrmdp} (Lines \ref{algline:q}-\ref{algline:v}).

\section{THEORETICAL RESULTS}

In this section, we present the main theoretical guarantees for \Cref{algorithm:alg}. First, we discuss an assumption that helps bound the average suboptimality.

\begin{assumption}\label{assumption:uniform_exploration}
(Uniform exploration) For all $\pi\in\Pi$ and $h\in[H]$, there exists constant $\alpha>0$ such that
\begin{align*}
\EE_\pi[\bm{\phi}(s_h,a_h)\bm{\phi}(s_h,a_h)^\top]\succeq\alpha I.
\end{align*}
\end{assumption}

By lower-bounding the smallest eigenvalue of the $\EE_\pi[\bm{\phi}(s_h,a_h)\bm{\phi}(s_h,a_h)^\top]$ matrix, \Cref{assumption:uniform_exploration} ensures that it is possible to sufficiently explore the state-action space of the MDP. It has been applied in previous works such as \citet{liu2024distributionally}, and \citet{wang2021what} showed that $\alpha$ is upper-bounded by $\cO(1/d)$.

\begin{theorem}\label{theorem:main_subopt_upper_bound}
(Average sub-optimality bound) Set $\lambda=1$ in \Cref{algorithm:alg}; then under \Cref{assumption:linearMDP} and \Cref{assumption:failstate}, for arbitrary $p\in(0,1)$, with probability at least $1-p/3$ the average sub-optimality of \alg satisfies
\begin{align*}
\text{AveSubopt}(K)\leq\sqrt{\frac{16H^3}{K}\log\frac{3}{p}}+\frac{1}{K}\sum_{k=1}^K\sum_{h=1}^H\Gamma_h^k\sahk
\end{align*}
where $\Gamma_h^k\sahk=\beta\sum_{i=1}^d\phi_i\sahk\sqrt{\mathbf{1}_i^\top(\Lambda_h^k)^{-1}\mathbf{1}_i}$ 
is the bonus term in \Cref{algorithm:alg}. If we further assume \Cref{assumption:uniform_exploration} and set $\beta=CdH\sqrt{\log\frac{3(c_\beta+1)dKH\vA}{p}}$ where $C,c_\beta$ are independent constants, there exists an absolute constant $c>0$ such that with probability at least $1-p$, the average sub-optimality of \alg satisfies
\begin{align*}
\text{AveSubopt}(K)\leq\frac{cd^2H^2}{\alpha\sqrt{K}}\log(3dHK\vA/p),
\end{align*}
for both the DRMDP and RRMDP settings.
\end{theorem}

\begin{remark}
Under \Cref{assumption:uniform_exploration}, \Cref{theorem:main_subopt_upper_bound} implies that the average suboptimality bound of \alg is $\tilde{\cO}(\sqrt{d^4H^4/K})$ up to logarithmic terms if $\alpha=\cO(1/d)$. This is consistent with the $\tilde{\cO}(\sqrt{d^4H^4/K})$ bound achieved by the DR-LSVI-UCB algorithm in \citet{liu2024distributionally}, which is the online greedy value-iteration counterpart of \alg in the $d$-rectangular linear DRMDP. The main difference lies in the logarithmic terms, where \alg pays an additional factor of $\log\vA$ due to the softmax policy update rule. A different previous work, \citet{blanchet2023doublepessimismprovablyefficient}, also arrives at a similar $\tilde{\cO}(\sqrt{d^4H^4/c^\dagger K})$ suboptimality bound with their P$^2$MPO algorithm on offline DRMDPs, where $c^\dagger$ pertains to dataset coverage. As for the RRMDP setting, the offline algorithm R$^2$PVI developed by \citet{tang2025robust} similarly achieves the $\tilde{\cO}(\sqrt{d^4H^4/K})$ bound under TV-regularization on the transition kernel. \alg is thus on par with existing literature on both online and offline DRMDPs and RRMDPs in terms of regret, but enjoys the additional benefits of being a policy optimization algorithm including policy stochasticity and regularization against a reference policy.
\end{remark}

\section{NUMERICAL EXPERIMENTS} \label{sec:exp}

To empirically test the performance of \algo, we compare the policies trained by \algo\ and various baselines on two experimental environments: a simulated off-dynamics linear MDP, and the American put option problem. We implemented several baselines for comparison against \algo. This includes LSVI-UCB \citep{jin2020} and DR-LSVI-UCB \citep{liu2024distributionally}, which serve as the non-robust and robust greedy value-iteration counterpart of \alg respectively. We also include OPPO \citep{cai2020provably} and Optimistic NPG \citep{liu2023optimisticnaturalpolicygradient} as baselines for non-robust policy optimization baselines. We note that \citet{cai2020provably} only gave theoretical guarantees for OPPO under linear mixture MDP, but their algorithm can be modified to apply to linear MDP.

\subsection{Simulated Linear MDP}

The off-dynamics linear MDP satisfies \Cref{assumption:linearMDP} with $d=4$ and has a state space $\cS=\{s_1,s_2,s_3,s_4,s_5\}$, an action space $\cA=\{-1,1\}^4$, and horizon $H=3$. The MDP is parameterized by a vector $\bxi\in\RR^4$, and we construct a feature $\bphi(\cdot,\cdot)$ which depends on $\inner{\bxi,a}$. Among the states, $s_1$ is the initial state, $s_4$ is a fail-state, and $s_5$ is an absorbing state with positive reward. $s_1$ can transition into $s_2,s_4$, and $s_5$, $s_2$ can transition into $s_3,s_4,s_5$, and $s_3$ can transition into $s_4$ and $s_5$. We adversarially select the source and test transition kernels such that in the source domain, the optimal first action is always $a=\{1,1,1,1\}$, whereas in the test domain, the optimal first action is always $a=\{-1,-1,-1,-1\}$. We defer more details on the MDP construction to \Cref{appendix:exp}.

We conduct experiments under different $\rho$ values and set $\|\bxi\|_1=0.3$. For all policy optimization algorithms, i.e. \algo, OPPO, and Optimistic NPG, we choose the uniform policy as the reference policy. For further discussion on the choice of reference policies, see \Cref{appendix:exp}. As shown in Figure \ref{fig:sim_mdp}, 
\alg is more robust to distribution shifts than LSVI-UCB, OPPO, and Optimistic NPG, especially with high perturbation levels. \alg is also on par with or slightly more robust than DR-LSVI-UCB. In addition, \alg substantially outperforms the reference policy, indicating significant policy optimization. We defer experiment details and ablation studies to \Cref{appendix:exp}.

In \Cref{fig:rrmdp}, we also demonstrate the robustness of \alg on RRMDP with different regularization parameters $\sigma\in\{0.1,1.0,2.0\}$. With lower $\sigma$ values, the $\inner{\bphi(s,a),\sigma D(\bmu(\cdot)\|\bmu^0(\cdot))}$ imposes a smaller penalty, and therefore the robust value function considers transitions that are farther from the nominal kernel, resulting in stronger robustness. Under all values of $\sigma$, \alg consistently outperforms LSVI-UCB, OPPO, and Optimistic NPG, and as the value of $\sigma$ decreases, it gradually approaches the performance of the DR-LSVI-UCB baseline under uncertainty level $\rho=0.3$.

\begin{figure*}[t]
    \centering
    \subfigure[$\|\bxi\|_1=0.3$,$\rho=0.1$]{\includegraphics[scale=0.3]{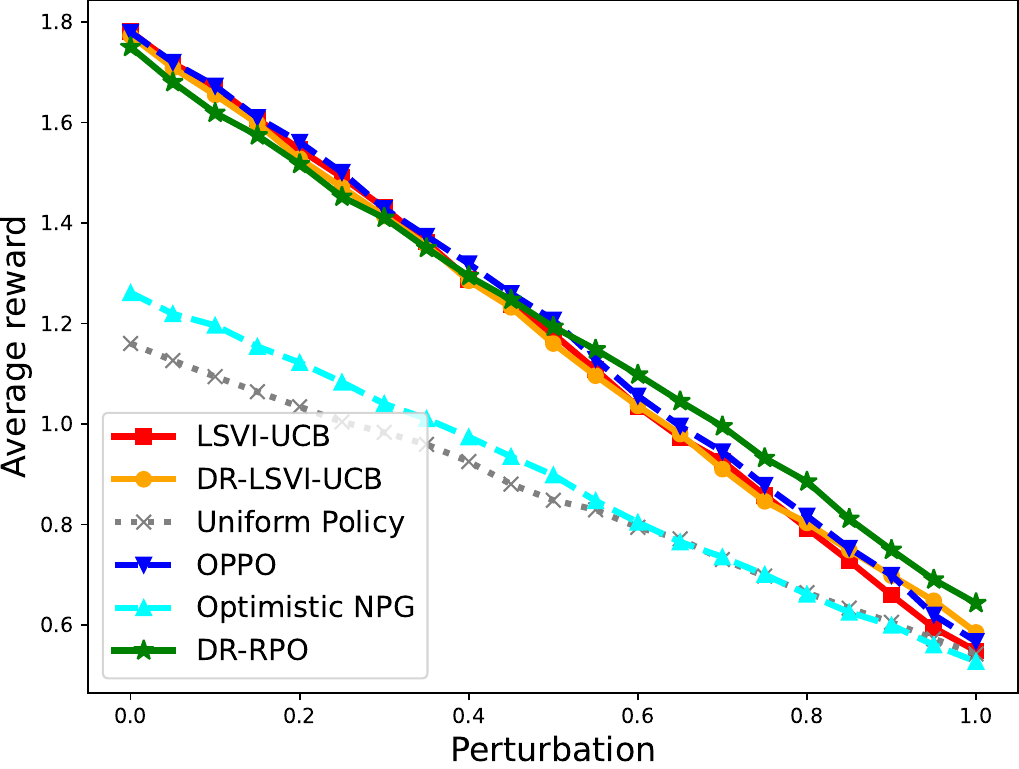}}
    \subfigure[$\|\bxi\|_1=0.3$,$\rho=0.3$]{\includegraphics[scale=0.3]{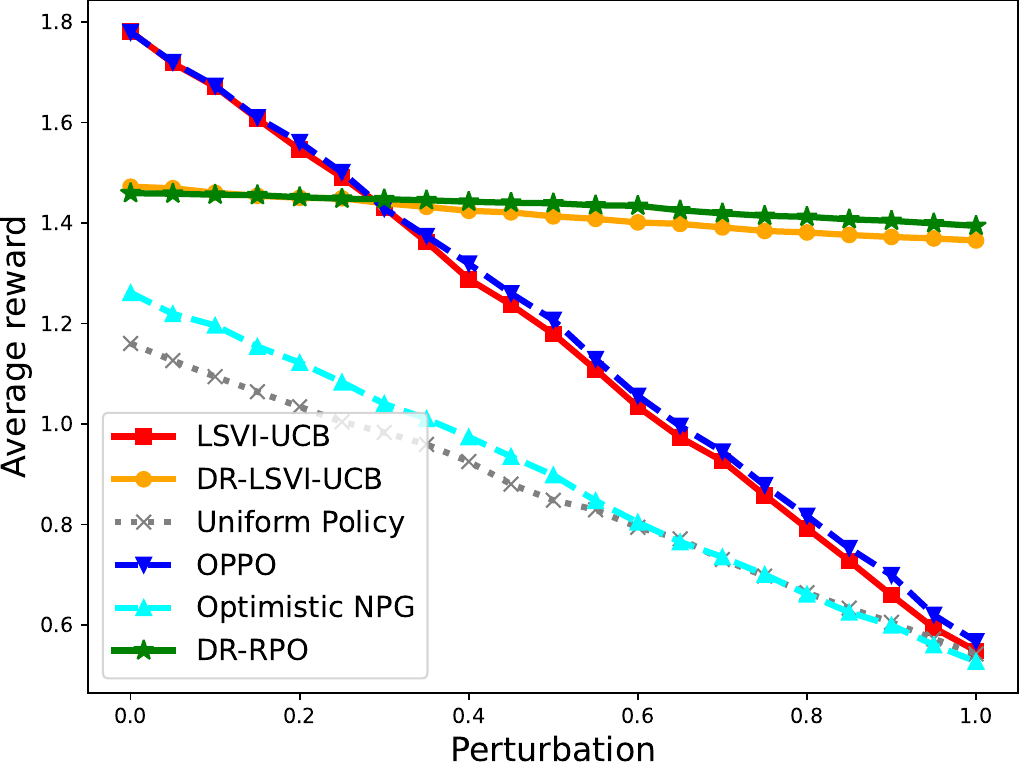}}
    \subfigure[$\|\bxi\|_1=0.3$,$\rho=0.5$]{\includegraphics[scale=0.3]{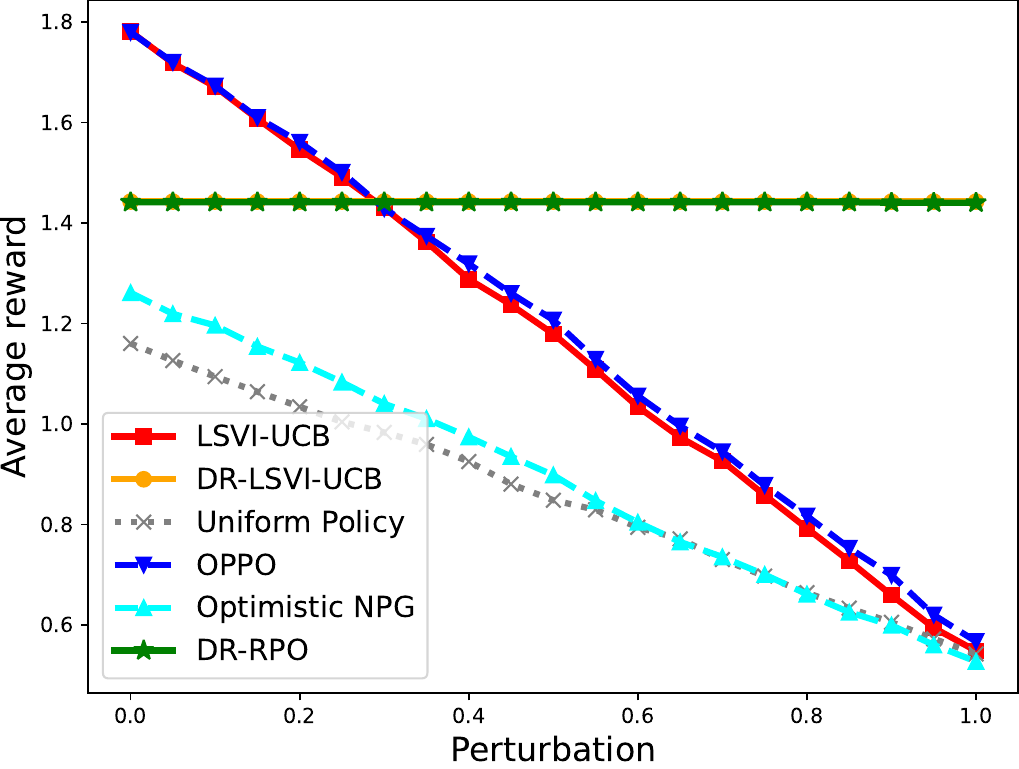}}
    \caption{Experimental results for off-dynamics linear DRMDP.}
    \label{fig:sim_mdp}
\end{figure*}
\begin{figure*}[!h]
    \centering
    \subfigure[$\sigma=0.1$]{\includegraphics[scale=0.3]{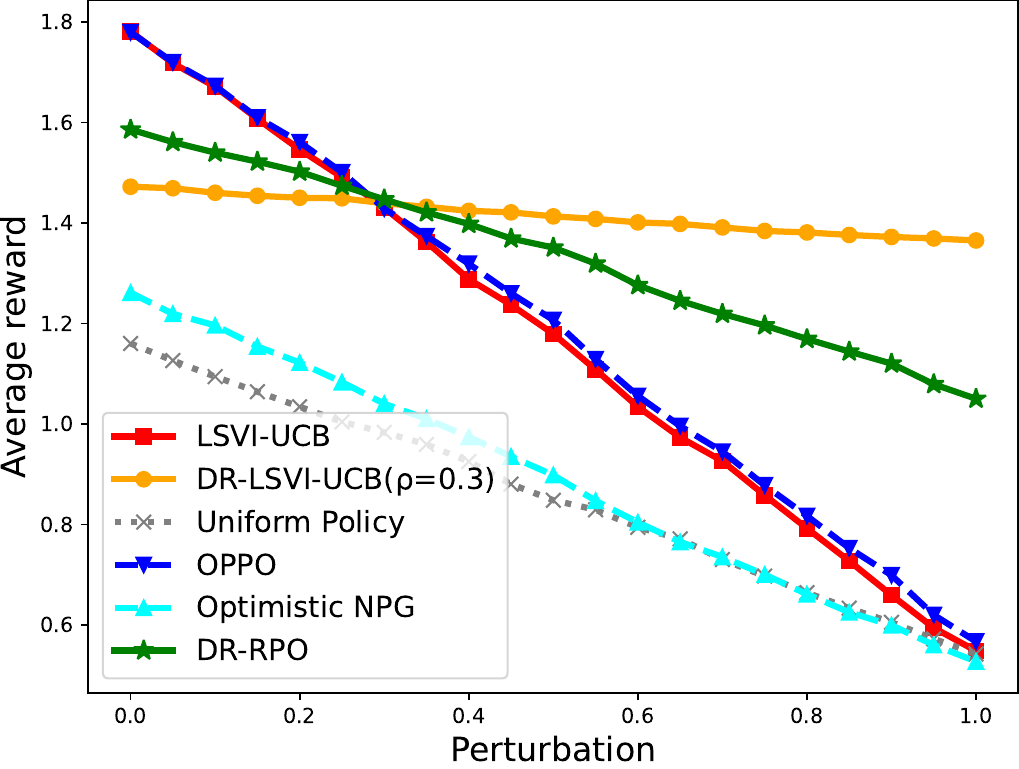}}
    \subfigure[$\sigma=1.0$]{\includegraphics[scale=0.3]{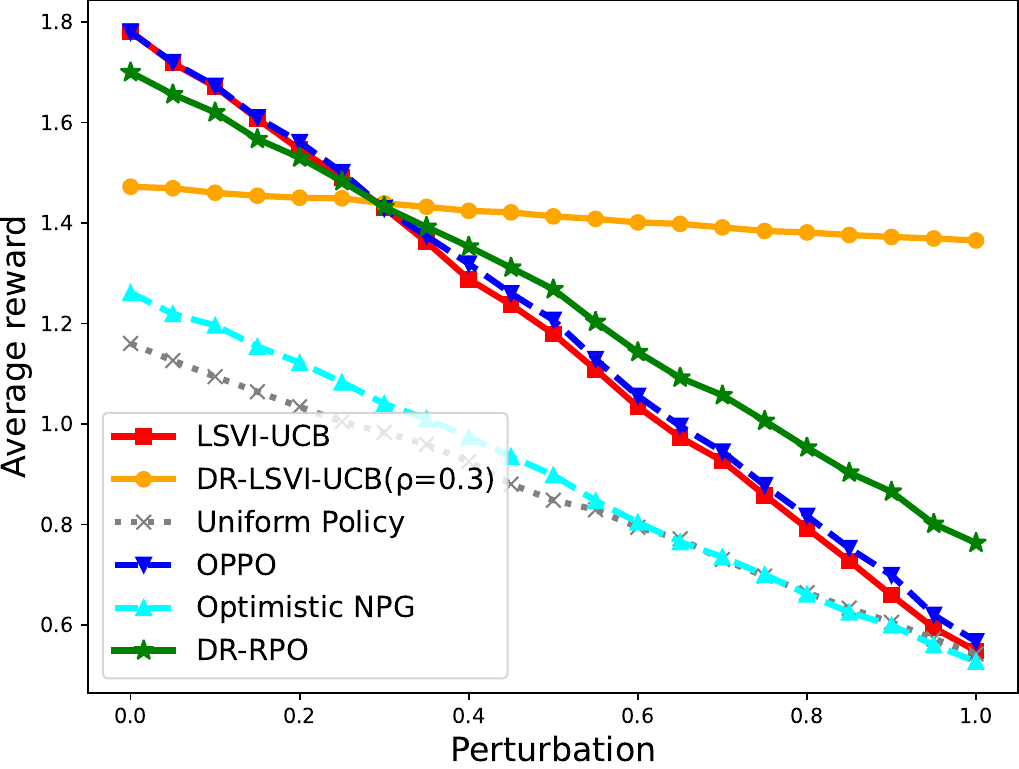}}
    \subfigure[$\sigma=2.0$]{\includegraphics[scale=0.3]{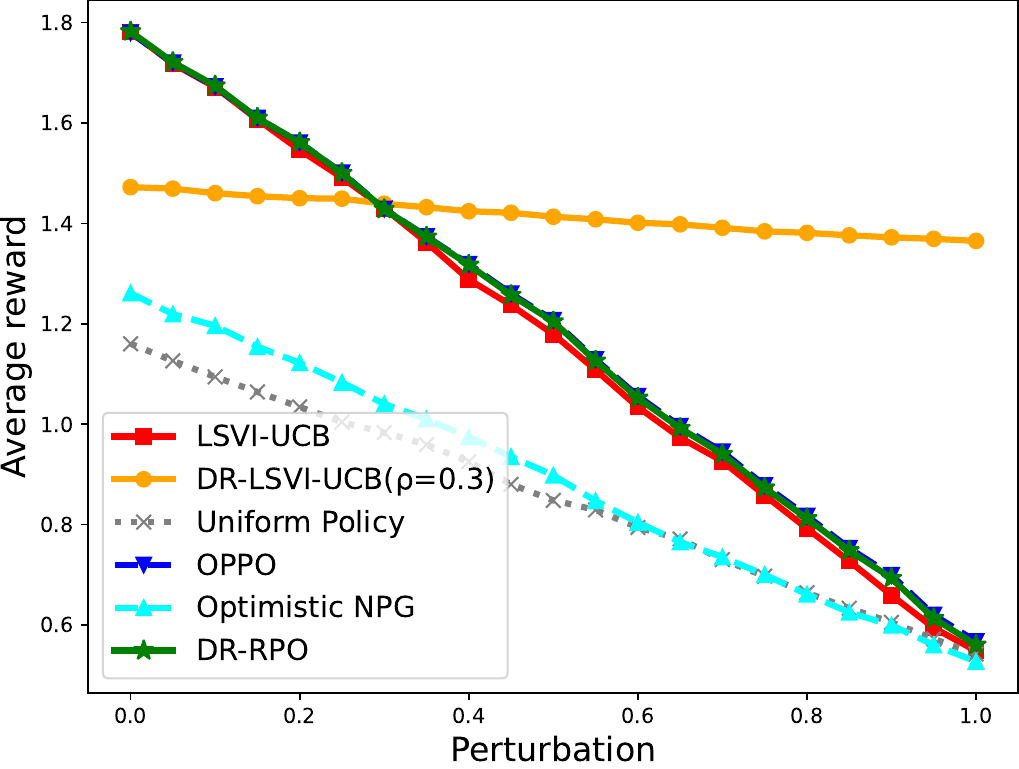}}
    \caption{Experimental results for off-dynamics linear RRMDP.}
    \label{fig:rrmdp}
\end{figure*}

\subsection{American Put Option}

The American Put Option \citep{ma2023distributionallyrobustofflinereinforcement} environment models the stochastic fluctuations of the price of a product, and in each step the agent chooses whether to exercise the option or not, with the objective being to exercise the option at the lowest price. The RMDP has a continuous state space $\cS=\RR^+$ to represent the price and two actions $\cA=\{0,1\}$, where $a=0$ exercises the option while $a=0$ does not. Once $a=0$ is chosen, the trajectory terminates and receives reward $r=\max\{0,100-s_h\}$; otherwise, it transitions to the next state until it reaches the horizon of 10 steps. The price is controlled by a hyperparameter $p\in(0,1)$, which represents the probability that the price increases at each step. The next-state price $s_{h+1}$ follows the distribution
\begin{align*}
s_{h+1}=\left\{\begin{array}{ll}
1.02s_h&\text{w.p. } p,\\
0.98s_h&\text{w.p. } 1-p.
\end{array}\right.
\end{align*}
We follow the same parameters as \citet{ma2023distributionallyrobustofflinereinforcement}; namely, we use the feature mapping
\begin{align*}
\bphi(s,0)=[0,\cdots,0,\max\{0,100-s\}]^\top,\qquad\bphi(s,1)=[\varphi_1(s),\cdots,\varphi_d(s),0]^\top,
\end{align*}
where $\varphi_i(s)=\max\{0,1-|s-s_i|/\Delta\}$. $\{s_i\}_{i=1}^d$ are a sequence of anchor states with $s_1=80$, $s_{i+1}-s_i=\Delta$ and $\Delta=60/d$. 

To construct distribution shift, we set the price-up probability $p=0.5$ in the source domain and sample $p\in[0.15,0.85]$ in the target domain. We then run experiments by training \alg and all baselines using a similar setup as in \Cref{sec:sim_mdp}, i.e. with the uniform policy as reference, step size $\eta=100$, and $K=100$ training episodes. The results on DRMDP and RRMDP are shown in \Cref{fig:apo_drmdp} and \Cref{fig:apo_rrmdp} respectively. On both DRMDP and RRMDP and on all values of $\rho,\sigma$, \alg exhibits more robust performance than the uniform reference policy, LSVI-UCB, and is as robust as DR-LSVI-UCB when the perturbation level is high.

\begin{figure*}[!h]
    \centering
    \subfigure[$\rho=0.1$]{\includegraphics[scale=0.3]{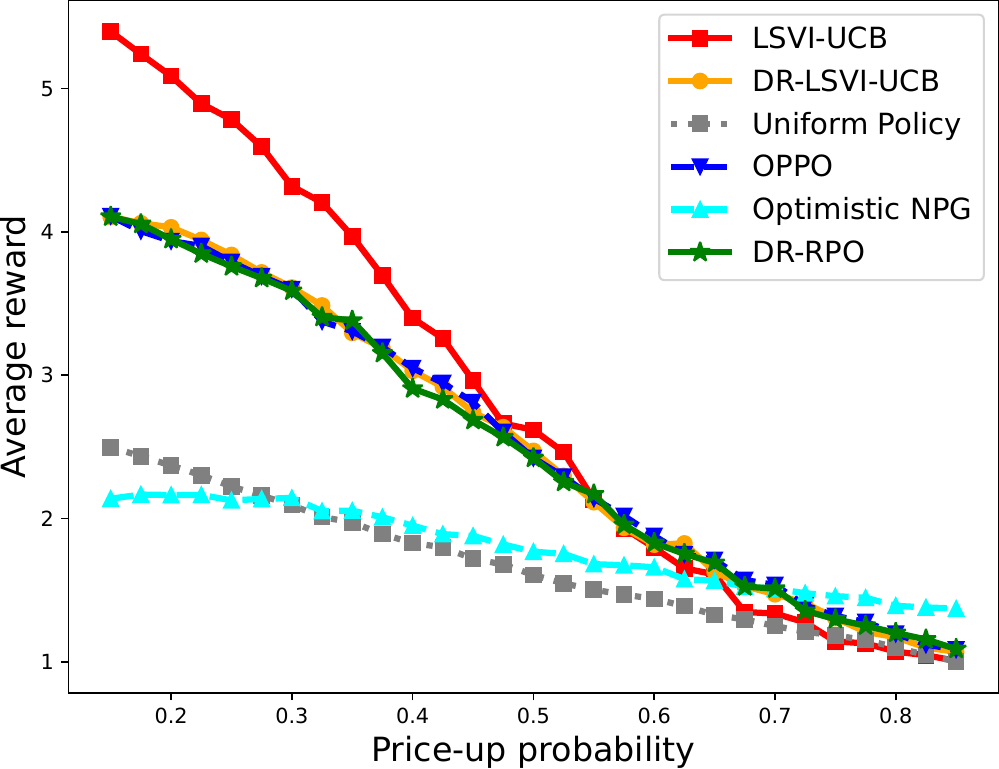}}
    \subfigure[$\rho=0.3$]{\includegraphics[scale=0.3]{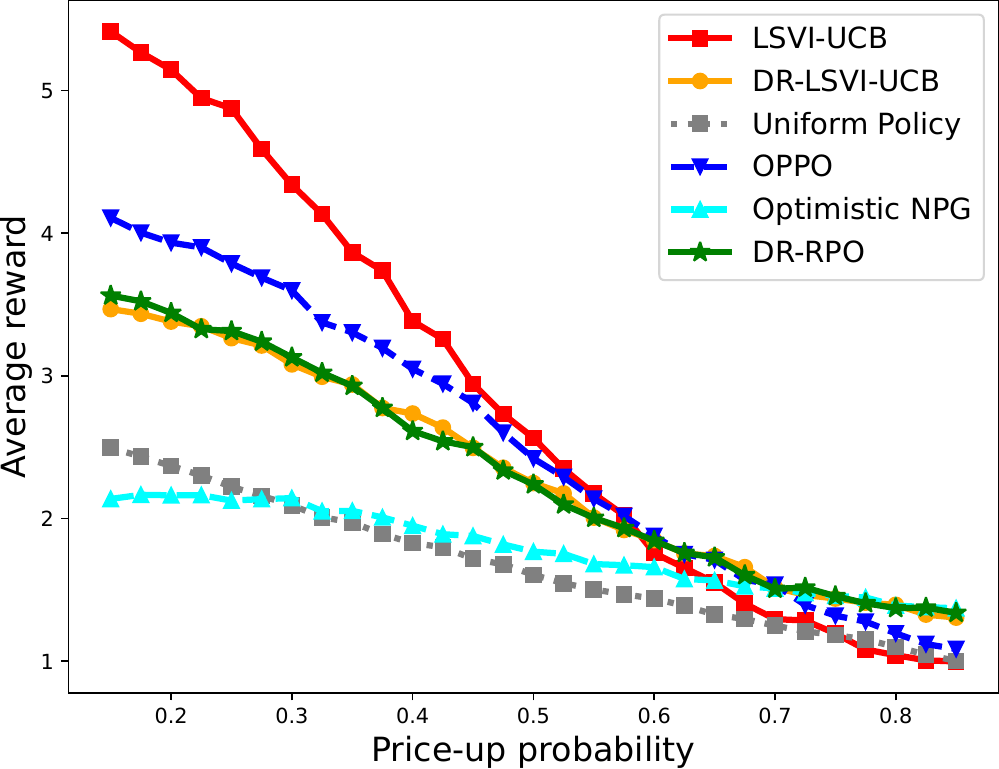}}
    \subfigure[$\rho=0.5$]{\includegraphics[scale=0.3]{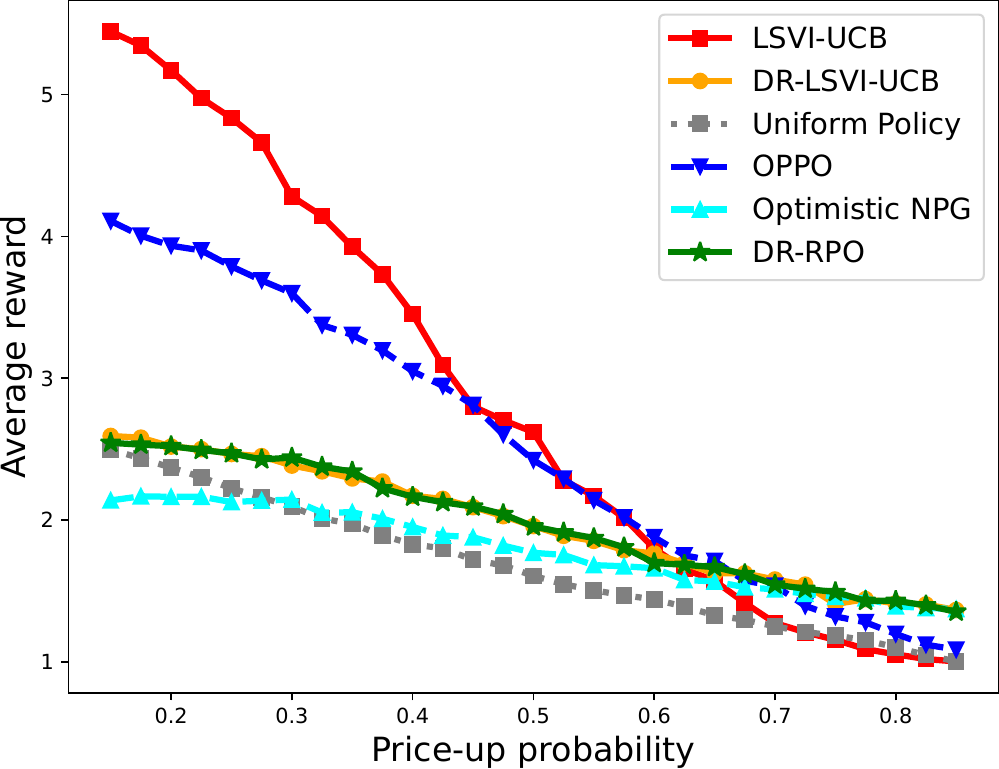}}
    \caption{American Put Option experiments on DRMDP.}
    \label{fig:apo_drmdp}
\end{figure*}
\begin{figure*}[!h]
    \centering
    \subfigure[$\sigma=0.1$]{\includegraphics[scale=0.3]{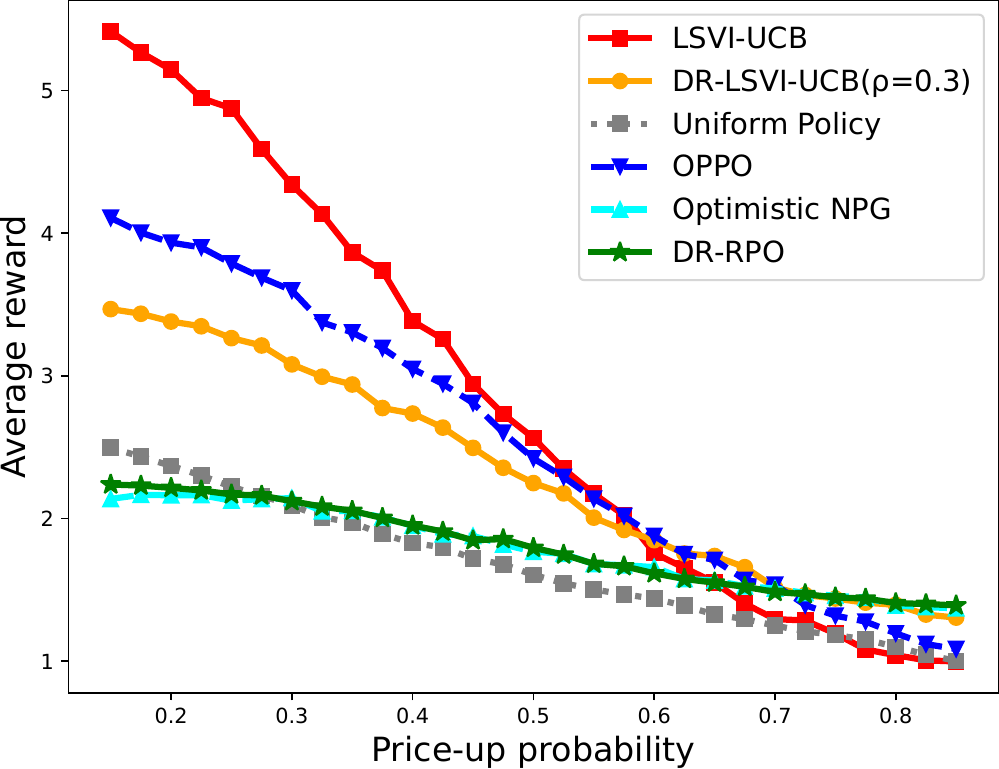}}
    \subfigure[$\sigma=1.0$]{\includegraphics[scale=0.3]{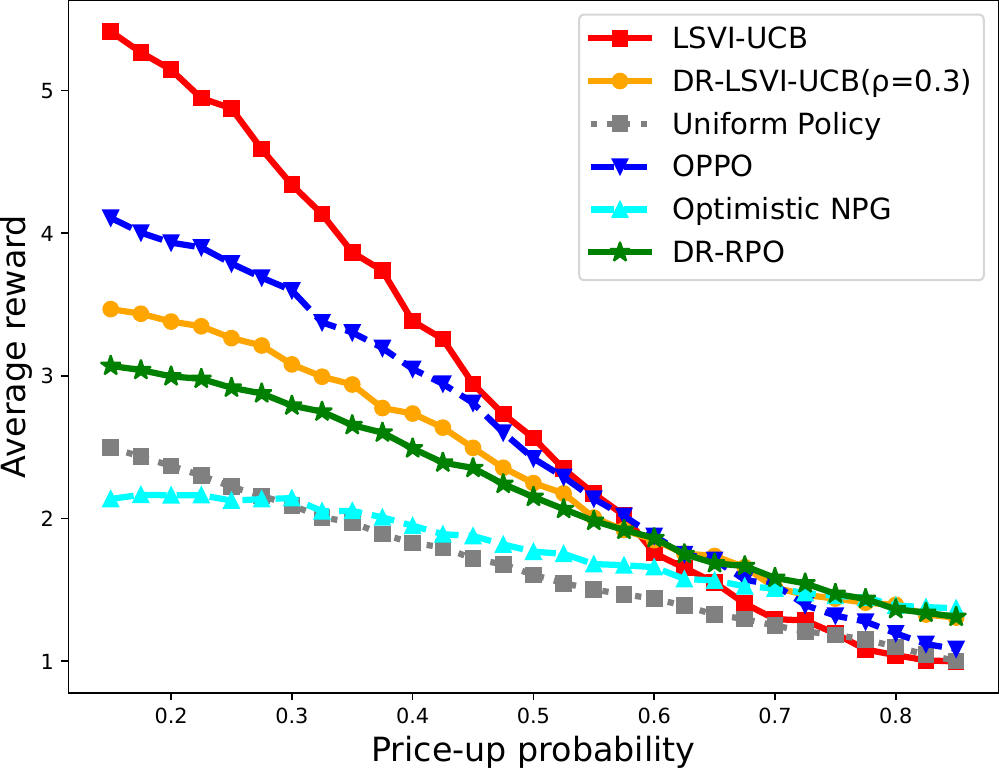}}
    \subfigure[$\sigma=2.0$]{\includegraphics[scale=0.3]{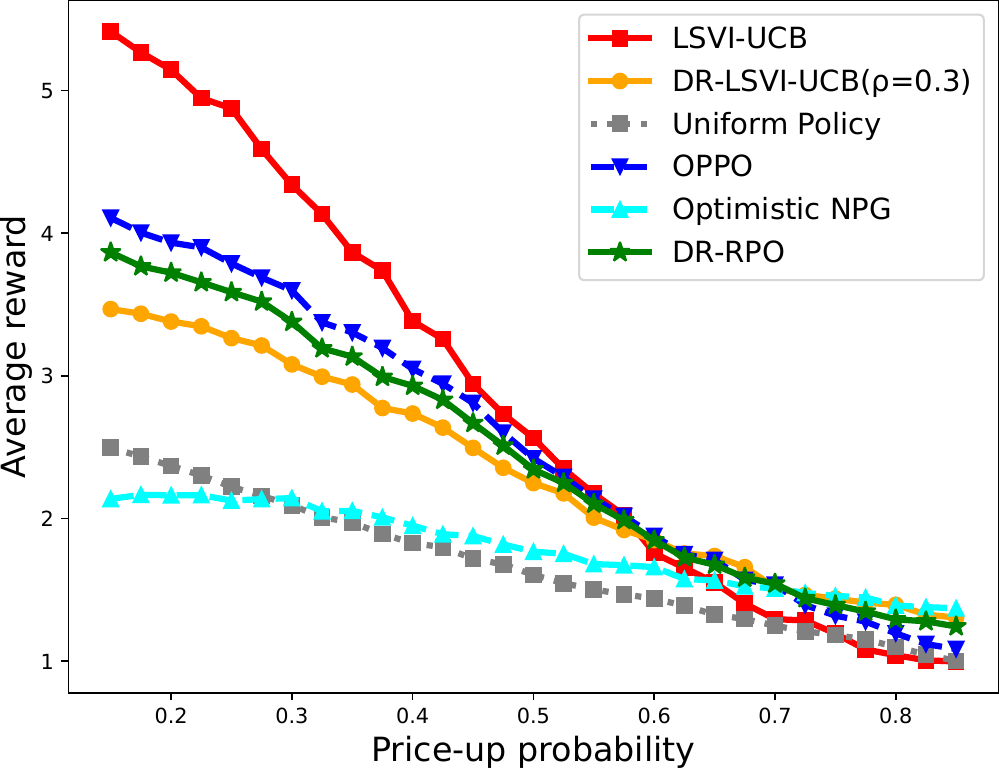}}
    \caption{American Put Option experiments on RRMDP.}
    \label{fig:apo_rrmdp}
\end{figure*}

\section{CONCLUSION}

We studied off-dynamics reinforcement learning for robust Markov decision processes with linear function approximation. We introduced policy-regularized RMDPs, which constrain the transition and policy, enabling stochastic and non-stationary policy optimization within a softmax policy class. We proposed \algo, a model-free algorithm that optimistically explores RMDPs. We studied \alg under $d$-rectangular linear DRMDPs and RRMDPs. Theoretical analysis shows \alg learns a robust policy with a sublinear regret comparable to greedy value-iteration. Numerical experiments validate its performance on diverse environments.

\appendix

\section{Experiment Details and Additional Results} \label{appendix:exp}

\subsection{Off-dynamics Linear MDP} \label{sec:sim_mdp}

In this section, we elaborate on the off-dynamics linear MDP in our numerical experiments. The construction of this MDP follows from \citet{liu2024distributionally}, and we restate it below. 

The MDP has 5 states $\cS=\{s_1,s_2,s_3,s_4,s_5\}$, action space $\cA=\{-1,1\}^4$, and horizon $H=3$, with hyperparameters $\zeta,p,q\in(0,1)$ and $\bxi\in\RR^4$. It satisfies \Cref{assumption:linearMDP} with dimension $d=4$. We construct the feature mapping $\bphi$ as follows:
\begin{align*}
\bphi(s_1,a)&=(1-\zeta-\inner{\bxi,a},0,0,\zeta+\inner{\bxi,a})^\top,\\
\bphi(s_2,a)&=(0,1-\zeta-\inner{\bxi,a},0,\zeta+\inner{\bxi,a})^\top,\\
\bphi(s_3,a)&=(0,0,1-\zeta-\inner{\bxi,a},\zeta+\inner{\bxi,a})^\top,\\
\bphi(s_4,a)&=(0,0,1,0)^\top,\\
\bphi(s_5,a)&=(0,0,0,1)^\top.
\end{align*}
The reward parameters are set to be
\begin{align*}
\btheta_1=(0,0,0,0)^\top,\qquad \btheta_2=\btheta_3=(0,0,0,1)^\top.
\end{align*}
We further define the factor distributions $\bmu_1,\bmu_2$ to set the transition kernels for the MDP. To construct an off-dynamics setting, we define the factor distributions differently between the source and target domains. On the source domain, we define
\begin{align*}
\bmu_1=\bmu_2=((1-p)\delta_{s_1}+p\delta_{s_4},(1-p)\delta_{s_3}+p\delta_{s_4},\delta_{s_4},\delta_{s_5})^\top,
\end{align*}
where $\delta_s$ is the Dirac measure with positive probability mass on $s$. On the target domain, $\bmu_2$ remains unchanged while $\bmu_1$ is altered, resulting in a perturbed factor distribution $\bmu_1'$:
\begin{align*}
\bmu_1'=(\delta_{s_2},\delta_{s_3},\delta_{s_4},(1-q)\delta_{s_5}+q\delta_{s_4}).
\end{align*}
The transitions between the source and target domains are therefore different, as illustrated in \Cref{fig:mdp}. In particular, the distribution shift increases as the hyperparameter $q$ increases, and hence we refer to $q$ as the perturbation level. 

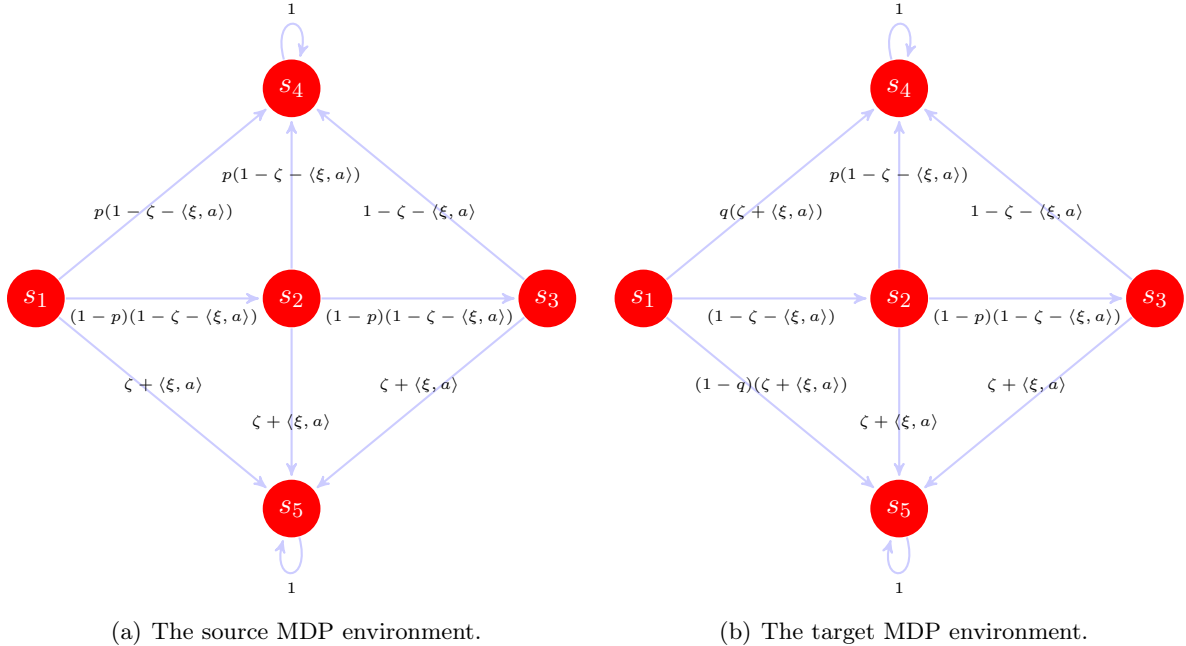
\begin{figure*}[ht]
    \centering
    \subfigure[The source MDP environment.]{
        \begin{tikzpicture}[->,>=stealth',shorten >=1pt,auto,node distance=3.4cm,thick]
            \tikzstyle{every state}=[fill=red,draw=none,text=white,minimum size=0.5cm]
            \node[state] (S1) {$s_1$};
            \node[state] (S2) [right of=S1] {$s_2$};
            \node[state] (S3) [right of=S2] {$s_3$};
            \node[state] (S4) [above=2cm of S2] {$s_4$};
            \node[state] (S5) [below=2cm of S2] {$s_5$};
            
            \path   (S1) edge[draw=blue!20] node[below] {\tiny $(1-p)(1-\zeta-\la\xi,a\ra)$} (S2)
                         edge[draw=blue!20] node[below] {\tiny$p(1-\zeta-\la\xi,a\ra)$} (S4)
                         edge[draw=blue!20] node[above] {\tiny$\zeta+\la\xi,a\ra$} (S5)
                    (S2) edge[draw=blue!20] node[below] {\tiny$(1-p)(1-\zeta-\la\xi,a\ra)$} (S3)
                         edge[draw=blue!20] node[above] {\tiny$p(1-\zeta-\la\xi,a\ra)$} (S4)
                         edge[draw=blue!20] node[below] {\tiny$\zeta+\la\xi,a\ra$} (S5)
                    (S3) edge[draw=blue!20] node[below] {\tiny$1-\zeta-\la\xi,a\ra$} (S4)
                         edge[draw=blue!20] node[above] {\tiny$\zeta+\la\xi,a\ra$} (S5)
                    (S4) edge[draw=blue!20] [loop above] node {\tiny 1} (S4)
                    (S5) edge[draw=blue!20] [loop below] node {\tiny 1} (S5);
        \end{tikzpicture}
        \label{fig:mdp_5states}
    }
    \subfigure[The target MDP environment.]{
        \begin{tikzpicture}[->,>=stealth',shorten >=1pt,auto,node distance=3.4cm,thick]
            \tikzstyle{every state}=[fill=red,draw=none,text=white,minimum size=0.5cm]
            \node[state] (S1) {$s_1$};
            \node[state] (S2) [right of=S1] {$s_2$};
            \node[state] (S3) [right of=S2] {$s_3$};
            \node[state] (S4) [above=2cm of S2] {$s_4$};
            \node[state] (S5) [below=2cm of S2] {$s_5$};
            
            \path   (S1) edge[draw=blue!20] node[below] {\tiny$(1-\zeta-\la\xi,a\ra)$} (S2)
                         edge[draw=blue!20] node[below] {\tiny$q(\zeta+\la\xi,a\ra)$} (S4)
                         edge[draw=blue!20] node[above] {\tiny$(1-q)(\zeta+\la\xi,a\ra)$} (S5)
                    (S2) edge[draw=blue!20] node[below] {\tiny$(1-p)(1-\zeta-\la\xi,a\ra)$} (S3)
                         edge[draw=blue!20] node[above] {\tiny$p(1-\zeta-\la\xi,a\ra)$} (S4)
                         edge[draw=blue!20] node[below] {\tiny$\zeta+\la\xi,a\ra$} (S5)
                    (S3) edge[draw=blue!20] node[below] {\tiny$1-\zeta-\la\xi,a\ra$} (S4)
                         edge[draw=blue!20] node[above] {\tiny$\zeta+\la\xi,a\ra$} (S5)
                    (S4) edge[draw=blue!20] [loop above] node {\tiny 1} (S4)
                    (S5) edge[draw=blue!20] [loop below] node {\tiny 1} (S5);
        \end{tikzpicture}
        \label{fig:perturbed_mdp}
    }
    \caption{The source and the target linear MDP environments. The value on each arrow represents the transition probability. }
    \label{fig:mdp}
\end{figure*}

\paragraph{Parameter settings} In our experiments on \alg and on all baselines, we train over $K=100$ episodes. For Optimistic NPG, which has a more complex sampling scheme, we re-sample every $m=10$ episodes and collect a batch of $N=30$ trajectories, divided into disjoint sub-batches of 10 trajectories for each time-step. For more details on the sampling schedule of Optimistic NPG, please refer to \citet{liu2023optimisticnaturalpolicygradient}.

We sample different perturbation levels $q\in[0,1]$, and set the hyperparameters $\zeta=0.3,p=0.001$, and $\bxi=(\|\bxi\|_1/4,\|\bxi\|_1/4,\|\bxi\|_1/4,\|\bxi\|_1/4)$. Unless otherwise noted, we set $\|\bxi\|_1=0.3$ and $\eta=100$, where $\eta$ is the policy optimization step-size parameter. As mentioned in \Cref{sec:exp}, we use the uniform policy as the reference policy. In subsequent sections, we conduct ablation studies with different values of $\|\bxi\|_1$, $\eta$, and $\piref$.

\subsection{Ablation Studies}

\paragraph{Factor distributions.} The RMDP is parameterized by the vector $\bxi$, and perturbing its norm results in different factor distributions, which in turn lead to different transition kernels. Hence, we studied the behavior of \alg with different factor distributions by selecting $\|\bxi\|_1\in\{0.1,0.2,0.3\}$. The results are shown in \Cref{fig:ablation_xi}. We observe that across different values of $\|\bxi\|_1$, \alg consistently outperforms the uniform policy and Optimistic NPG, and is either on par with or exceeds the performance of LSVI-UCB, DR-LSVI-UCB, and OPPO. With larger $\|\bxi\|_1$, the source domain deviates further from the target domain, and the robustness of \alg becomes more pronounced compared with non-robust baselines.

\begin{figure*}[!h]
    \centering
    \subfigure[$\|\bxi\|_1=0.1$,$\rho=0.1$]{\includegraphics[scale=0.3]{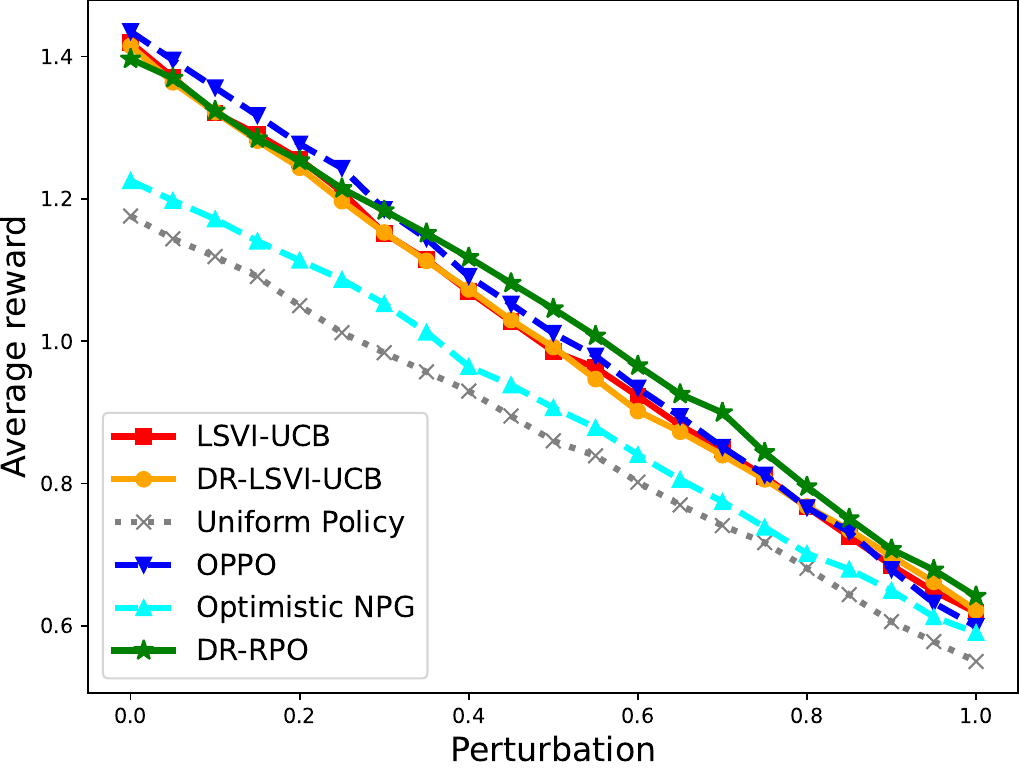}}
    \subfigure[$\|\bxi\|_1=0.1$,$\rho=0.3$]{\includegraphics[scale=0.3]{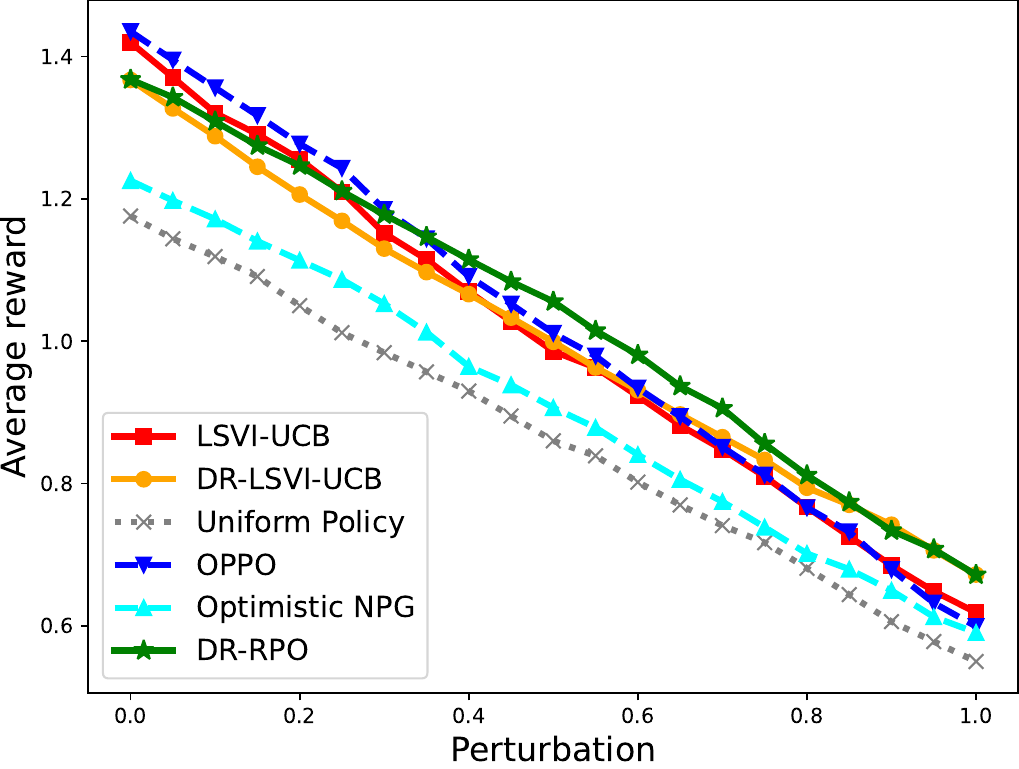}}
    \subfigure[$\|\bxi\|_1=0.1$,$\rho=0.5$]{\includegraphics[scale=0.3]{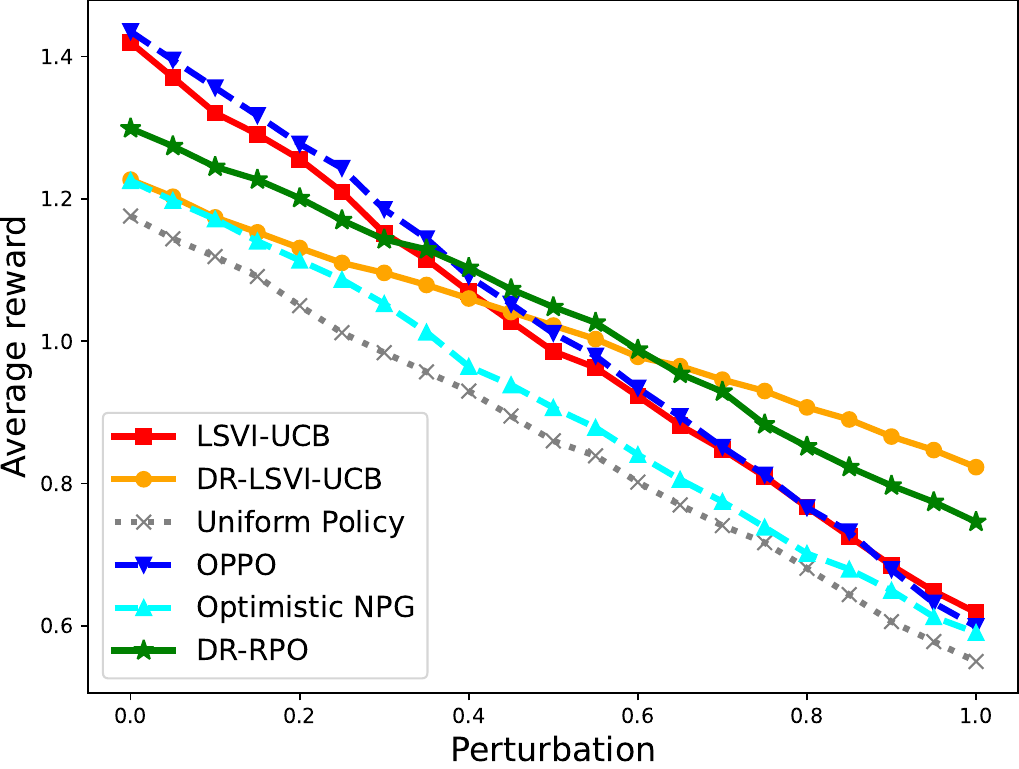}}
    \subfigure[$\|\bxi\|_1=0.2$,$\rho=0.1$]{\includegraphics[scale=0.3]{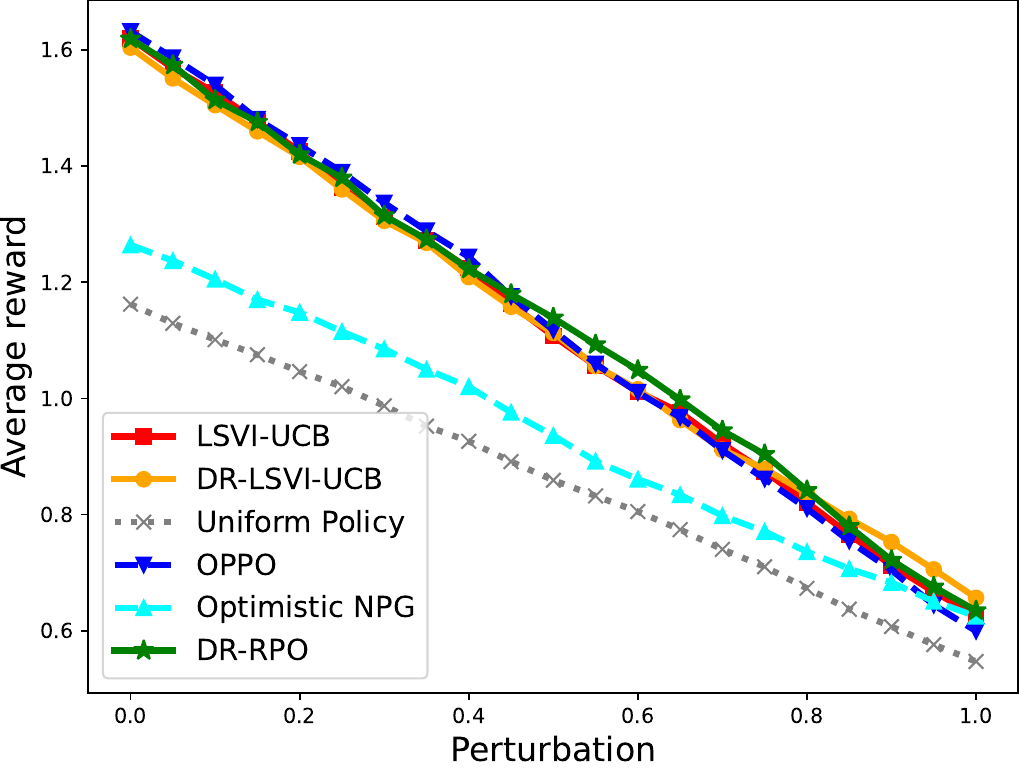}}
    \subfigure[$\|\bxi\|_1=0.2$,$\rho=0.3$]{\includegraphics[scale=0.3]{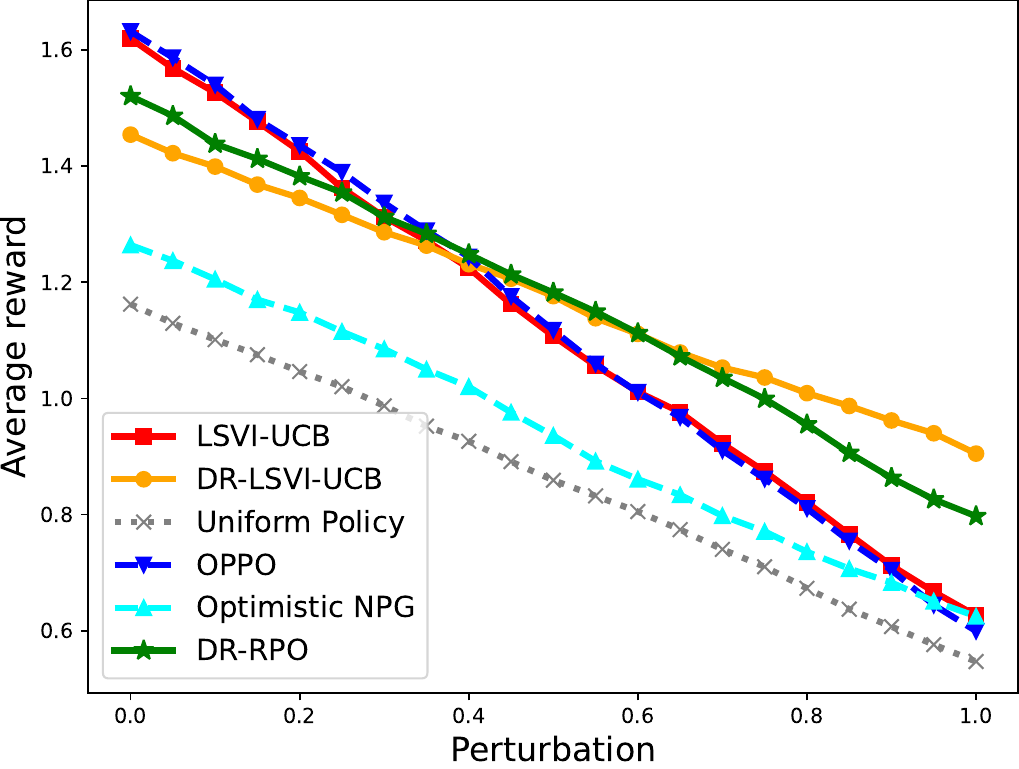}}
    \subfigure[$\|\bxi\|_1=0.2$,$\rho=0.5$]{\includegraphics[scale=0.3]{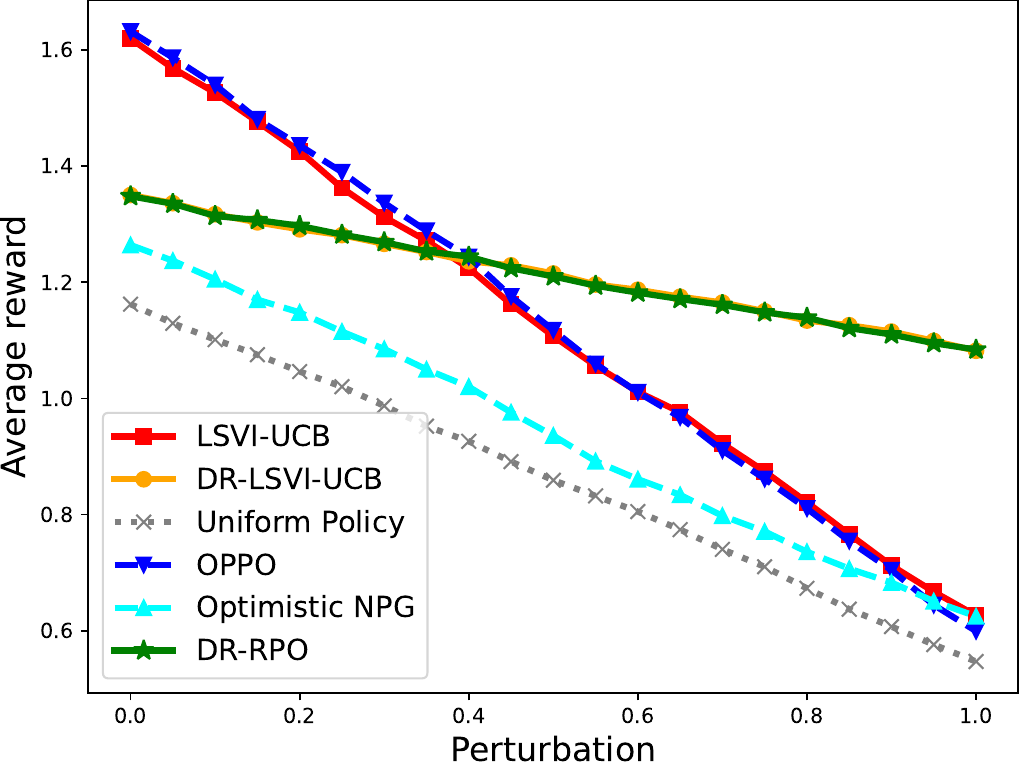}}
    \subfigure[$\|\bxi\|_1=0.3$,$\rho=0.1$]{\includegraphics[scale=0.3]{figures/sim_mdp_delta=0.3_xi=0.3_rho=1.pdf}}
    \subfigure[$\|\bxi\|_1=0.3$,$\rho=0.3$]{\includegraphics[scale=0.3]{figures/sim_mdp_delta=0.3_xi=0.3_rho=3.pdf}}
    \subfigure[$\|\bxi\|_1=0.3$,$\rho=0.5$]{\includegraphics[scale=0.3]{figures/sim_mdp_delta=0.3_xi=0.3_rho=5.pdf}}
    \caption{Ablation on factor distributions.}
    \label{fig:ablation_xi}
\end{figure*}

\paragraph{Reference policy.} Different reference policies can affect the performance of policy optimization algorithms. As discussed in \citet{liu2024distributionally}, the RMDP design results in different optimal policies in the source and target domains. In the source domain, $(1,1,1,1)$ is the optimal first action; in the target domain, as $q$ increases, the $(-1,-1,-1,-1)$ action gradually gains advantage and eventually becomes the optimal first action. 
Therefore, in addition to the uniform reference policy studied in \Cref{sec:exp}, which we denote by $\pi^{0}$, we introduce two additional reference policies: $\pi^1$, which always chooses the $(-1,-1,-1,-1)$ action in the first step and stays uniformly random in the second step; and $\pi^2$, which assigns 0.5 probability to both the $(1,1,1,1)$ action and the $(-1,-1,-1,-1)$ action in the first step, and then stays uniformly random in the second step. $\pi^0$ represents a completely neutral reference policy, $\pi^1$ is biased towards the optimal policy in the target domain, and $\pi^2$ lies in-between.

As shown in \Cref{fig:ablation_piref}, more robust reference policies generally improves the robustness of \alg on target domains. This improvement is more pronounced when the uncertainty level $\rho$ is low, and becomes much more marginal when $\rho$ reaches 0.3 or 0.5. Regardless of the reference policy, \alg still consistently outperforms Optimistic NPG, OPPO, and LSVI-UCB in terms of robustness, and is on par with DR-LSVI-UCB.

\begin{figure*}[!h]
    \centering
    \subfigure[$\rho=0.1$]{\includegraphics[scale=0.275]{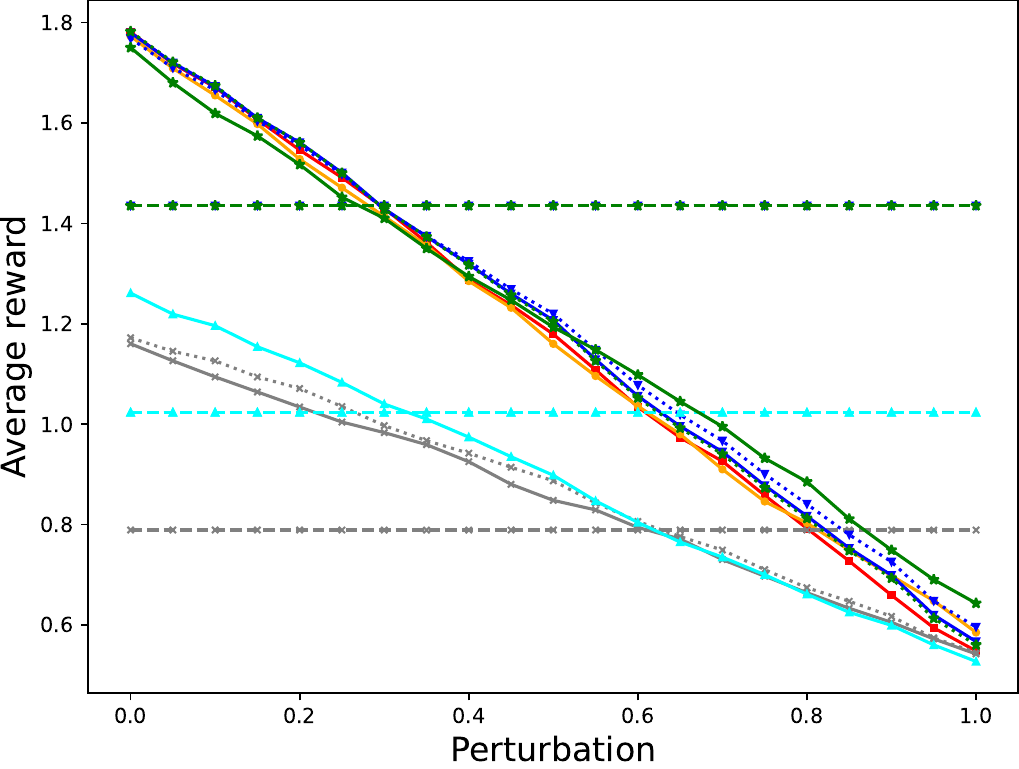}}
    \subfigure[$\rho=0.3$]{\includegraphics[scale=0.275]{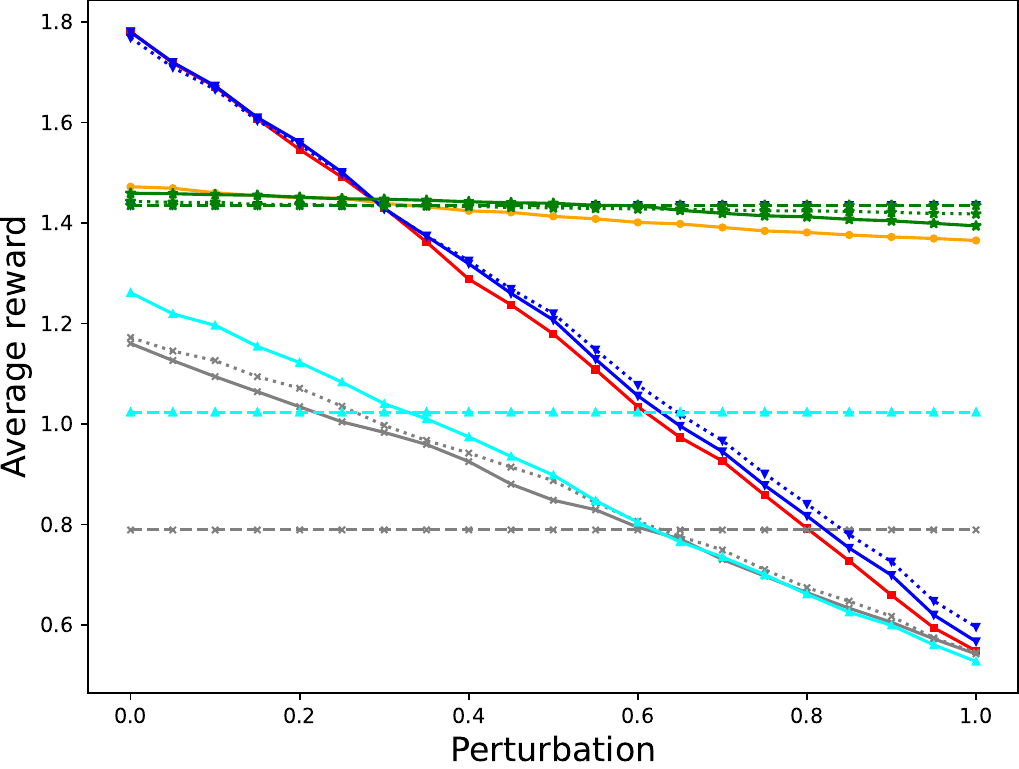}}
    \subfigure[$\rho=0.5$]{\includegraphics[scale=0.275]{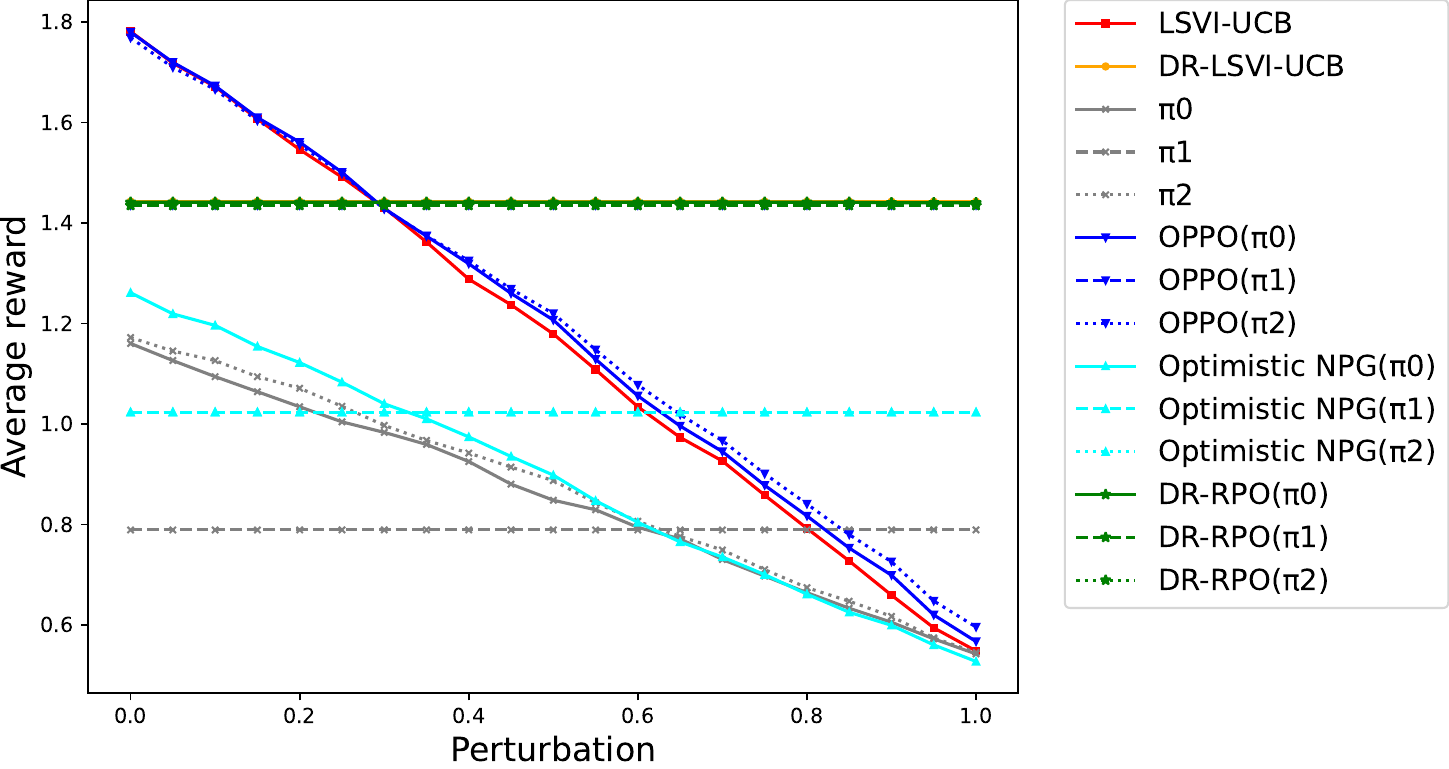}}
    \caption{Ablation on reference policies.}
    \label{fig:ablation_piref}
\end{figure*}

\paragraph{Step size.} The $\eta$ parameter controls the step size of policy improvement and the degree to which the value function penalizes discrepancy from the reference policy. As shown in \Cref{fig:ablation_eta}, with higher $\eta$, \alg deviates further from the reference policy and achieves better performance. Under all selected values of $\eta$, \alg is more robust than OPPO, Optimistic NPG, and LSVI-UCB, and as $\eta$ increases, \alg gradually approaches and even surpasses the performance of DR-LSVI-UCB.

\begin{figure*}[!h]
    \centering
    \subfigure[$\rho=0.1$]{\includegraphics[scale=0.27]{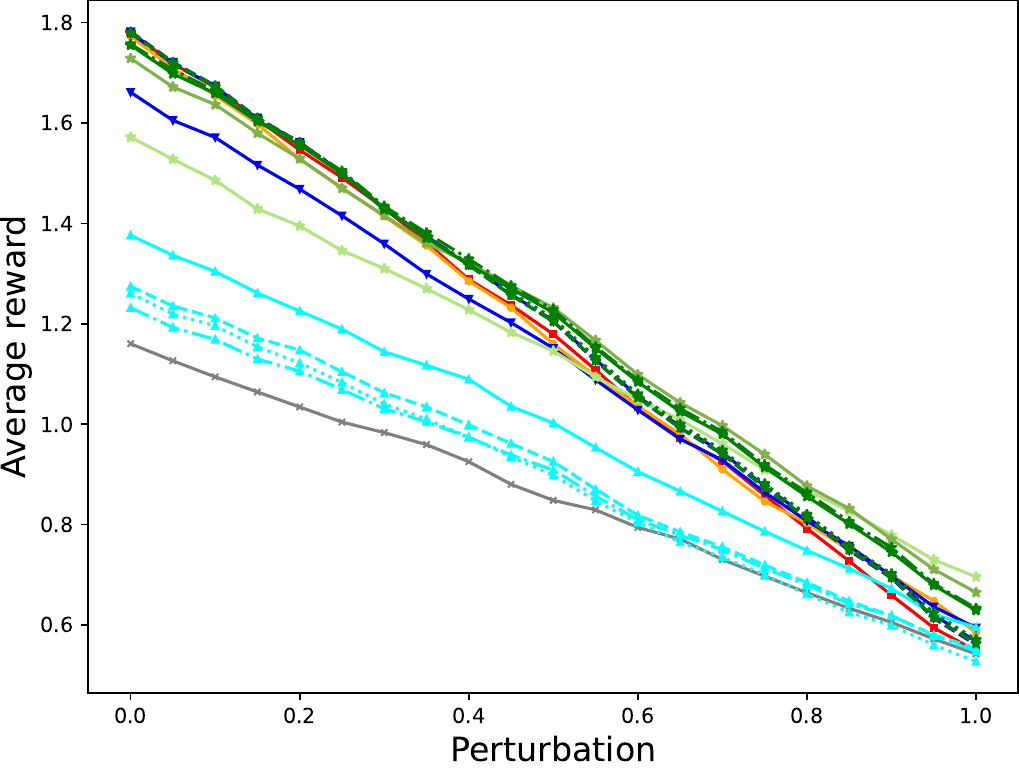}}
    \subfigure[$\rho=0.3$]{\includegraphics[scale=0.27]{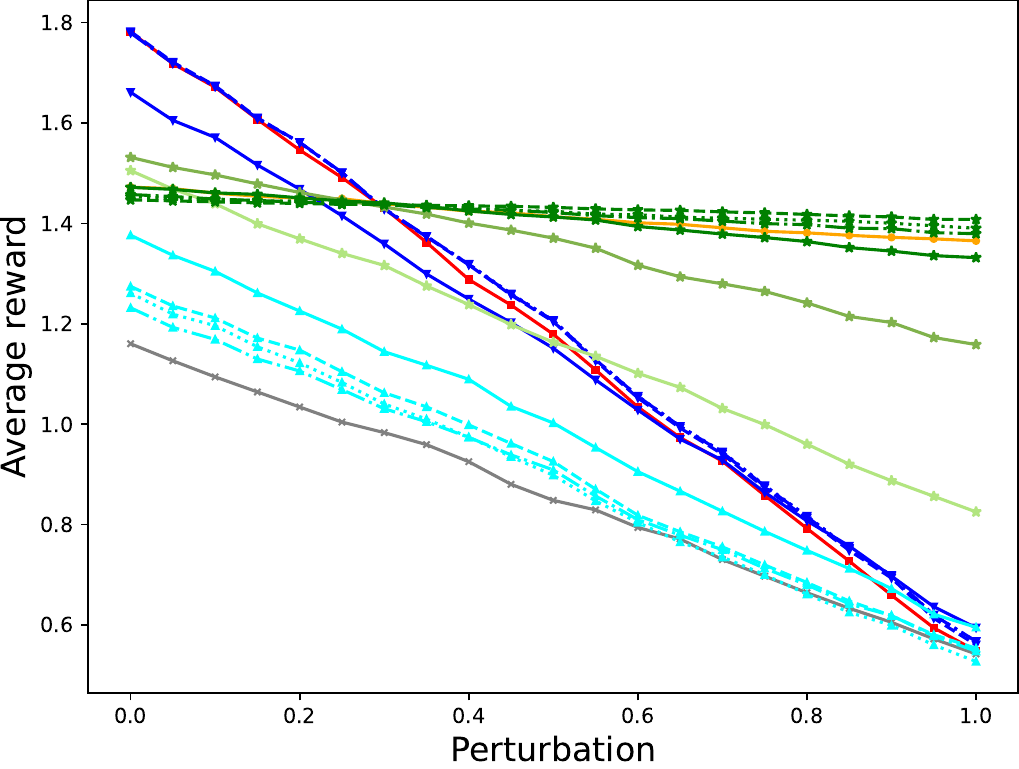}}
    \subfigure[$\rho=0.5$]{\includegraphics[scale=0.27]{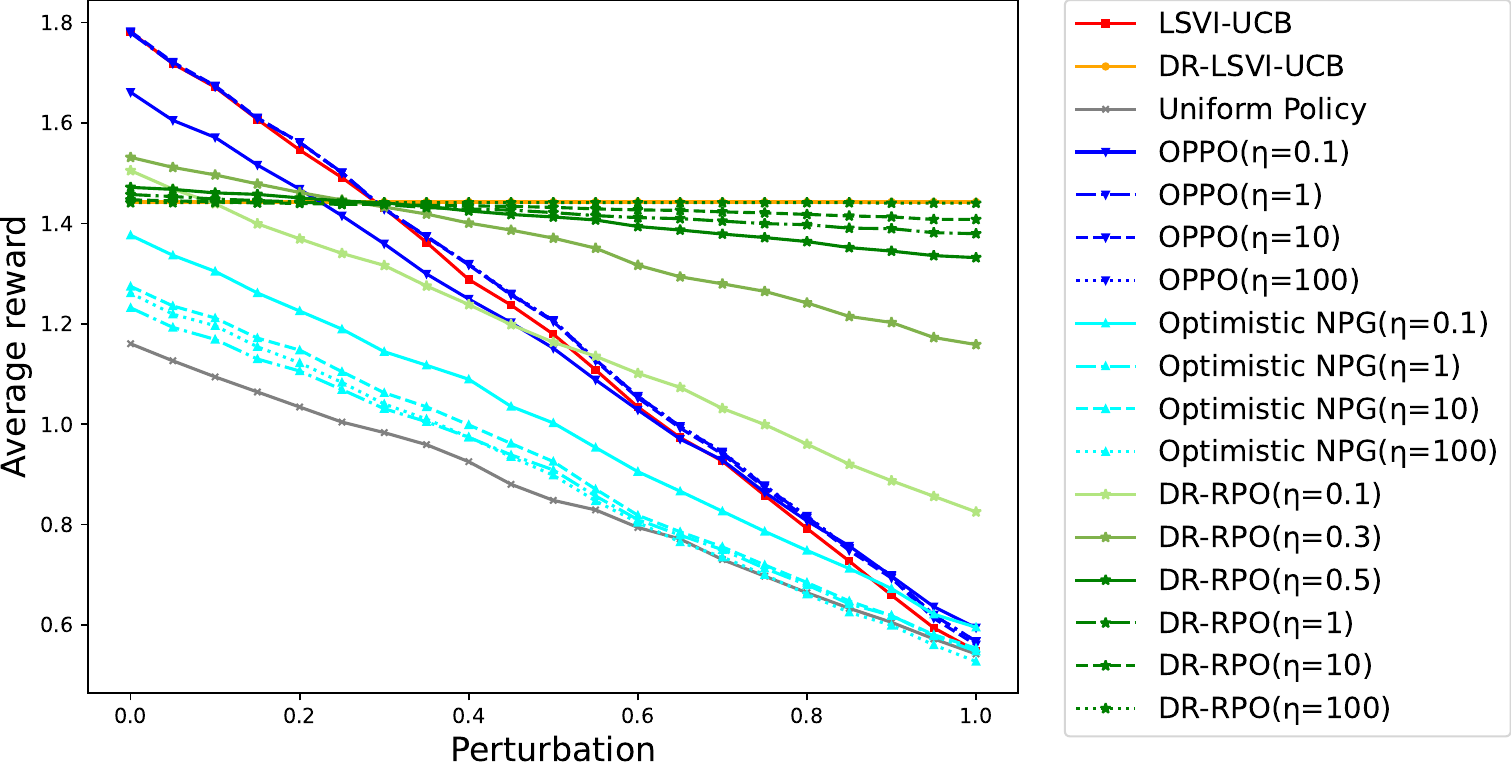}}
    \caption{Ablation on step size.}
    \label{fig:ablation_eta}
\end{figure*}

\section{Dynamic Programming Principles}

\subsection{Proof of \Cref{proposition:robust_bellman_kl_drmdp}}
To prove the robust Bellman equation for the policy-regularized $d$-rectangular linear DRMDP, we consider the following stronger proposition: there exists transition kernels $\bar{P}^\pi=\{\bar{P}_h^\pi\}_{h=1}^H$ where $\bar{P}_h^\pi\in\cU_h^{\rho}(P_h^0)$ such that the robust Bellman equations hold, i.e.
\begin{align}
\tilde{Q}_h^{\pi,\rho}(s)&=r_h(s,a)+\inf_{P_h(\cdot|s,a)\in\cU_h^\rho(s,a,;\bmu_h^0)}\EE_{s'\sim P_h(\cdot|s,a)}\tilde{V}_{h+1}^{\pi,\rho}(s'),\label{eq:bellman_drmdp_q}\\
\tilde{V}_h^{\pi,\rho}(s)&=\innerA{\tilde{Q}_h^{\pi,\rho}(s,\cdot), \pi_h(\cdot|s)}-\frac{1}{\eta}D_{KL}[\pi_h(\cdot|s)\Vert\piref_h(\cdot|s)],\label{eq:bellman_drmdp_v}
\end{align}
and the robust value functions are equal to the value functions under transitions $\bar{P}^\pi$:
\begin{align}
\tilde{Q}_h^{\pi,\rho}(s)&=\tilde{Q}_h^{\pi,\{\bar{P}_t^\pi\}_{t=1}^H}(s),\label{eq:pbar_drmdp_q}\\
\tilde{V}_h^{\pi,\rho}(s)&=\tilde{V}_h^{\pi,\{\bar{P}_t^\pi\}_{t=1}^H}(s).\label{eq:pbar_drmdp_v}
\end{align}
\begin{proof}
We prove this stronger proposition by induction. At the last stage $H$, the base case holds trivially since there is no transition involved. If the inductive hypothesis holds for step $h+1$, then there exist transition kernels $\{\bar{P}_t^\pi\}_{t=h+1}^H$ such that
\begin{align} \label{eq:drmdp_bellman_proof_ih}
\tilde{V}_{h+1}^{\pi,\rho}(s)=\tilde{V}_{h+1}^{\pi,\{\bar{P}_t^\pi\}_{t=h+1}^H}(s).
\end{align}
By definition of $\tilde{Q}^{\pi,\rho}_h$ in \eqref{eq:drmdp_q}, for any $(s,a)\in\cS\times\cA$, we have
\begin{align}
&\tilde{Q}^{\pi,\rho}_h(s,a) \notag\\
=&r_h(s,a)+\infPP\EE_{\{P_t\}_{t=h}^H}\Bigg[\sum_{t=h+1}^H\bigg(r_t(s_t,a_t)-\frac{1}{\eta}D_{KL}[\pi_t(\cdot|s_t)\Vert\piref_t(\cdot|s_t)]\bigg)\Bigg|s_h=s,a_h=a,\pi\Bigg]\\
=&r_h(s,a)+\inf_{P_t\in\cU_t^\rho(P_t^0),h\leq t\leq H}\EE_{\{P_t\}_{t=h}^H}\Bigg[\sum_{t=h+1}^H\bigg(r_t(s_t,a_t)-\frac{1}{\eta}D_{KL}[\pi_t(\cdot|s_t)\Vert\piref_t(\cdot|s_t)]\bigg)\Bigg|s_h=s,a_h=a,\pi\Bigg]\\
=&r_h(s,a)+\inf_{P_t\in\cU_t^\rho(P_t^0),h\leq t\leq H}\int_{\cS}P_h(ds'|s,a) \notag\\
&\qquad\EE_{\{P_t\}_{t=h+1}^H}\Bigg[\sum_{t=h+1}^H\bigg(r_t(s_t,a_t)-\frac{1}{\eta}D_{KL}[\pi_t(\cdot|s_t)\Vert\piref_t(\cdot|s_t)]\bigg)\Bigg|s_{h+1}=s',\pi\Bigg]\\
\leq&r_h(s,a)+\inf_{P_h(\cdot|s,a)\in\cU_h^\rho(s,a;\bmu_h^0)}\int_{\cS}P_h(ds'|s,a) \notag\\
&\qquad\EE_{\{\bar{P}_t\}_{t=h+1}^H}\Bigg[\sum_{t=h+1}^H\bigg(r_t(s_t,a_t)-\frac{1}{\eta}D_{KL}[\pi_t(\cdot|s_t)\Vert\piref_t(\cdot|s_t)]\bigg)\Bigg|s_{h+1}=s',\pi\Bigg]. \label{eq:drmdp_bellman_proof_leq0}
\end{align}
Since the uncertainty sets $\{\cU_h^\rho(s,a;\bmu_h^0)\}_{(s,a)\in\cS\times\cA}$ are closed and $\{\cU_{h,i}^{\rho}\}_{h,i=1}^{H,d}$ are independent from the state-action pair $(s,a)$ under $d$-rectangular linear DRMDP, there exists a distribution $\bar{P}_h^\pi$ that attains the infimum, i.e.
\begin{align}
\bar{P}_h^\pi=&\arginf_{P_h(\cdot|s,a)\in\cU_h^\rho(s,a;\bmu_h^0),h\leq t\leq H}\int_{\cS}P_h(ds'|s,a) \label{eq:drmdp_bellman_proof_arginf}\\
&\qquad\EE_{\{\bar{P}_t\}_{t=h+1}^H}\Bigg[\sum_{t=h+1}^H\bigg(r_t(s_t,a_t)-\frac{1}{\eta}D_{KL}[\pi_t(\cdot|s_t)\Vert\piref_t(\cdot|s_t)]\bigg)\Bigg|s_{h+1}=s',\pi\Bigg]. \notag
\end{align}

Using definition of the policy-regularized value function $\tilde{V}_h^{\pi,P}$ in \Cref{sec:pre_linearDRMDP} and combining \eqref{eq:drmdp_bellman_proof_ih}\eqref{eq:drmdp_bellman_proof_leq0}, we have
\begin{align}
&\tilde{Q}^{\pi,\rho}_h(s,a)\notag\\
\leq&r_h(s,a)+\inf_{P_h(\cdot|s,a)\in\cU_h^\rho(s,a;\bmu_h^0)}\int_{\cS}P_h(ds'|s,a)\tilde{V}_{h+1}^{\pi,\{\bar{P}_t\}_{t=h+1}^H}(s') \label{eq:drmdp_bellman_proof_leq_restate}\\
=&r_h(s,a)+\inf_{P_h(\cdot|s,a)\in\cU_h^\rho(s,a;\bmu_h^0)}\int_{\cS}P_h(ds'|s,a)\tilde{V}_{h+1}^{\pi,\rho}(s')\label{eq:drmdp_bellman_proof_use_ih}\\
=&r_h(s,a)+\inf_{P_h(\cdot|s,a)\in\cU_h^\rho(s,a;\bmu_h^0)}\int_{\cS}P_h(ds'|s,a)\inf_{P_t(\cdot|s,a)\in\cU_t^\rho(P_t^0),h+1\leq t\leq H}\tilde{V}_{h+1}^{\pi,\{P_t\}_{t=h+1}^H}(s') \label{eq:drmdp_bellman_proof_use_v_def}\\
=&r_h(s,a)+\inf_{P_t(\cdot|s,a)\in\cU_t^\rho(P_t^0),h\leq t\leq H}\int_{\cS}P_h(ds'|s,a)\tilde{V}_{h+1}^{\pi,\{P_t\}_{t=h+1}^H}(s'),\label{eq:drmdp_bellman_proof_leq}
\end{align}
where \eqref{eq:drmdp_bellman_proof_use_ih} follows from the inductive hypothesis \eqref{eq:drmdp_bellman_proof_ih}, and \eqref{eq:drmdp_bellman_proof_use_v_def} is due to the definition of $\tilde{V}_h^{\pi,\rho}$ in \eqref{eq:drmdp_v}. Recall the definition of $\tilde{Q}_h^{\pi,\rho}$ from \eqref{eq:drmdp_q}:
\begin{align}
\tilde{Q}_h^{\pi,\rho}(s,a)
=&r_h(s,a)+\infPP\EE_{P}\Bigg[\sum_{t=h+1}^H\bigg(r_t(s_t,a_t)-\frac{1}{\eta}D_{KL}[\pi_t(\cdot|s_t)\Vert\piref_t(\cdot|s_t)]\bigg)\Bigg|s_h=s,a_h=a,\pi\Bigg]\notag\\
=&r_h(s,a)+\infPP\EE_{P}\Big[\tilde{V}_{h+1}^{\pi,P}\Big|s_h=s,a_h=a,\pi\Big] \label{eq:drmdp_bellman_proof_use_v_def1}\\
=&r_h(s,a)+\inf_{P_t(\cdot|s,a)\in\cU_t^\rho(P_t^0),h\leq t\leq H}\int_{\cS}P_h(ds'|s,a)\tilde{V}_{h+1}^{\pi,\{P_t\}_{t=h+1}^H}(s'), \label{eq:drmdp_bellman_proof_eq}
\end{align}
where \eqref{eq:drmdp_bellman_proof_use_v_def1} follows the definition of the policy-regularized value function $\tilde{V}_h^{\pi,P}$. Note that \eqref{eq:drmdp_bellman_proof_leq} and \eqref{eq:drmdp_bellman_proof_eq} are identical, so the inequality in \eqref{eq:drmdp_bellman_proof_leq0} is actually an equality. Therefore, from \eqref{eq:drmdp_bellman_proof_use_ih} we get
\begin{align}
\tilde{Q}^{\pi,\rho}_h(s,a)=&r_h(s,a)+\inf_{P_h(\cdot|s,a)\in\cU_h^\rho(s,a;\bmu_h^0)}\int_{\cS}P_h(ds'|s,a)\tilde{V}_{h+1}^{\pi,\rho}(s') \notag\\
=&r_h(s,a)+\inf_{P_h(\cdot|s,a)\in\cU_h^\rho(s,a,;\bmu_h^0)}\EE_{s'\sim P_h(\cdot|s,a)}\tilde{V}_{h+1}^{\pi,\rho}(s'), \notag 
\end{align}
which implies that \eqref{eq:bellman_drmdp_q} holds at step $h$.

On the other hand, now that \eqref{eq:drmdp_bellman_proof_leq0} is an equality, we can combine it with \eqref{eq:drmdp_bellman_proof_arginf} and get
\begin{align}
&\tilde{Q}^{\pi,\rho}_h(s,a)\notag\\
=&r_h(s,a)+\int_{\cS}\bar{P}_h(ds'|s,a)\EE_{\{\bar{P}_t\}_{t=h+1}^H}\Bigg[\sum_{t=h+1}^H\bigg(r_t(s_t,a_t)-\frac{1}{\eta}D_{KL}[\pi_t(\cdot|s_t)\Vert\piref_t(\cdot|s_t)]\bigg)\Bigg|s_{h+1}=s',\pi\Bigg] \notag\\
=&r_h(s,a)+\EE_{\{\bar{P}_t\}_{t=h}^H}\Bigg[\sum_{t=h+1}^H\bigg(r_t(s_t,a_t)-\frac{1}{\eta}D_{KL}[\pi_t(\cdot|s_t)\Vert\piref_t(\cdot|s_t)]\bigg)\Bigg|s_h=s,a_h=a,\pi\Bigg] \label{eq:drmdp_bellman_proof_q_pbar_alt}\\
=&\tilde{Q}^{\pi,\{\bar{P}_t\}_{t=h}^H}_h(s,a), \notag
\end{align}
which implies that \eqref{eq:pbar_drmdp_q} holds at step $h$.

Next we prove the statements regarding $\tilde{V}_h^{\pi,\rho}$. By the definition of $\tilde{V}_h^{\pi,\rho}$ in \eqref{eq:drmdp_v}, we have
\begin{align}
&\tilde{V}_h^{\pi,\rho}(s) \notag\\
=&\inf_{P_t\in\cU_t^\rho(P_t^0),h\leq t\leq H}\EE_{\{P_t\}_{t=h}^H}\Bigg[\sum_{t=h}^H\bigg(r_t(s_t,a_t)-\frac{1}{\eta}D_{KL}[\pi_t(\cdot|s_t)\Vert\piref_t(\cdot|s_t)]\bigg)\Bigg|s_h=s,\pi\Bigg] \notag\\
=&\inf_{P_t\in\cU_t^\rho(P_t^0),h\leq t\leq H}\int_{a\in\cA}\pi_h(a|s)\EE_{\{P_t\}_{t=h}^H}\Bigg[\sum_{t=h}^H\bigg(r_t(s_t,a_t) \notag\\
&\qquad\qquad\qquad-\frac{1}{\eta}D_{KL}[\pi_t(\cdot|s_t)\Vert\piref_t(\cdot|s_t)]\bigg)\Bigg|s_h=s,a_h=a,\pi\Bigg]da \notag\\
\leq&\int_{a\in\cA}\pi_h(a|s)\EE_{\{\bar{P}_t\}_{t=h}^H}\Bigg[\sum_{t=h}^H\bigg(r_t(s_t,a_t)-\frac{1}{\eta}D_{KL}[\pi_t(\cdot|s_t)\Vert\piref_t(\cdot|s_t)]\bigg)\Bigg|s_h=s,a_h=a,\pi\Bigg]da  \label{eq:drmdp_bellman_proof_v_leq}\\
=&\int_{a\in\cA}\pi_h(a|s)\Bigg[r_h(s_h,a_h)-\frac{1}{\eta}D_{KL}[\pi_h(\cdot|s)\Vert\piref_h(\cdot|s)] \notag\\
&\qquad\qquad\qquad+\EE_{\{\bar{P}_t\}_{t=h}^H}\Bigg[\sum_{t=h+1}^H\bigg(r_t(s_t,a_t)-\frac{1}{\eta}D_{KL}[\pi_t(\cdot|s_t)\Vert\piref_t(\cdot|s_t)]\bigg)\Bigg|s_h=s,a_h=a,\pi\Bigg]\Bigg]da \notag\\
=&\int_{a\in\cA}\pi_h(a|s)\Bigg[r_h(s_h,a_h)+\EE_{\{\bar{P}_t\}_{t=h}^H}\Bigg[\sum_{t=h+1}^H\bigg(r_t(s_t,a_t) \notag\\
&\qquad\qquad\qquad-\frac{1}{\eta}D_{KL}[\pi_t(\cdot|s_t)\Vert\piref_t(\cdot|s_t)]\bigg)\Bigg|s_h=s,a_h=a,\pi\Bigg]\Bigg]da-\frac{1}{\eta}D_{KL}[\pi_h(\cdot|s)\Vert\piref_h(\cdot|s)] \notag\\
=&\int_{a\in\cA}\pi_h(a|s)\tilde{Q}^{\pi,\rho}_h(s,a)da-\frac{1}{\eta}D_{KL}[\pi_h(\cdot|s)\Vert\piref_h(\cdot|s)]\label{eq:drmdp_bellman_proof_use_q_pbar}\\
=&\int_{a\in\cA}\pi_h(a|s)\Bigg[r_h(s,a)+\infPP\EE_{P}\Bigg[\sum_{t=h+1}^H\bigg(r_t(s_t,a_t) \label{eq:drmdp_bellman_proof_v_use_v_def}\\
&\qquad\qquad\qquad-\frac{1}{\eta}D_{KL}[\pi_t(\cdot|s_t)\Vert\piref_t(\cdot|s_t)]\bigg)\Bigg|s_h=s,a_h=a,\pi\Bigg]\Bigg]da-\frac{1}{\eta}D_{KL}[\pi_h(\cdot|s)\Vert\piref_h(\cdot|s)] \notag\\
=&\infPP\EE_{P}\int_{a\in\cA}\pi_h(a|s)\Bigg[\sum_{t=h}^H\bigg(r_t(s_t,a_t)-\frac{1}{\eta}D_{KL}[\pi_t(\cdot|s_t)\Vert\piref_t(\cdot|s_t)]\bigg)\Bigg|s_h=s,a_h=a,\pi\Bigg]da \notag\\
=&\infPP\EE_{P}\Bigg[\sum_{t=h}^H\bigg(r_t(s_t,a_t)-\frac{1}{\eta}D_{KL}[\pi_t(\cdot|s_t)\Vert\piref_t(\cdot|s_t)]\bigg)\Bigg|s_h=s,\pi\Bigg],\label{eq:drmdp_bellman_proof_v_eq}
\end{align}
where \eqref{eq:drmdp_bellman_proof_use_q_pbar} follows from \eqref{eq:drmdp_bellman_proof_q_pbar_alt} and \eqref{eq:drmdp_bellman_proof_v_use_v_def} is due to the definition of $\tilde{Q}^{\pi,\rho}_h(s,a)$ in \eqref{eq:drmdp_q}. Note that \eqref{eq:drmdp_bellman_proof_v_eq} is identical to \eqref{eq:drmdp_v}, so the inequality in \eqref{eq:drmdp_bellman_proof_v_leq} is actually an equality. Therefore, from \eqref{eq:drmdp_bellman_proof_v_leq} we get
\begin{align}
\tilde{V}_h^{\pi,\rho}(s)&=\int_{a\in\cA}\pi_h(a|s)\tilde{Q}^{\pi,\rho}_h(s,a)da-\frac{1}{\eta}D_{KL}[\pi_h(\cdot|s)\Vert\piref_h(\cdot|s)] \label{eq:drmdp_bellman_proof_v_expand}\\
&=\innerA{\tilde{Q}_h^{\pi,\rho}(s,\cdot), \pi_h(\cdot|s)}-\frac{1}{\eta}D_{KL}[\pi_h(\cdot|s)\Vert\piref_h(\cdot|s)],\notag
\end{align}
which implies that \eqref{eq:pbar_drmdp_v} holds at step $h$.

On the other hand, combining \eqref{eq:drmdp_bellman_proof_q_pbar_alt} and \eqref{eq:drmdp_bellman_proof_v_expand}, we get
\begin{align}
\tilde{V}_h^{\pi,\rho}(s)=&\int_{a\in\cA}\pi_h(a|s)\Bigg[r_h(s,a)+\EE_{\{\bar{P}_t\}_{t=h}^H}\Bigg[\sum_{t=h+1}^H\bigg(r_t(s_t,a_t) \notag\\
&\qquad\qquad-\frac{1}{\eta}D_{KL}[\pi_t(\cdot|s_t)\Vert\piref_t(\cdot|s_t)]\bigg)\Bigg|s_h=s,a_h=a,\pi\Bigg]\Bigg]da-\frac{1}{\eta}D_{KL}[\pi_h(\cdot|s)\Vert\piref_h(\cdot|s)] \notag\\
=&\int_{a\in\cA}\pi_h(a|s)\Bigg[r_h(s,a)-\frac{1}{\eta}D_{KL}[\pi_h(\cdot|s)\Vert\piref_h(\cdot|s)] \notag\\
&\qquad\qquad+\EE_{\{\bar{P}_t\}_{t=h}^H}\Bigg[\sum_{t=h+1}^H\bigg(r_t(s_t,a_t)-\frac{1}{\eta}D_{KL}[\pi_t(\cdot|s_t)\Vert\piref_t(\cdot|s_t)]\bigg)\Bigg|s_h=s,a_h=a,\pi\Bigg]\Bigg]da \notag\\
=&\int_{a\in\cA}\pi_h(a|s)\Bigg[\EE_{\{\bar{P}_t\}_{t=h}^H}\Bigg[\sum_{t=h}^H\bigg(r_t(s_t,a_t)-\frac{1}{\eta}D_{KL}[\pi_t(\cdot|s_t)\Vert\piref_t(\cdot|s_t)]\bigg)\Bigg|s_h=s,a_h=a,\pi\Bigg]\Bigg]da \notag\\
=&\EE_{\{\bar{P}_t\}_{t=h}^H}\Bigg[\sum_{t=h}^H\bigg(r_t(s_t,a_t)-\frac{1}{\eta}D_{KL}[\pi_t(\cdot|s_t)\Vert\piref_t(\cdot|s_t)]\bigg)\Bigg|s_h=s,\pi\Bigg] \notag\\
=&\tilde{V}_h^{\pi,\{\bar{P}_t^\pi\}_{t=1}^H}(s), \notag
\end{align}
which implies that \eqref{eq:pbar_drmdp_q} holds at step $h$. Hence by induction, we finish the proof of \Cref{proposition:robust_bellman_kl_drmdp}.
\end{proof}

\subsection{Proof of \Cref{proposition:robust_bellman_kl_rrmdp}}
To prove the robust Bellman equation for the policy-regularized $d$-rectangular linear RRMDP, we consider the following stronger proposition: there exists transition kernels $\bar{\bmu}=\{\bar{\bmu}_h\}_{h=1}^H, P_h^\pi=\inner{\bm{\phi},\bmu_t}$ such that the robust Bellman equations hold, i.e.
\begin{align}
\tilde{Q}_h^{\pi,\sigma}(s)=&r_h(s,a)+\inf_{\bmu_h\in\Delta(\cS)^d,P_h=\inner{\bm{\phi},\bmu_h}}\Big[\EE_{s'\sim P_h(\cdot|s,a)}[V_{h+1}^{\pi,\sigma}(s')]+\sigma\inner{\bm{\phi}(s,a),D(\bmu_h\|\bmu_h^0)}\Big],\label{eq:bellman_rrmdp_q}\\
\tilde{V}_h^{\pi,\sigma}(s)=&\innerA{\tilde{Q}_h^{\pi,\sigma}(s,\cdot), \pi_h(\cdot|s)}-\frac{1}{\eta}D_{KL}[\pi_h(\cdot|s)\Vert\piref_h(\cdot|s)],\label{eq:bellman_rrmdp_v}
\end{align}
and the robust value functions are equal to the value functions under transitions $\bar{P}^\pi$:
\begin{align}
\tilde{Q}_h^{\pi,\sigma}(s)&=\tilde{Q}_h^{\pi,\{\bar{P}_t^\pi\}_{t=1}^H}(s),\label{eq:pbar_rrmdp_q}\\
\tilde{V}_h^{\pi,\sigma}(s)&=\tilde{V}_h^{\pi,\{\bar{P}_t^\pi\}_{t=1}^H}(s).\label{eq:pbar_rrmdp_v}
\end{align}
\begin{proof}
We prove this stronger proposition by induction. At the last stage $H$, the base case holds trivially since there is no transition involved. If the inductive hypothesis holds for step $h+1$, then there exist transition kernels $\{\bar{P}_t^\pi\}_{t=h+1}^H$ such that
\begin{align} \label{eq:rrmdp_bellman_proof_ih}
\tilde{V}_{h+1}^{\pi,\sigma}(s)=\tilde{V}_{h+1}^{\pi,\{\bar{P}_t^\pi\}_{t=h+1}^H}(s).
\end{align}
By definition of $\tilde{Q}^{\pi,\sigma}_h$ in \eqref{eq:rrmdp_q}, for any $(s,a)\in\cS\times\cA$, we have
\begin{align}
&\tilde{Q}^{\pi,\sigma}_h(s,a) \\
&=r_h(s,a)+\inf_{\bmu_h\in\Delta(\cS)^d,P_h=\inner{\bm{\phi},\bmu_h}}\sigma\inner{\bm{\phi}(s,a),D(\bmu_h\|\bmu_h^0)}+\EE_{P}\Bigg[\sum_{t=h+1}^H\bigg(r_t(s_t,a_t) \notag\\
&\qquad+\sigma\inner{\bm{\phi}(s_t,a_t),D(\bmu_t\|\bmu_t^0)}-\frac{1}{\eta}D_{KL}[\pi_t(\cdot|s_t)\Vert\piref_t(\cdot|s_t)]\bigg)\Bigg|s_h=s,a_h=a,\pi\Bigg] \notag\\
&=r_h(s,a)+\inf_{\bmu_t\in\Delta(\cS)^d,P_t=\inner{\bm{\phi},\bmu_t},h\leq t\leq H}\sigma\inner{\bm{\phi}(s,a),D(\bmu_h\|\bmu_h^0)}+\EE_{\{P_t\}_{t=h}^H}\Bigg[\sum_{t=h+1}^H\bigg(r_t(s_t,a_t) \notag\\
&\qquad+\sigma\inner{\bm{\phi}(s_t,a_t),D(\bmu_t\|\bmu_t^0)}-\frac{1}{\eta}D_{KL}[\pi_t(\cdot|s_t)\Vert\piref_t(\cdot|s_t)]\bigg)\Bigg|s_h=s,a_h=a,\pi\Bigg]\\
&=r_h(s,a)+\inf_{\bmu_t\in\Delta(\cS)^d,P_t=\inner{\bm{\phi},\bmu_t},h\leq t\leq H}\sigma\inner{\bm{\phi}(s,a),D(\bmu_h\|\bmu_h^0)}+\int_{\cS}P_h(ds'|s,a)\EE_{\{P_t\}_{t=h+1}^H}\Bigg[ \notag\\
&\qquad\sum_{t=h+1}^H\bigg(r_t(s_t,a_t)+\sigma\inner{\bm{\phi}(s_t,a_t),D(\bmu_t\|\bmu_t^0)}-\frac{1}{\eta}D_{KL}[\pi_t(\cdot|s_t)\Vert\piref_t(\cdot|s_t)]\bigg)\Bigg|s_{h+1}=s',\pi\Bigg]\\
&\leq r_h(s,a)+\inf_{\bmu_h\in\Delta(\cS)^d,P_h=\inner{\bm{\phi},\bmu_h}}\sigma\inner{\bm{\phi}(s,a),D(\bmu_h\|\bmu_h^0)}+\int_{\cS}P_h(ds'|s,a)\EE_{\{\bar{P}_t\}_{t=h+1}^H}\Bigg[\sum_{t=h+1}^H\bigg(r_t(s_t,a_t) \notag\\
&\qquad+\sigma\inner{\bphi(s_t,a_t),D(\bmu_t\|\bmu_t^0)}-\frac{1}{\eta}D_{KL}[\pi_t(\cdot|s_t)\|\piref_t(\cdot|s_t)]\bigg)\Bigg|s_{h+1}=s',\pi\Bigg]. \label{eq:rrmdp_bellman_proof_leq0}
\end{align}
Since $\Delta(\cS)^d$ is closed, there exists a transition kernel $\bar{\bmu}_h, \bar{P}^\pi_h=\inner{\bm{\phi},\bmu_h}$ that attains the infimum, i.e.
\begin{align}
\bar{P}_h^\pi=&\arginf_{\bmu_h\in\Delta(\cS)^d,P_h=\inner{\bm{\phi},\bmu_h}}\sigma\inner{\bm{\phi}(s,a),D(\bmu_h\|\bmu_h^0)}+\int_{\cS}P_h(ds'|s,a)\EE_{\{\bar{P}_t\}_{t=h+1}^H}\Bigg[\sum_{t=h+1}^H\bigg(r_t(s_t,a_t) \notag\\
&\qquad+\sigma\inner{\bm{\phi}(s_t,a_t),D(\bmu_t\|\bmu_t^0)}-\frac{1}{\eta}D_{KL}[\pi_t(\cdot|s_t)\Vert\piref_t(\cdot|s_t)]\bigg)\Bigg|s_{h+1}=s',\pi\Bigg]. \label{eq:rrmdp_bellman_proof_arginf}
\end{align}

Using definition of the policy-regularized value function $\tilde{V}_h^{\pi,P}$ and combining \eqref{eq:rrmdp_bellman_proof_ih}\eqref{eq:rrmdp_bellman_proof_leq0}, we have
\begin{align}
\tilde{Q}^{\pi,\sigma}_h(s,a)\leq&r_h(s,a)+\inf_{\bmu_h\in\Delta(\cS)^d,P_h=\inner{\bm{\phi},\bmu_h}}\sigma\inner{\bm{\phi}(s,a),D(\bmu_h\|\bmu_h^0)}+\int_{\cS}P_h(ds'|s,a)\tilde{V}_{h+1}^{\pi,\{\bar{P}_t\}_{t=h+1}^H}(s') \label{eq:rrmdp_bellman_proof_leq_restate}\\
=&r_h(s,a)+\inf_{\bmu_h\in\Delta(\cS)^d,P_h=\inner{\bm{\phi},\bmu_h}}\sigma\inner{\bm{\phi}(s,a),D(\bmu_h\|\bmu_h^0)}+\int_{\cS}P_h(ds'|s,a)\tilde{V}_{h+1}^{\pi,\sigma}(s')\label{eq:rrmdp_bellman_proof_use_ih}\\
=&r_h(s,a)+\inf_{\bmu_h\in\Delta(\cS)^d,P_h=\inner{\bm{\phi},\bmu_h}}\sigma\inner{\bm{\phi}(s,a),D(\bmu_h\|\bmu_h^0)} \notag\\
&\qquad+\int_{\cS}P_h(ds'|s,a)\inf_{\bmu_t\in\Delta(\cS)^d,P_t=\inner{\bm{\phi},\bmu_t},h+1\leq t\leq H}\tilde{V}_{h+1}^{\pi,\{P_t\}_{t=h+1}^H}(s') \label{eq:rrmdp_bellman_proof_use_v_def}\\
=&r_h(s,a)+\inf_{\bmu_t\in\Delta(\cS)^d,P_t=\inner{\bm{\phi},\bmu_t}, h\leq t\leq H}\sigma\inner{\bm{\phi}(s,a),D(\bmu_h\|\bmu_h^0)}+\int_{\cS}P_h(ds'|s,a)\tilde{V}_{h+1}^{\pi,\{P_t\}_{t=h+1}^H}(s'),\label{eq:rrmdp_bellman_proof_leq}
\end{align}
where \eqref{eq:rrmdp_bellman_proof_use_ih} follows from the inductive hypothesis \eqref{eq:rrmdp_bellman_proof_ih}, and \eqref{eq:rrmdp_bellman_proof_use_v_def} is due to the definition of $\tilde{V}_h^{\pi,\sigma}$ in \eqref{eq:rrmdp_v}. Recall the definition of $\tilde{Q}_h^{\pi,\sigma}$ from \eqref{eq:rrmdp_q}:
\begin{align}
&\tilde{Q}_h^{\pi,\sigma}(s,a)\notag\\
=&r_h(s,a)+\inf_{\bmu\in\Delta(\cS)^d,P=\inner{\bm{\phi},\bmu}}\EE_{P}\Bigg[\sigma\inner{\bm{\phi}(s,a),D(\bmu_h\|\bmu_h^0)}+\sum_{t=h+1}^H\bigg(r_t(s_t,a_t) \notag\\
&\qquad+\sigma\inner{\bm{\phi}(s_t,a_t),D(\bmu_t\|\bmu_t^0)}-\frac{1}{\eta}D_{KL}[\pi_t(\cdot|s_t)\Vert\piref_t(\cdot|s_t)]\bigg)\Bigg|s_h=s,a_h=a,\pi\Bigg]\\
=&r_h(s,a)+\inf_{\bmu\in\Delta(\cS)^d,P=\inner{\bm{\phi},\bmu}}\sigma\inner{\bm{\phi}(s,a),D(\bmu_h\|\bmu_h^0)}+\EE_{P}\Big[\tilde{V}_{h+1}^{\pi,P}\Big|s_h=s,a_h=a,\pi\Big] \label{eq:rrmdp_bellman_proof_use_v_def1}\\
=&r_h(s,a)+\inf_{\bmu_t\in\Delta(\cS)^d,P_t=\inner{\bm{\phi},\bmu_t},h\leq t\leq H}\sigma\inner{\bm{\phi}(s,a),D(\bmu_h\|\bmu_h^0)}+\int_{\cS}P_h(ds'|s,a)\tilde{V}_{h+1}^{\pi,\{P_t\}_{t=h+1}^H}(s'), \label{eq:rrmdp_bellman_proof_eq}
\end{align}
where \eqref{eq:rrmdp_bellman_proof_use_v_def1} follows the definition of the policy-regularized value function $\tilde{V}_h^{\pi,P}$. Note that \eqref{eq:rrmdp_bellman_proof_leq} and \eqref{eq:rrmdp_bellman_proof_eq} are identical, so the inequality in \eqref{eq:rrmdp_bellman_proof_leq0} is actually an equality. Therefore, from \eqref{eq:rrmdp_bellman_proof_use_ih} we get
\begin{align}
\tilde{Q}^{\pi,\sigma}_h(s,a)=&r_h(s,a)+\inf_{\bmu_h\in\Delta(\cS)^d,P_h=\inner{\bm{\phi},\bmu_h}}\sigma\inner{\bm{\phi}(s,a),D(\bmu_h\|\bmu_h^0)}+\int_{\cS}P_h(ds'|s,a)\tilde{V}_{h+1}^{\pi,\sigma}(s') \notag\\
=&r_h(s,a)+\inf_{\bmu_h\in\Delta(\cS)^d,P_h=\inner{\bm{\phi},\bmu_h}}\sigma\inner{\bm{\phi}(s,a),D(\bmu_h\|\bmu_h^0)}+\EE_{s'\sim P_h(\cdot|s,a)}\tilde{V}_{h+1}^{\pi,\sigma}(s') \notag\\
=&r_h(s,a)+\inf_{\bmu_h\in\Delta(\cS)^d,P_h=\inner{\bm{\phi},\bmu_h}}\Big[\EE_{s'\sim P_h(\cdot|s,a)}[V_{h+1}^{\pi,\sigma}(s')]+\sigma\inner{\bm{\phi}(s,a),D(\bmu_h\|\bmu_h^0)}\Big], \notag
\end{align}
which implies that \eqref{eq:bellman_rrmdp_q} holds at step $h$.

On the other hand, now that \eqref{eq:rrmdp_bellman_proof_leq0} is an equality, we can combine it with \eqref{eq:rrmdp_bellman_proof_arginf} and get
\begin{align}
\tilde{Q}^{\pi,\sigma}_h(s,a)=&r_h(s,a)+\sigma\inner{\bm{\phi}(s,a),D(\bmu_h\|\bmu_h^0)}+\int_{\cS}\bar{P}_h(ds'|s,a)\EE_{\{\bar{P}_t\}_{t=h+1}^H}\Bigg[\sum_{t=h+1}^H\bigg(r_t(s_t,a_t) \notag\\
&\qquad+\sigma\inner{\bm{\phi}(s_t,a_t),D(\bmu_t\|\bmu_t^0)}-\frac{1}{\eta}D_{KL}[\pi_t(\cdot|s_t)\Vert\piref_t(\cdot|s_t)]\bigg)\Bigg|s_{h+1}=s',\pi\Bigg] \notag\\
=&r_h(s,a)+\sigma\inner{\bm{\phi}(s,a),D(\bmu_h\|\bmu_h^0)}+\EE_{\{\bar{P}_t\}_{t=h}^H}\Bigg[\sum_{t=h+1}^H\bigg(r_t(s_t,a_t) \notag\\
&\qquad+\sigma\inner{\bm{\phi}(s_t,a_t),D(\bmu_t\|\bmu_t^0)}-\frac{1}{\eta}D_{KL}[\pi_t(\cdot|s_t)\Vert\piref_t(\cdot|s_t)]\bigg)\Bigg|s_h=s,a_h=a,\pi\Bigg] \label{eq:rrmdp_bellman_proof_q_pbar_alt}\\
=&\tilde{Q}^{\pi,\{\bar{P}_t\}_{t=h}^H}_h(s,a), \notag
\end{align}
which implies that \eqref{eq:pbar_rrmdp_q} holds at step $h$.

Next we prove the statements regarding $\tilde{V}_h^{\pi,\sigma}$. By the definition of $\tilde{V}_h^{\pi,\sigma}$ in \eqref{eq:rrmdp_v}, we have
\begin{align}
&\tilde{V}_h^{\pi,\sigma}(s) \\
=&\inf_{\bmu_h\in\Delta(\cS)^d,P_h=\inner{\bm{\phi},\bmu_h},h\leq t\leq H}\EE_{\{P_t\}_{t=h}^H}\Bigg[\sum_{t=h}^H\bigg(r_t(s_t,a_t)+\sigma\inner{\bm{\phi}(s_t,a_t),D(\bmu_t\|\bmu_t^0)} \notag\\
&\qquad -\frac{1}{\eta}D_{KL}[\pi_t(\cdot|s_t)\Vert\piref_t(\cdot|s_t)]\bigg)\Bigg|s_h=s,\pi\Bigg] \notag\\
=&\inf_{\bmu_h\in\Delta(\cS)^d,P_h=\inner{\bm{\phi},\bmu_h},h\leq t\leq H}\int_{a\in\cA}\pi_h(a|s)\EE_{\{P_t\}_{t=h}^H}\Bigg[\sum_{t=h}^H\bigg(r_t(s_t,a_t)+\sigma\inner{\bm{\phi}(s_t,a_t),D(\bmu_t\|\bmu_t^0)} \notag\\
&\qquad -\frac{1}{\eta}D_{KL}[\pi_t(\cdot|s_t)\Vert\piref_t(\cdot|s_t)]\bigg)\Bigg|s_h=s,a_h=a,\pi\Bigg]da \notag\\
\leq&\int_{a\in\cA}\pi_h(a|s)\EE_{\{\bar{P}_t\}_{t=h}^H}\Bigg[\sum_{t=h}^H\bigg(r_t(s_t,a_t)+\sigma\inner{\bm{\phi}(s_t,a_t),D(\bmu_t\|\bmu_t^0)} \notag\\
&\qquad -\frac{1}{\eta}D_{KL}[\pi_t(\cdot|s_t)\Vert\piref_t(\cdot|s_t)]\bigg)\Bigg|s_h=s,a_h=a,\pi\Bigg]da  \label{eq:rrmdp_bellman_proof_v_leq}\\
=&\int_{a\in\cA}\pi_h(a|s)\Bigg[r_h(s,a)+\sigma\inner{\bm{\phi}(s,a),D(\bmu_h\|\bmu_h^0)}-\frac{1}{\eta}D_{KL}[\pi_h(\cdot|s)\Vert\piref_h(\cdot|s)]+\EE_{\{\bar{P}_t\}_{t=h}^H}\Bigg[ \notag\\
&\qquad \sum_{t=h+1}^H\bigg(r_t(s_t,a_t)+\sigma\inner{\bm{\phi}(s_t,a_t),D(\bmu_t\|\bmu_t^0)}-\frac{1}{\eta}D_{KL}[\pi_t(\cdot|s_t)\Vert\piref_t(\cdot|s_t)]\bigg)\Bigg|s_h=s,a_h=a,\pi\Bigg]\Bigg]da \notag\\
=&\int_{a\in\cA}\pi_h(a|s)\Bigg[r_h(s,a)+\sigma\inner{\bm{\phi}(s,a),D(\bmu_h\|\bmu_h^0)}+\EE_{\{\bar{P}_t\}_{t=h}^H}\Bigg[\sum_{t=h+1}^H\bigg(r_t(s_t,a_t)+\sigma\la\bm{\phi}(s_t,a_t), \notag\\
&\qquad D(\bmu_t\|\bmu_t^0)\ra-\frac{1}{\eta}D_{KL}[\pi_t(\cdot|s_t)\Vert\piref_t(\cdot|s_t)]\bigg)\Bigg|s_h=s,a_h=a,\pi\Bigg]\Bigg]da-\frac{1}{\eta}D_{KL}[\pi_h(\cdot|s)\Vert\piref_h(\cdot|s)] \notag\\
=&\int_{a\in\cA}\pi_h(a|s)\tilde{Q}^{\pi,\sigma}_h(s,a)da-\frac{1}{\eta}D_{KL}[\pi_h(\cdot|s)\Vert\piref_h(\cdot|s)]\label{eq:rrmdp_bellman_proof_use_q_pbar}\\
=&\int_{a\in\cA}\pi_h(a|s)\Bigg[r_h(s,a)+\sigma\inner{\bm{\phi}(s,a),D(\bmu_h\|\bmu_h^0)}+\inf_{\bmu\in\Delta(\cS)^d,P=\inner{\bm{\phi},\bmu}}\EE_{P}\Bigg[\sum_{t=h+1}^H\bigg(r_t(s_t,a_t)+\sigma\la\bm{\phi}(s_t,a_t), \notag\\
&D(\bmu_t\|\bmu_t^0)\ra-\frac{1}{\eta}D_{KL}[\pi_t(\cdot|s_t)\Vert\piref_t(\cdot|s_t)]\bigg)\Bigg|s_h=s,a_h=a,\pi\Bigg]\Bigg]da -\frac{1}{\eta}D_{KL}[\pi_h(\cdot|s)\Vert\piref_h(\cdot|s)] \label{eq:rrmdp_bellman_proof_v_use_v_def}\\
=&\inf_{\bmu\in\Delta(\cS)^d,P=\inner{\bm{\phi},\bmu}}\EE_{P}\int_{a\in\cA}\pi_h(a|s)\Bigg[\sum_{t=h}^H\bigg(r_t(s_t,a_t)-\frac{1}{\eta}D_{KL}[\pi_t(\cdot|s_t)\Vert\piref_t(\cdot|s_t)]\bigg)\Bigg|s_h=s,a_h=a,\pi\Bigg]da \notag\\
=&\inf_{\bmu\in\Delta(\cS)^d,P=\inner{\bm{\phi},\bmu}}\EE_{P}\Bigg[\sum_{t=h}^H\bigg(r_t(s_t,a_t)-\frac{1}{\eta}D_{KL}[\pi_t(\cdot|s_t)\Vert\piref_t(\cdot|s_t)]\bigg)\Bigg|s_h=s,\pi\Bigg],\label{eq:rrmdp_bellman_proof_v_eq}
\end{align}
where \eqref{eq:rrmdp_bellman_proof_use_q_pbar} follows from \eqref{eq:rrmdp_bellman_proof_q_pbar_alt} and \eqref{eq:rrmdp_bellman_proof_v_use_v_def} is due to the definition of $\tilde{Q}^{\pi,\sigma}_h(s,a)$ in \eqref{eq:rrmdp_q}. Note that \eqref{eq:rrmdp_bellman_proof_v_eq} is identical to \eqref{eq:rrmdp_v}, so the inequality in \eqref{eq:rrmdp_bellman_proof_v_leq} is actually an equality. Therefore, from \eqref{eq:rrmdp_bellman_proof_v_leq} we get
\begin{align}
\tilde{V}_h^{\pi,\sigma}(s)&=\int_{a\in\cA}\pi_h(a|s)\tilde{Q}^{\pi,\sigma}_h(s,a)da-\frac{1}{\eta}D_{KL}[\pi_h(\cdot|s)\Vert\piref_h(\cdot|s)] \label{eq:rrmdp_bellman_proof_v_expand}\\
&=\innerA{\tilde{Q}_h^{\pi,\sigma}(s,\cdot), \pi_h(\cdot|s)}-\frac{1}{\eta}D_{KL}[\pi_h(\cdot|s)\Vert\piref_h(\cdot|s)],\notag
\end{align}
which implies that \eqref{eq:pbar_rrmdp_v} holds at step $h$.

On the other hand, combining \eqref{eq:rrmdp_bellman_proof_q_pbar_alt} and \eqref{eq:rrmdp_bellman_proof_v_expand}, we get
\begin{align}
&\tilde{V}_h^{\pi,\sigma}(s)\\ 
=&\int_{a\in\cA}\pi_h(a|s)\Bigg[r_h(s,a)+\sigma\inner{\bm{\phi}(s,a),D(\bmu_h\|\bmu_h^0)}+\EE_{\{\bar{P}_t\}_{t=h}^H}\Bigg[\sum_{t=h+1}^H\bigg(r_t(s_t,a_t)+\sigma\la\bm{\phi}(s_t,a_t), \notag\\
&\qquad D(\bmu_t\|\bmu_t^0)\ra-\frac{1}{\eta}D_{KL}[\pi_t(\cdot|s_t)\Vert\piref_t(\cdot|s_t)]\bigg)\Bigg|s_h=s,a_h=a,\pi\Bigg]\Bigg]da -\frac{1}{\eta}D_{KL}[\pi_h(\cdot|s)\Vert\piref_h(\cdot|s)] \notag\\
=&\int_{a\in\cA}\pi_h(a|s)\Bigg[r_h(s,a)+\sigma\inner{\bm{\phi}(s,a),D(\bmu_h\|\bmu_h^0)}-\frac{1}{\eta}D_{KL}[\pi_h(\cdot|s)\Vert\piref_h(\cdot|s)]+\EE_{\{\bar{P}_t\}_{t=h}^H}\Bigg[ \notag\\
&\qquad \sum_{t=h+1}^H\bigg(r_t(s_t,a_t)+\sigma\inner{\bm{\phi}(s_t,a_t),D(\bmu_t\|\bmu_t^0)}-\frac{1}{\eta}D_{KL}[\pi_t(\cdot|s_t)\Vert\piref_t(\cdot|s_t)]\bigg)\Bigg|s_h=s,a_h=a,\pi\Bigg]\Bigg]da \notag\\
=&\int_{a\in\cA}\pi_h(a|s)\Bigg[\EE_{\{\bar{P}_t\}_{t=h}^H}\Bigg[\sum_{t=h}^H\bigg(r_t(s_t,a_t)+\sigma\inner{\bm{\phi}(s_t,a_t),D(\bmu_t\|\bmu_t^0)} \notag\\
&\qquad -\frac{1}{\eta}D_{KL}[\pi_t(\cdot|s_t)\Vert\piref_t(\cdot|s_t)]\bigg)\Bigg|s_h=s,a_h=a,\pi\Bigg]\Bigg]da \notag\\
=&\EE_{\{\bar{P}_t\}_{t=h}^H}\Bigg[\sum_{t=h}^H\bigg(r_t(s_t,a_t)+\sigma\inner{\bm{\phi}(s_t,a_t),D(\bmu_t\|\bmu_t^0)}-\frac{1}{\eta}D_{KL}[\pi_t(\cdot|s_t)\Vert\piref_t(\cdot|s_t)]\bigg)\Bigg|s_h=s,\pi\Bigg] \notag\\
=&\tilde{V}_h^{\pi,\{\bar{P}_t^\pi\}_{t=1}^H}(s), \notag
\end{align}
which implies that \eqref{eq:pbar_rrmdp_q} holds at step $h$. Hence by induction, we finish the proof of \Cref{proposition:robust_bellman_kl_rrmdp}.
\end{proof}

\subsection{Proof of \Cref{proposition:optimal_policy_drmdp}}

Next, we show that there exists an optimal policy under policy-regularized $d$-rectangular DRMDP.

\begin{proof}
We define a policy $\hat{\pi}=\{\hat{\pi}\}_{h=1}^H$ such that for all $(s,a,h)\in\cS\times\cA\times[H]$,
\begin{align}
\hat{\pi}(a|s)=\frac{1}{Z_h(s)}\piref_h(a|s)\exp\Big(\eta\tilde{Q}^{*,\rho}_{h}(s,a)\Big),
\end{align}
where $Z_h(s)=\int_{a\in\cA}\piref_h(a|s)\exp\Big(\eta\tilde{Q}^{*,\rho}_{h}(s,a)\Big)da$ is the partition function. By induction, we proceed to show that $\hat{\pi}$ is optimal, i.e. for all $(s,h)\in\cS\times[H]$, we have that
\begin{align*}
\tilde{V}^{\hat{\pi},\rho}_h(s)=\tilde{V}^{*,\rho}_h(s).
\end{align*}

At the last stage $H$, we have $\tilde{Q}_H^{\pi,\rho}(s,a)=r_H(s,a)$ for any policy $\pi$. By \Cref{proposition:robust_bellman_kl_drmdp}, we have that $\tilde{V}_H^{\pi,\rho}(s)=\inner{r_H^{\pi,\rho}(s,\cdot),\pi(\cdot|s)}-\frac{1}{\eta} D_{KL}[\pi_H(\cdot|s)\|\piref(\cdot|s)]$. Applying \Cref{propositon:optimization_kl} with $F=r_H$, it follows that $\hat{\pi}(\cdot|s)$ is the maximizer of $\tilde{V}_H^{\pi,\rho}(s)$, so $\tilde{V}^{\hat{\pi},\rho}_H(s)=\tilde{V}^{*,\rho}_H(s)$.

Next, suppose that the conclusion holds at stage $h+1$, i.e. for all $(s,h)\in\cS\times[H]$,
\begin{align*}
\tilde{V}^{\hat{\pi},\rho}_{h+1}(s)=\tilde{V}^{*,\rho}_{h+1}(s).
\end{align*}
By definition of $\tilde{V}_h^{*,\rho}$ where $\forall \pi\in\Pi, \tilde{V}^{*,\rho}_{h}(s)\geq\tilde{V}^{\pi,\rho}_{h}(s)$, we have
\begin{align}
&\tilde{V}^{*,\rho}_{h}(s)\notag\\
=&\sup_{\pi\in\Pi}\Big[\innerA{\tilde{Q}_h^{\pi,\rho}(s,\cdot), \pi_h(\cdot|s)}-\frac{1}{\eta}D_{KL}[\pi_h(\cdot|s)\Vert\piref_h(\cdot|s)]\Big] \label{eq:optimal_policy_proof_drmdp_apply_bellman}\\
=&\sup_{\pi\in\Pi}\bigg[\inner{\Big[r_h(s,a)+\inf_{P_h(\cdot|s,a)\in\cU_h^\rho(s,a;\bmu_h^0)}\EE_{s'\sim P_h(\cdot|s,a)}\tilde{V}_{h+1}^{\pi,\rho}(s')\Big],\pi_h(\cdot|s)}-\frac{1}{\eta}D_{KL}[\pi_h(\cdot|s)\Vert\piref_h(\cdot|s)]\bigg] \label{eq:optimal_policy_proof_drmdp_apply_bellman1}\\
\leq&\sup_{\pi\in\Pi}\bigg[\inner{\Big[r_h(s,a)+\inf_{P_h(\cdot|s,a)\in\cU_h^\rho(s,a;\bmu_h^0)}\EE_{s'\sim P_h(\cdot|s,a)}\tilde{V}_{h+1}^{*,\rho}(s')\Big],\pi_h(\cdot|s)}-\frac{1}{\eta}D_{KL}[\pi_h(\cdot|s)\Vert\piref_h(\cdot|s)]\bigg] \label{eq:optimal_policy_proof_drmdp_apply_opt}\\
=&\sup_{\pi\in\Pi}\bigg[\inner{\Big[r_h(s,a)+\inf_{P_h(\cdot|s,a)\in\cU_h^\rho(s,a;\bmu_h^0)}\EE_{s'\sim P_h(\cdot|s,a)}\tilde{V}_{h+1}^{\hat{\pi},\rho}(s')\Big],\pi_h(\cdot|s)}-\frac{1}{\eta}D_{KL}[\pi_h(\cdot|s)\Vert\piref_h(\cdot|s)]\bigg], \label{eq:optimal_policy_proof_drmdp_apply_ih}
\end{align}
where \eqref{eq:optimal_policy_proof_drmdp_apply_bellman} and \eqref{eq:optimal_policy_proof_drmdp_apply_bellman1} are due to \Cref{proposition:robust_bellman_kl_drmdp}, \eqref{eq:optimal_policy_proof_drmdp_apply_opt} follows from the definition of $\tilde{V}_h^{*,\rho}$, and \eqref{eq:optimal_policy_proof_drmdp_apply_ih} follows from the inductive hypothesis.

Applying \Cref{propositon:optimization_kl} with $Q=r_h(s,a)+\inf_{P_h(\cdot|s,a)\in\cU_h^\rho(s,a;\bmu_h^0)}\EE_{s'\sim P_h(\cdot|s,a)}\tilde{V}_{h+1}^{\hat{\pi},\rho}(s')$, $\hat{\pi}$ is the maximizer of the RHS of \eqref{eq:optimal_policy_proof_drmdp_apply_ih}. Therefore, we further have that
\begin{align}
\tilde{V}^{*,\rho}_{h}(s)&\leq\inner{\Big[r_h(s,a)+\inf_{P_h(\cdot|s,a)\in\cU_h^\rho(s,a;\bmu_h^0)}\EE_{s'\sim P_h(\cdot|s,a)}\tilde{V}_{h+1}^{\hat{\pi},\rho}(s')\Big],\hat{\pi}_h(\cdot|s)}-\frac{1}{\eta}D_{KL}[\hat{\pi}_h(\cdot|s)\Vert\piref_h(\cdot|s)] \notag\\
&=\innerA{\tilde{Q}_h^{\hat{\pi},\rho}(s,\cdot), \hat{\pi}_h(\cdot|s)}-\frac{1}{\eta}D_{KL}[\hat{\pi}_h(\cdot|s)\Vert\piref_h(\cdot|s)] \label{eq:optimal_policy_proof_drmdp_apply_bellman2}\\
&=\tilde{V}^{\hat{\pi},\rho}_{h}(s), \notag
\end{align}
where \eqref{eq:optimal_policy_proof_drmdp_apply_bellman2} is due to \Cref{proposition:robust_bellman_kl_drmdp}. Since $\forall \pi\in\Pi, \tilde{V}^{*,\rho}_{h}(s)\geq\tilde{V}^{\pi,\rho}_{h}(s)$, we subsequently have $\tilde{V}^{*,\rho}_{h}(s)=\tilde{V}^{\pi,\rho}_{h}(s)$. Hence we finish the proof of \Cref{proposition:optimal_policy_drmdp} using the induction argument.
\end{proof}

\subsection{Proof of \Cref{proposition:optimal_policy_rrmdp}}

Similarly, we show that there exists an optimal policy under policy-regularized $d$-rectangular RRMDP.

\begin{proof}
We define a policy $\hat{\pi}=\{\hat{\pi}\}_{h=1}^H$ such that for all $(s,a,h)\in\cS\times\cA\times[H]$,
\begin{align}
\hat{\pi}(a|s)=\frac{1}{Z_h(s)}\piref_h(a|s)\exp\Big(\eta\tilde{Q}^{*,\sigma}_{h}(s,a)\Big),
\end{align}
where $Z_h(s)=\int_{a\in\cA}\piref_h(a|s)\exp\big(\eta\tilde{Q}^{*,\sigma}_{h}(s,a)\big)da$ is the partition function.

By induction, we proceed to show that $\hat{\pi}$ is optimal, i.e. for all $(s,h)\in\cS\times[H]$, we have that
\begin{align*}
\tilde{V}^{\hat{\pi},\sigma}_h(s)=\tilde{V}^{*,\sigma}_h(s).
\end{align*}
At the last stage $H$, we have $\tilde{Q}_H^{\pi,\sigma}(s,a)=r_H(s,a)$ for any policy $\pi$. By \Cref{proposition:robust_bellman_kl_rrmdp}, we have that $\tilde{V}_H^{\pi,\sigma}(s)=\inner{r_H^{\pi,\sigma}(s,\cdot),\pi(\cdot|s)}-\frac{1}{\eta} D_{KL}[\pi_H(\cdot|s)\|\piref(\cdot|s)]$. Applying \Cref{propositon:optimization_kl} with $Q=r_H$, it follows that $\hat{\pi}(\cdot|s)$ is the maximizer of $\tilde{V}_H^{\pi,\sigma}(s)$, so $\tilde{V}^{\hat{\pi},\sigma}_H(s)=\tilde{V}^{*,\sigma}_H(s)$.

Next, suppose that the conclusion holds at stage $h+1$, i.e. for all $(s,h)\in\cS\times[H]$,
\begin{align*}
\tilde{V}^{\hat{\pi},\sigma}_{h+1}(s)=\tilde{V}^{*,\sigma}_{h+1}(s).
\end{align*}
By definition of $\tilde{V}_h^{*,\sigma}$ where $\forall \pi\in\Pi, \tilde{V}^{*,\sigma}_{h}(s)\geq\tilde{V}^{\pi,\sigma}_{h}(s)$, we have
\begin{align}
\tilde{V}^{*,\sigma}_{h}(s)=&\sup_{\pi\in\Pi}\Big[\innerA{\tilde{Q}_h^{\pi,\sigma}(s,\cdot), \pi_h(\cdot|s)}-\frac{1}{\eta}D_{KL}[\pi_h(\cdot|s)\Vert\piref_h(\cdot|s)]\Big] \label{eq:optimal_policy_proof_rrmdp_apply_bellman}\\
=&\sup_{\pi\in\Pi}\bigg[\bigg\la\Big[r_h(s,a)+\inf_{\bmu_h\in\Delta(\cS)^d,P_h=\inner{\bm{\phi},\bmu_h}}\Big[\EE_{s'\sim P_h(\cdot|s,a)}[\tilde{V}_{h+1}^{\pi,\sigma}(s')] \notag\\
&\qquad\qquad+\sigma\inner{\bm{\phi}(s,a),D(\bmu_h\|\bmu_h^0)}\Big],\pi_h(\cdot|s)\bigg\ra-\frac{1}{\eta}D_{KL}[\pi_h(\cdot|s)\Vert\piref_h(\cdot|s)]\bigg] \label{eq:optimal_policy_proof_rrmdp_apply_bellman1}\\
\leq&\sup_{\pi\in\Pi}\bigg[\bigg\la\Big[r_h(s,a)+\inf_{\bmu_h\in\Delta(\cS)^d,P_h=\inner{\bm{\phi},\bmu_h}}\Big[\EE_{s'\sim P_h(\cdot|s,a)}[\tilde{V}_{h+1}^{*,\sigma}(s')] \notag\\
&\qquad\qquad+\sigma\inner{\bm{\phi}(s,a),D(\bmu_h\|\bmu_h^0)}\Big],\pi_h(\cdot|s)\bigg\ra-\frac{1}{\eta}D_{KL}[\pi_h(\cdot|s)\Vert\piref_h(\cdot|s)]\bigg] \label{eq:optimal_policy_proof_rrmdp_apply_opt}\\
=&\sup_{\pi\in\Pi}\bigg[\bigg\la\Big[r_h(s,a)+\inf_{\bmu_h\in\Delta(\cS)^d,P_h=\inner{\bm{\phi},\bmu_h}}\Big[\EE_{s'\sim P_h(\cdot|s,a)}[\tilde{V}_{h+1}^{\hat{\pi},\sigma}(s')] \notag\\
&\qquad\qquad+\sigma\inner{\bm{\phi}(s,a),D(\bmu_h\|\bmu_h^0)}\Big],\pi_h(\cdot|s)\bigg\ra-\frac{1}{\eta}D_{KL}[\pi_h(\cdot|s)\Vert\piref_h(\cdot|s)]\bigg], \label{eq:optimal_policy_proof_rrmdp_apply_ih}
\end{align}
where \eqref{eq:optimal_policy_proof_rrmdp_apply_bellman} and \eqref{eq:optimal_policy_proof_rrmdp_apply_bellman1} are due to \Cref{proposition:robust_bellman_kl_rrmdp}, \eqref{eq:optimal_policy_proof_rrmdp_apply_opt} follows from the definition of $\tilde{V}_h^{*,\sigma}$, and \eqref{eq:optimal_policy_proof_rrmdp_apply_ih} follows from the inductive hypothesis.

Applying \Cref{propositon:optimization_kl} with $Q=r_h(s,a)+\inf_{P_h(\cdot|s,a)\in\cU_h^\sigma(s,a;\bmu_h^0)}\EE_{s'\sim P_h(\cdot|s,a)}\tilde{V}_{h+1}^{\hat{\pi},\sigma}(s')$, $\hat{\pi}$ is the maximizer of the RHS of \eqref{eq:optimal_policy_proof_rrmdp_apply_ih}. Therefore, we further have that
\begin{align}
\tilde{V}^{*,\sigma}_{h}(s)\leq&\inner{\Big[r_h(s,a)+\inf_{\bmu_h\in\Delta(\cS)^d,P_h=\inner{\bm{\phi},\bmu_h}}\Big[\EE_{s'\sim P_h(\cdot|s,a)}[\tilde{V}_{h+1}^{\hat{\pi},\sigma}(s')]+\sigma\inner{\bm{\phi}(s,a),D(\bmu_h\|\bmu_h^0)}\Big],\hat{\pi}_h(\cdot|s)} \notag\\
&-\frac{1}{\eta}D_{KL}[\hat{\pi}_h(\cdot|s)\Vert\piref_h(\cdot|s)] \notag\\
=&\innerA{\tilde{Q}_h^{\hat{\pi},\sigma}(s,\cdot), \hat{\pi}_h(\cdot|s)}-\frac{1}{\eta}D_{KL}[\hat{\pi}_h(\cdot|s)\Vert\piref_h(\cdot|s)] \label{eq:optimal_policy_proof_rrmdp_apply_bellman2}\\
=&\tilde{V}^{\hat{\pi},\sigma}_{h}(s), \notag
\end{align}
where \eqref{eq:optimal_policy_proof_rrmdp_apply_bellman2} is due to \Cref{proposition:robust_bellman_kl_rrmdp}. Since $\forall \pi\in\Pi, \tilde{V}^{*,\sigma}_{h}(s)\geq\tilde{V}^{\pi,\sigma}_{h}(s)$, we subsequently have $\tilde{V}^{*,\sigma}_{h}(s)=\tilde{V}^{\pi,\sigma}_{h}(s)$. Hence we finish the proof of \Cref{proposition:optimal_policy_rrmdp} using the induction argument.
\end{proof}

\section{Proof of Main Results}

\subsection{Proof of \Cref{theorem:main_subopt_upper_bound} under $d$-rectangular linear DRMDP}

We first define the model prediction error:
\begin{align}\label{eq:iota_def}
\iota_h^{k,\rho}\sahk=r_h\sahk+\inf_{P_h\in \cU_h^\rho(s_h^k,a_h^k;\bmu_h^0)}[\PP_h \tilde{V}_{h+1}^{k,\rho}]\sahk-\tilde{Q}_h^{k,\rho}\sahk.
\end{align}

To decompose and bound the regret, we will also rely on the following technical lemmas:

\begin{lemma}\label{lemma:UCB}
(UCB) Under the policy-regularized $d$-rectangular linear DRMDP, on the event $\cE$ defined in \Cref{lemma:C1}, $\tilde{V}_h^{k,\rho}(s)\geq \tilde{V}_h^{*,\rho}(s)$ holds for all $(s,a,h,k)\in\cS\times\cA\times[H]\times[K]$.
\end{lemma}

\begin{lemma}\label{lemma:sum_bonus}
(Bound for the sum of bonuses, \citet{liu2024distributionally}, Corollary 5.3) Under the policy-regularized $d$-rectangular linear DRMDP and RRMDP and \Cref{assumption:uniform_exploration}, with probability at least $1-p$, the sum of bonus $\Gamma_h^k$ over all timesteps and episodes $(h,k)\in[H]\times[K]$ can be bounded by
\begin{align*}
\sum_{k=1}^K\sum_{h=1}^H\Gamma_h^k\leq 2H\frac{\beta}{\sqrt{K}}\sqrt{\frac{128}{\alpha^2}\log\frac{3dHK}{p} + \frac{2}{\alpha}\log K},
\end{align*}
where $\alpha>0$ denotes the lower bound of the smallest eigenvalue of $\EE_\pi[\bm{\phi}(s_h,a_h)\bm{\phi}(s_h,a_h)^\top]$.
\end{lemma}

\begin{lemma}\label{lemma:iota_bound}
(Bound on model prediction error) Under the policy-regularized $d$-rectangular linear DRMDP, on the event $\cE$ defined in \Cref{lemma:C1}, $-2\Gamma_h^k\sahk\leq\iota_h^{k,\rho}\sahk\leq0$ holds for all $(s,a,h,k)\in\cS\times\cA\times[H]\times[K]$.
\end{lemma}

\begin{proof}
    By \Cref{lemma:UCB}, we have the UCB property $\forall s,a,h,k,\tilde{V}_h^{k,\rho}(s)\geq \tilde{V}_h^{*,\rho}(s)$.
Therefore the regret decomposition becomes
\begin{align}
\text{Regret}=&\sum_{k=1}^K[\tilde{V}_1^{*,\rho}(s_1^k)-\tilde{V}_1^{\pi^k,\rho}(s_1^k)]\leq\sum_{k=1}^K[\tilde{V}_1^{k,\rho}(s_1^k)-\tilde{V}_1^{\pi^k,\rho}(s_1^k)].\notag
\end{align}

First rewrite the model prediction error:
\begin{align}
&\iota_h^{k,\rho}\sahk\notag\\
=&r_h\sahk+\infPh{h}[\PP_h\tilde{V}_{h+1}^{k,\rho}]\sahk-\tilde{Q}_h^{k,\rho}\sahk\notag\\
=&r_h\sahk+\infPh{h}[\PP_h\tilde{V}_{h+1}^{k,\rho}]\sahk-\tilde{Q}_h^{\pi^k,\rho}\sahk+\tilde{Q}_h^{\pi^k,\rho}\sahk-\tilde{Q}_h^{k,\rho}\sahk\notag\\
=&r_h\sahk+\infPh{h}[\PP_h\tilde{V}_{h+1}^{k,\rho}]\sahk\notag\\
&\qquad-\bigg(r_h\sahk+\infPh{h}P_h\tilde{V}_{h+1}^{\pi^k,\rho}\sahk\bigg)+\tilde{Q}_h^{\pi^k,\rho}\sahk-\tilde{Q}_h^{k,\rho}\sahk\notag\\
=&\bigg(\infPh{h}[\PP_h\tilde{V}_{h+1}^{k,\rho}]\sahk-\infPh{h}[\PP_h\tilde{V}_{h+1}^{\pi^k,\rho}]\sahk\bigg)\notag\\
&\qquad-\big(\tilde{Q}_h^{k,\rho}\sahk -\tilde{Q}_h^{\pi^k,\rho}\sahk\big).\label{eq:iota_rewrite} 
\end{align}
Again, using the dual formulation in \eqref{eq:dual_drmdp_simp}, which is simplified under the fail-state simplification in \Cref{assumption:failstate}, we have
\begin{align}
&\infPh{h}[\PP_h\tilde{V}_{h+1}^{k,\rho}]\sahk-\infPh{h}[\PP_h\tilde{V}_{h+1}^{\pi^k,\rho}]\sahk \notag\\
&=\inner{\bm{\phi}\sahk,\inf_{\bmu_h\in\cU_h^\rho(\bmu_h^0)}[\PP_h\tilde{V}_{h+1}^{k,\rho}]\sahk-\inf_{\bmu_h\in\cU_h^\rho(\bmu_h^0)}[\PP_h\tilde{V}_{h+1}^{\pi^k,\rho}]\sahk} \notag\\
&=\inner{\bm{\phi}\sahk,\Big[\max_{\alpha\in[0,H]}\{\EE_{s\sim\bmu_{h,i}^0}[\tilde{V}_{h+1}^{k,\rho}(s)]_\alpha-\rho\alpha\}\Big]_{i\in[d]}-\Big[\max_{\alpha\in[0,H]}\{\EE_{s\sim\bmu_{h,i}^0}[\tilde{V}_{h+1}^{\pi^k,\rho}(s)]_\alpha-\rho\alpha\}\Big]_{i\in[d]}} \notag\\
&\leq\inner{\bm{\phi}\sahk,\bigg[\max_{\alpha_i\in[0,H]}\{\EE_{s\in \bmu_{h,i}^0}[\tilde{V}_{h+1}^{k,\rho}(s)]_{\alpha_i}-\EE_{s\in \bmu_{h,i}^0}[\tilde{V}_{h+1}^{\pi^k,\rho}(s)]_{\alpha_i}\}\bigg]_{i\in[d]}}\notag\\
&\leq\inner{\bm{\phi}\sahk,\Big[\EE_{s\in \bmu_{h,i}^0}[\tilde{V}_{h+1}^{k,\rho}(s)-\tilde{V}_{h+1}^{\pi^k,\rho}(s)]\}\Big]_{i\in[d]}}\label{eq:ucb_model_pred_error}\\
&=P_h^0\big(\tilde{V}_{h+1}^{k,\rho}-\tilde{V}_{h+1}^{\pi^k,\rho}\big)\sahk.\label{eq:expected_model_pred_error_value}
\end{align}
where \eqref{eq:ucb_model_pred_error} is due to \Cref{lemma:UCB} (UCB). Combining equations \eqref{eq:iota_rewrite} and \eqref{eq:expected_model_pred_error_value}, we have
\begin{align}
\iota_h^{k,\rho}\sahk\leq P_h^0\big(\tilde{V}_{h+1}^{k,\rho}-\tilde{V}_{h+1}^{\pi^k,\rho}\big)\sahk-\big(\tilde{Q}_h^{k,\rho}\sahk-\tilde{Q}_h^{\pi^k,\rho}\sahk\big),\notag
\end{align}

Then we have 
\begin{align}
&\tilde{V}_h^{k,\rho}(s_h^k)-\tilde{V}_h^{\pi^k,\rho}(s_h^k)\notag\\
=&\innerA{\tilde{Q}_h^{k,\rho}(s_h^k,\cdot),\pi_h^k(\cdot|s_h^k)}-\frac{1}{\eta}D_{KL}[\pi_h^k(\cdot|s_h^k)\|\pi_h^0(\cdot|s_h^k)]-\innerA{\tilde{Q}_h^{\pi^k,\rho}(s_h^k,\cdot),\pi_h^k(\cdot|s_h^k)}\notag\\
&\qquad\qquad\qquad+\frac{1}{\eta}D_{KL}[\pi_h^k(\cdot|s_h^k)\|\pi_h^0(\cdot|s_h^k)]\notag\\
=&\innerA{\tilde{Q}_h^{k,\rho}(s_h^k,\cdot)-\tilde{Q}_h^{\pi^k,\rho}(s_h^k,\cdot),\pi_h^k(\cdot|s_h^k)}+\iota_h^{k,\rho}\sahk-\iota_h^{k,\rho}\sahk\notag\\
\leq&\underbrace{\innerA{\tilde{Q}_h^{k,\rho}(s_h^k,\cdot)-\tilde{Q}_h^{\pi^k,\rho}(s_h^k,\cdot),\pi_h^k(\cdot|s_h^k)}-\Big(\tilde{Q}_h^{k,\rho}\sahk-\tilde{Q}_h^{\pi^k,\rho}\sahk\Big)}_{\kappa_h^k}-\iota_h^{k,\rho}\sahk\label{eq:kappa_omega}\\
&+\underbrace{\bigg(P_h^0\Big(\tilde{V}_{h+1}^{k,\rho}-\tilde{V}_{h+1}^{\pi^k,\rho}\Big)\sahk-(\tilde{V}_{h+1}^{k,\rho}(s_{h+1}^k)-\tilde{V}_{h+1}^{\pi^k,\rho}(s_{h+1}^k)\Big)}_{\omega_h^k}+\Big(\tilde{V}_{h+1}^{k,\rho}(s_{h+1}^k)-\tilde{V}_{h+1}^{\pi^k,\rho}(s_{h+1}^k)\Big)\notag\\
\leq&\Big((\tilde{V}_{h+1}^{k,\rho}(s_{h+1}^k)-\tilde{V}_{h+1}^{\pi^k,\rho}(s_{h+1}^k)\Big)+\kappa_h^k\sahk+\omega_h^k\sahk+2\Gamma_h^k\sahk,\label{eq:4_22}
\end{align}
where the last inequality \eqref{eq:4_22} is due to the bound on $\iota_h^{k,\rho}$ given by \Cref{lemma:iota_bound}. Using the shorthand notations $\kappa_h^k$ and $\omega_h^k$ to denote the two terms in \eqref{eq:kappa_omega}, the regret is
\begin{align}
\text{Regret}=&\sum_{k=1}^K[\tilde{V}_1^{k,\rho}(s_1^k)-\tilde{V}_1^{\pi^k,\rho}(s_1^k)]
\leq\sum_{k=1}^K\sum_{h=1}^H\kappa_h^k\sahk+\sum_{k=1}^K\sum_{h=1}^H\omega_h^k\sahk+2\sum_{k=1}^K\sum_{h=1}^H\Gamma_h^k\sahk.\notag
\end{align}
We bound the first two terms as martingales. Since $\tilde{Q}_h^{k,\rho}\sahk\leq H$ by construction of the algorithm and $\tilde{Q}_h^{\pi^k,\rho}\sahk\leq H$ by definition of the $\tilde{Q}$ function, we have
\begin{align}
\kappa_h^k=\innerA{\tilde{Q}_h^{k,\rho}-\tilde{Q}_h^{\pi^k,\rho},\pi_h^k}-\Big(\tilde{Q}_h^{k,\rho}\sahk-\tilde{Q}_h^{\pi^k,\rho}\sahk\Big)\leq2H.\notag
\end{align}
Similarly,
\begin{align}
\omega_h^k=\big(P_h^0\big(\tilde{V}_{h+1}^{k,\rho}-\tilde{V}_{h+1}^{\pi^k,\rho}\big)\sahk-\big(\tilde{V}_{h+1}^{k,\rho}-\tilde{V}_{h+1}^{\pi^k,\rho}\big)\big)\leq2H.\notag
\end{align}
Consequently, $\sum_{k=1}^K\sum_{h=1}^H(\kappa_h^k\sahk+\omega_h^k\sahk)$ is a martingale with each term bounded by $4H$. Then applying the Azuma-Hoeffding inequality, we have that for any $t>0$,
\begin{align}
P\Bigg(\Bigg|\sum_{k=1}^K\sum_{h=1}^H(\kappa_h^k\sahk+\omega_h^k\sahk)\Bigg|>t\Bigg)\leq2\exp\bigg(\frac{-t^2}{(4H)^2HK}\bigg).\notag
\end{align}
Then if we set $t=\sqrt{16H^3K\log(3/p)}$, we have with probability at least $1-p/3$ that
\begin{align}\label{eq:martingale_bound}
\sum_{k=1}^K\sum_{h=1}^H(\kappa_h^k\sahk+\omega_h^k\sahk)\leq\sqrt{16H^3K\log\frac{3}{p}},
\end{align}
and thus
\begin{align}
\text{AveSubopt}(K)\leq\sqrt{\frac{16H^3}{K}\log\frac{3}{p}}+\frac{1}{K}\sum_{k=1}^K\sum_{h=1}^H\Gamma_h^k\sahk.
\end{align}

The last term represents the sum of bonuses encountered across all timesteps and episodes. Under assumption \Cref{assumption:uniform_exploration}, it can be upper-bounded by $2H\frac{\beta}{\sqrt{K}}\sqrt{\frac{128}{\alpha^2}\log\frac{3dHK}{p}+\frac{2}{\alpha}\log K}$ with probability at least $1-p/3$ according to \Cref{lemma:sum_bonus}. Together with the choice that $\beta=c_\beta dH\sqrt{\log(3dKH\vA/p)}$ in which $\vA$ denotes the measure of $\mathcal{A}$, there exists a constant $c$ such that
\begin{align}
\text{AveSubopt}(K)\leq\frac{cdH^2\log(3dHK\vA/p)}{\alpha\sqrt{K}}.
\end{align}
\end{proof}

\subsection{Proof of \Cref{theorem:main_subopt_upper_bound} under $d$-rectangular linear RRMDP}

For simplicity, for any value function $V:\cS\rightarrow[0,H]$ and $(s,a)\in\cS\times\cA$ we denote
\begin{align} \label{eq:psi_def}
\psi(V,s,a) = \inf_{\bmu_h\in\Delta(\cS)^d,P_h=\inner{\bm{\phi},\bmu_h}}\Big[\EE_{s'\sim P_h(\cdot|s,a)}[V(s')]+\sigma\inner{\bm{\phi}(s,a),D(\bmu_h\|\bmu_h^0)}\Big], 
\end{align}
and define
\begin{align}\label{eq:iota_def1}
\iota_h^{k,\sigma}\sahk=r_h\sahk+\psi(\tilde{V}^{k,\sigma}_{h+1},s_h^k,a_h^k)-\tilde{Q}_h^{k,\sigma}\sahk.
\end{align}

To decompose and bound the regret, we will rely on the following technical lemmas:

\begin{lemma}\label{lemma:UCB1}
(UCB) Under the policy-regularized $d$-rectangular linear RRMDP, on the event $\cE$ defined in \Cref{lemma:C1_1}, $\tilde{V}_h^{k,\sigma}(s)\geq \tilde{V}_h^{*,\sigma}(s)$ holds for all $(s,a,h,k)\in\cS\times\cA\times[H]\times[K]$.
\end{lemma}

\begin{lemma}\label{lemma:iota_bound1}
(Bound on model prediction error) Under the policy-regularized $d$-rectangular linear RRMDP, on the event $\cE$ defined in \Cref{lemma:C1_1}, $-2\Gamma_h^k\sahk\leq\iota_h^{k,\sigma}\sahk\leq0$ holds for all $(s,a,h,k)\in\cS\times\cA\times[H]\times[K]$.
\end{lemma}

As we will prove later in \Cref{lemma:UCB}, now we do have the UCB property $\forall s,a,h,k,\tilde{V}_h^{k,\rho}(s)\geq \tilde{V}_h^{*,\rho}(s)$.
Therefore the regret decomposition becomes
\begin{align}
\text{Regret}=&\sum_{k=1}^K[\tilde{V}_1^{*,\sigma}(s_1^k)-\tilde{V}_1^{\pi^k,\sigma}(s_1^k)]\leq\sum_{k=1}^K[\tilde{V}_1^{k,\sigma}(s_1^k)-\tilde{V}_1^{\pi^k,\sigma}(s_1^k)].\notag
\end{align}
First rewrite the model prediction error:
\begin{align}
&\iota_h^{k,\sigma}\sahk\notag\\
=&r_h\sahk+\psi(\tilde{V}^{k,\sigma}_{h+1},s_h^k,a_h^k)-\tilde{Q}_h^{k,\sigma}\sahk\notag\\
=&r_h\sahk+\psi(\tilde{V}^{k,\sigma}_{h+1},s_h^k,a_h^k)-\tilde{Q}_h^{\pi^k,\sigma}\sahk+\tilde{Q}_h^{\pi^k,\sigma}\sahk-\tilde{Q}_h^{k,\sigma}\sahk\notag\\
=&r_h\sahk+\psi(\tilde{V}^{k,\sigma}_{h+1},s_h^k,a_h^k)-\big(r_h\sahk+\psi(\tilde{V}^{\pi^k,\sigma}_{h+1},s_h^k,a_h^k)\big)+\tilde{Q}_h^{\pi^k,\sigma}\sahk-\tilde{Q}_h^{k,\sigma}\sahk\notag\\
=&\big(\psi(\tilde{V}^{k,\sigma}_{h+1},s_h^k,a_h^k)-\psi(\tilde{V}^{\pi^k,\sigma}_{h+1},s_h^k,a_h^k)\big)-\big(\tilde{Q}_h^{k,\sigma}\sahk-\tilde{Q}_h^{\pi^k,\sigma}\sahk\big).\label{eq:iota_rewrite1} 
\end{align}
Proposition 4.3 of \cite{tang2025robust} shows that
\begin{align}
\inf_{\bmu\in\Delta(\cS)}\EE_{s\sim\bmu}V(s)+\sigma D(\bmu\|\bmu^0)=\EE_{s\sim\bmu^0}[V(s)]_{V_{\min}+\sigma} \notag
\end{align}
under \Cref{assumption:failstate}, the existence of the fail-state implies that $V_{\min}=0$, so this further simplifies to
\begin{align}
\inf_{\bmu\in\Delta(\cS)}\EE_{s\sim\bmu}V(s)+\sigma D(\bmu\|\bmu^0)=\EE_{s\sim\bmu^0}[V(s)]_{\sigma}, \notag
\end{align}
and thus
\begin{align}
\psi(V,s,a)=\la\bm{\phi}(s,a),\big[\EE_{s\sim\bmu_i^0}[V(s)]_{\sigma}\big]_{i\in[d]}\ra. \notag
\end{align}
Therefore, we have that
\begin{align}
\psi(\tilde{V}^{k,\sigma}_{h+1},s_h^k,a_h^k)-\psi(\tilde{V}^{\pi^k,\sigma}_{h+1},s_h^k,a_h^k)
&=\inner{\bm{\phi}\sahk, \big[\EE_{s\sim\bmu_{h,i}^0}[V^{k,\sigma}_{h+1}(s)]_{\sigma}-\EE_{s\sim\bmu_{h,i}^0}[V^{\pi^k,\sigma}_{h+1}(s)]_{\sigma}\big]_{i\in[d]}} \notag\\
&\leq\inner{\bm{\phi}\sahk,\big[\EE_{s\sim\bmu_{h,i}^0}\big[\tilde{V}_{h+1}^{k,\sigma}(s)-\tilde{V}_{h+1}^{\pi^k,\sigma}(s)\big]\big]_{i\in[d]}}\label{eq:ucb_model_pred_error1}\\
&=P_h^0\big(\tilde{V}_{h+1}^{k,\sigma}-\tilde{V}_{h+1}^{\pi^k,\sigma}\big)\sahk,\label{eq:expected_model_pred_error_value1}
\end{align}
where \eqref{eq:ucb_model_pred_error1} is due to \Cref{lemma:UCB1} (UCB). Combining equations \eqref{eq:iota_rewrite1} and \eqref{eq:expected_model_pred_error_value1}, we have
\begin{align}
\iota_h^{k,\sigma}\sahk\leq P_h^0\big(\tilde{V}_{h+1}^{k,\sigma}-\tilde{V}_{h+1}^{\pi^k,\sigma}\big)\sahk-\big(\tilde{Q}_h^{k,\sigma}\sahk-\tilde{Q}_h^{\pi^k,\sigma}\sahk\big).\notag
\end{align}
Then we have 
\begin{align}
&\tilde{V}_h^{k,\sigma}(s_h^k)-\tilde{V}_h^{\pi^k,\sigma}(s_h^k)\notag\\
=&\innerA{\tilde{Q}_h^{k,\sigma}(s_h^k,\cdot),\pi_h^k(\cdot|s_h^k)}-\frac{1}{\eta}D_{KL}[\pi_h^k(\cdot|s_h^k)\|\pi_h^0(\cdot|s_h^k)]-\innerA{\tilde{Q}_h^{\pi^k,\sigma}(s_h^k,\cdot),\pi_h^k(\cdot|s_h^k)}\notag\\
&\qquad\qquad\qquad+\frac{1}{\eta}D_{KL}[\pi_h^k(\cdot|s_h^k)\|\pi_h^0(\cdot|s_h^k)]\notag\\
=&\innerA{\tilde{Q}_h^{k,\sigma}(s_h^k,\cdot)-\tilde{Q}_h^{\pi^k,\sigma}(s_h^k,\cdot),\pi_h^k(\cdot|s_h^k)}+\iota_h^{k,\sigma}\sahk-\iota_h^{k,\sigma}\sahk\notag\\
\leq&\underbrace{\innerA{\tilde{Q}_h^{k,\sigma}(s_h^k,\cdot)-\tilde{Q}_h^{\pi^k,\sigma}(s_h^k,\cdot),\pi_h^k(\cdot|s_h^k)}-\Big(\tilde{Q}_h^{k,\sigma}\sahk-\tilde{Q}_h^{\pi^k,\sigma}\sahk\Big)}_{\kappa_h^k}-\iota_h^{k,\sigma}\sahk\label{eq:kappa_omega1}\\
&+\underbrace{\bigg(P_h^0\Big(\tilde{V}_{h+1}^{k,\sigma}-\tilde{V}_{h+1}^{\pi^k,\sigma}\Big)\sahk-\Big(\tilde{V}_{h+1}^{k,\sigma}(s_{h+1}^k)-\tilde{V}_{h+1}^{\pi^k,\sigma}(s_{h+1}^k)\Big)\bigg)}_{\omega_h^k}+\Big(\tilde{V}_{h+1}^{k,\sigma}(s_{h+1}^k)-\tilde{V}_{h+1}^{\pi^k,\sigma}(s_{h+1}^k)\Big)\notag\\
\leq&\Big(\tilde{V}_{h+1}^{k,\sigma}(s_{h+1}^k)-\tilde{V}_{h+1}^{\pi^k,\sigma}(s_{h+1}^k)\Big)+\kappa_h^k\sahk+\omega_h^k\sahk+2\Gamma_h^k\sahk,\label{eq:4_22_1}
\end{align}
where the last inequality \eqref{eq:4_22_1} is due to the bound on $\iota_h^{k,\sigma}$ given by \Cref{lemma:iota_bound}. Using the shorthand notations $\kappa_h^k$ and $\omega_h^k$ to denote the two terms in \eqref{eq:kappa_omega1}, the regret is
\begin{align}
\text{Regret}\leq&\sum_{k=1}^K[\tilde{V}_1^{k,\sigma}(s_1^k)-\tilde{V}_1^{\pi^k,\sigma}(s_1^k)]
\leq\sum_{k=1}^K\sum_{h=1}^H\kappa_h^k\sahk+\sum_{k=1}^K\sum_{h=1}^H\omega_h^k\sahk+2\sum_{k=1}^K\sum_{h=1}^H\Gamma_h^k\sahk.\notag
\end{align}

We bound the first two terms as martingales. Since $\tilde{Q}_h^{k,\sigma}\sahk\leq H$ by construction of the algorithm and $\tilde{Q}_h^{\pi^k,\sigma}\sahk\leq H$ by definition of the $\tilde{Q}$ function, we have
\begin{align}
\kappa_h^k=\innerA{\tilde{Q}_h^{k,\sigma}-\tilde{Q}_h^{\pi^k,\sigma},\pi_h^k}-\big(\tilde{Q}_h^{k,\sigma}\sahk-\tilde{Q}_h^{\pi^k,\sigma}\sahk\big)\leq2H.\notag
\end{align}
Similarly,
\begin{align}
\omega_h^k=\big(P_h^0\big(\tilde{V}_{h+1}^{k,\sigma}-\tilde{V}_{h+1}^{\pi^k,\sigma}\big)\sahk-\big(\tilde{V}_{h+1}^{k,\sigma}-\tilde{V}_{h+1}^{\pi^k,\sigma}\big)\big)\leq2H.\notag
\end{align}
Consequently, $\sum_{k=1}^K\sum_{h=1}^H(\kappa_h^k\sahk+\omega_h^k\sahk)$ is a martingale with each term bounded by $4H$. 
Then applying the Azuma-Hoeffding inequality, we have that for any $t>0$,
\begin{align}
P\Bigg(\Bigg|\sum_{k=1}^K\sum_{h=1}^H(\kappa_h^k\sahk+\omega_h^k\sahk)\Bigg|>t\Bigg)\leq2\exp\bigg(\frac{-t^2}{(4H)^2HK}\bigg).\notag
\end{align}
Then if we set $t=\sqrt{16H^3K\log(3/p)}$, we have with probability at least $1-p/3$ that
\begin{align}\label{eq:martingale_bound1}
\sum_{k=1}^K\sum_{h=1}^H(\kappa_h^k\sahk+\omega_h^k\sahk)\leq\sqrt{16H^3K\log\frac{3}{p}},
\end{align}
and thus
\begin{align}
\text{AveSubopt}(K)\leq\sqrt{\frac{16H^3}{K}\log\frac{3}{p}}+\frac{1}{K}\sum_{k=1}^K\sum_{h=1}^H\Gamma_h^k\sahk.
\end{align}
The last term represents the sum of bonuses encountered across all timesteps and episodes. Under assumption \Cref{assumption:uniform_exploration}, it can be upper-bounded by $2H\frac{\beta}{\sqrt{K}}\sqrt{\frac{128}{\alpha^2}\log\frac{3dHK}{p}+\frac{2}{\alpha}\log K}$ with probability at least $1-p/3$ according to \Cref{lemma:sum_bonus}. Together with the choice that $\beta=c_\beta dH\sqrt{\log(3dKH\vA/p)}$ in which $\vA$ denotes the measure of $\mathcal{A}$, we have that
\begin{align}
\text{AveSubopt}(K)\leq\frac{cdH^2\log(3dHK\vA/p)}{\alpha\sqrt{K}}.
\end{align}

\section{Proof of Technical Lemmas}

\subsection{Proof of \Cref{lemma:UCB} (UCB under $d$-rectangular linear DRMDP)}

To facilitate the proof of the UCB property, we first present a supporting lemma:
\begin{lemma}\label{lemma:d5}
    Recall that $\Gamma_h^k(s,a)=\beta\sum_{i=1}^d\phi_i(s,a)\sqrt{\mathbf{1}_i^\top\Lambda_h^{-1}\mathbf{1}_i}$. For any fixed policy $\pi$, on the event $\mathcal{E}$ defined in \Cref{lemma:C1}, we have for all $(s,a,h,k)\in S/{s_f}\times A\times[H]\times[K]$ that
    \begin{align*}
    \inner{\bm{\phi}(s,a),\bm{\theta}_h+\bm{\nu}_h^{\rho,k}}-\tilde{Q}_h^{\pi,\rho}(s,a)=&\infPh{h}[\PP_h\tilde{V}_{h+1}^{k,\rho}](s,a)\notag\\
    &\qquad\qquad-\infPh{h}[\PP_h\tilde{V}_{h+1}^{\pi,\rho}](s,a)+\Delta_h^k(s,a)
    \end{align*}
    for some $\Delta_h^k(s,a)$ where $|\Delta_h^k(s,a)|\leq\Gamma_h^k(s,a)$.
\end{lemma}

Subsequently, we prove the UCB property by induction:

Base Case: when $h=H$, since $Q_H^k=r_H=Q_H^\pi,\forall\pi$, we have that $\tilde{V}_H^{k,\rho}=\tilde{V}_h^{*,\rho}$.

Inductive Case: assume that $\forall s,a,k, \tilde{V}_{h+1}^{k,\rho}(s)\geq \tilde{V}_{h+1}^{*,\rho}(s)$. This implies
\begin{align}
\infPh{h}[\PP_h\tilde{V}_{h+1}^{k,\rho}](s,a)-\infPh{h}[\PP_h\tilde{V}_{h+1}^{*,\rho}](s,a)\geq0.\label{eq:inductive_case}
\end{align}
Plugging equation \eqref{eq:inductive_case} into \Cref{lemma:d5} with $\pi=\pi^*$,
\begin{align}
\bigg|\inner{\bm{\phi}(s,a),\bm{\theta}_h+\bm{\nu}_h}-\tilde{Q}_h^{*,\rho}(s,a) \qquad\qquad\qquad\qquad\qquad\qquad\qquad\qquad\qquad\qquad\qquad\qquad\notag\\
-\bigg(\infPh{h}[\PP_h\tilde{V}_{h+1}^{k,\rho}](s,a)-\infPh{h}[\PP_h\tilde{V}_{h+1}^{*,\rho}](s,a)\bigg)\bigg|&\leq\Gamma_h^k(s,a)\notag\\
\inner{\bm{\phi}(s,a),\bm{\theta}_h+\bm{\nu}_h}-\tilde{Q}_h^{*,\rho}(s,a) \qquad\qquad\qquad\qquad\qquad\qquad\qquad\qquad\qquad\qquad\qquad\qquad\notag\\
-\bigg(\infPh{h}[\PP_h\tilde{V}_{h+1}^{k,\rho}](s,a)-\infPh{h}[\PP_h\tilde{V}_{h+1}^{*,\rho}](s,a)\bigg)&\geq-\Gamma_h^k(s,a)\notag\\
\inner{\bm{\phi}(s,a),\bm{\theta}_h+\bm{\nu}_h}+\Gamma_h^k(s,a)&\geq \tilde{Q}_h^{*,\rho}(s,a)\notag\\
\tilde{Q}_h^k(s,a)&\geq \tilde{Q}_h^{*,\rho}(s,a).\notag
\end{align}
Fix $s\in\cS$ and consider the objective function $J(\pi,s)=\la\tilde{Q}_h^{k,\rho}(s,\cdot),\pi(\cdot|s)\ra-\frac{1}{\eta}D_{KL}[\pi(\cdot|s)\|\piref(\cdot|s)]$. Then by \Cref{propositon:optimization_kl}, the maximizer of $J(\pi,s)$ is
\begin{align*}
\pi(\cdot|s)=\frac{\piref(a|s)\exp\Big(\eta\tilde{Q}_h^{k,\rho}(s,a)\Big)}{\int_\cA \piref(a|s)\exp\Big(\eta\tilde{Q}_h^{k,\rho}(s,a)\Big)da}.
\end{align*}
Therefore $\pi_{h}^k(\cdot|s)=\arg\max_{\pi}\big\{\la\tilde{Q}_h^{k,\rho}(s,\cdot),\pi(\cdot|s)\ra-\frac{1}{\eta}D_{KL}[\pi(\cdot|s)\Vert\piref(\cdot|s)]\big\}$, so we have that
\begin{align}
\tilde{V}_{h}^{k,\rho}(s)&=\inner{\tilde{Q}_h^{k,\rho}(s,\cdot),\pi^k(\cdot|s)}-\frac{1}{\eta}D_{KL}[\pi^k(\cdot|s)\Vert\piref(\cdot|s)]\notag\\
&\geq \inner{\tilde{Q}_h^{k,\rho}(s,\cdot),\pi^*(\cdot|s)}-\frac{1}{\eta}D_{KL}[\pi^*(\cdot|s)\Vert\piref(\cdot|s)]\notag\\
&\geq\inner{\tilde{Q}_h^{*,\rho}(s,\cdot),\pi^*(\cdot|s)}-\frac{1}{\eta}D_{KL}[\pi^*(\cdot|s)\Vert\piref(\cdot|s)]\notag\\
&=\tilde{V}_{h}^{*,\rho}(s).\notag
\end{align}

\subsection{Proof of \Cref{lemma:UCB1} (UCB under $d$-rectangular linear RRMDP)}

Similarly to the $d$-rectangular DRMDP, we also present the following supporting lemma to help prove the UCB property:
\begin{lemma}\label{lemma:d5_1}
    Recall that $\Gamma_h^k(s,a)=\beta\sum_{i=1}^d\phi_i(s,a)\sqrt{\mathbf{1}_i^\top\Lambda_h^{-1}\mathbf{1}_i}$. For any fixed policy $\pi$, on the event $\mathcal{E}$ defined in \Cref{lemma:C1_1}, we have for all $(s,a,h,k)\in S/{s_f}\times A\times[H]\times[K]$ that
    \[\inner{\bm{\phi}(s,a),\bm{\theta}_h+\bm{\nu}_h^{\sigma,k}}-\tilde{Q}_h^{\pi,\sigma}(s,a)=\psi(\tilde{V}_{h+1}^{k,\sigma},s,a)-\psi(\tilde{V}_{h+1}^{\pi,\sigma},s,a)+\Delta_h^k(s,a)\]
    for some $\Delta_h^k(s,a)$ where $|\Delta_h^k(s,a)|\leq\Gamma_h^k(s,a)$.
\end{lemma}

Next, we prove the UCB property by induction:

Base Case: when $h=H$, since $Q_H^k=r_H=Q_H^\pi,\forall\pi$, we have that $\tilde{V}_H^{k,\sigma}=\tilde{V}_h^{*,\sigma}$.

Inductive Case: assume that $\forall s,a,k, \tilde{V}_{h+1}^{k,\sigma}(s)\geq \tilde{V}_{h+1}^{*,\sigma}(s)$. This implies that $\forall \bmu_h\in\Delta(\cS)^d$ and $P_h=\inner{\bm{\phi}(s,a),\bmu_h}$, we have that
\begin{align}
\EE_{s'\sim P_h(\cdot|s,a)}[\tilde{V}_{h+1}^{k,\sigma}(s')]\geq\EE_{s'\sim P_h(\cdot|s,a)}[\tilde{V}_{h+1}^{*,\sigma}(s')].\notag
\end{align}
From \eqref{eq:psi_def}, recall that
\begin{align}
\psi(V,s,a) = \inf_{\bmu_h\in\Delta(\cS)^d,P_h=\inner{\bm{\phi},\bmu_h}}\Big[\EE_{s'\sim P_h(\cdot|s,a)}[\tilde{V}_{h+1}^{\pi,\sigma}(s')]+\sigma\inner{\bm{\phi}(s,a),D(\bmu_h\|\bmu_h^0)}\Big],
\end{align}
and thus
\begin{align}\label{eq:inductive_case1}
\psi(\tilde{V}_{h+1}^{k,\sigma},s,a)&= \inf_{\bmu_h\in\Delta(\cS)^d,P_h=\inner{\bm{\phi},\bmu_h}}\Big[\EE_{s'\sim P_h(\cdot|s,a)}[\tilde{V}_{h+1}^{k,\sigma}(s')]+\sigma\inner{\bm{\phi}(s,a),D(\bmu_h\|\bmu_h^0)}\Big] \notag\\
&\geq\inf_{\bmu_h\in\Delta(\cS)^d,P_h=\inner{\bm{\phi},\bmu_h}}\Big[\EE_{s'\sim P_h(\cdot|s,a)}[\tilde{V}_{h+1}^{*,\sigma}(s')]+\sigma\inner{\bm{\phi}(s,a),D(\bmu_h\|\bmu_h^0)}\Big] \notag\\
&=\psi(\tilde{V}_{h+1}^{*,\sigma},s,a).
\end{align}
Plugging equation \eqref{eq:inductive_case1} into \Cref{lemma:d5_1} with $\pi=\pi^*$,
\begin{align}
\big|\inner{\bm{\phi}(s,a),\bm{\theta}_h+\bm{\nu}_h}-\tilde{Q}_h^{*,\sigma}(s,a)-\big(\psi(\tilde{V}_{h+1}^{k,\sigma},s,a)-\psi(\tilde{V}_{h+1}^{*,\sigma},s,a)\big)\big|&\leq\Gamma_h^k(s,a)\notag\\
\inner{\bm{\phi}(s,a),\bm{\theta}_h+\bm{\nu}_h}-\tilde{Q}_h^{*,\sigma}(s,a)-\big(\psi(\tilde{V}_{h+1}^{k,\sigma},s,a)-\psi(\tilde{V}_{h+1}^{*,\sigma},s,a)\big)&\geq-\Gamma_h^k(s,a)\notag\\
\inner{\bm{\phi}(s,a),\bm{\theta}_h+\bm{\nu}_h}+\Gamma_h^k(s,a)&\geq \tilde{Q}_h^{*,\sigma}(s,a)\notag\\
\tilde{Q}_h^k(s,a)&\geq \tilde{Q}_h^{*,\sigma}(s,a).\notag
\end{align}
Fix $s\in\cS$ and consider the objective function $J(\pi,s)=\la\tilde{Q}_h^{k,\rho}(s,\cdot),\pi(\cdot|s)\ra-\frac{1}{\eta}D_{KL}[\pi(\cdot|s)\|\piref(\cdot|s)]$. Then by \Cref{propositon:optimization_kl}, the maximizer of $J(\pi,s)$ is
\begin{align*}
\pi(\cdot|s)=\frac{\piref(a|s)\exp\big(\eta\tilde{Q}_h^{k,\rho}(s,a)\big)}{\int_\cA \piref(a|s)\exp\big(\eta\tilde{Q}_h^{k,\rho}(s,a)\big)da}.
\end{align*}
Therefore $\pi_{h}^k(\cdot|s)=\arg\max_{\pi}\big\{\la\tilde{Q}_h^{k,\sigma}(s,\cdot),\pi(\cdot|s)\ra-\frac{1}{\eta}D_{KL}[\pi^*(\cdot|s)\|\piref(\cdot|s)]\big\}$, so we have that
\begin{align}
\tilde{V}_{h}^{k,\sigma}(s)&=\inner{\tilde{Q}_h^{k,\sigma}(s,\cdot),\pi^k(\cdot|s)}-\frac{1}{\eta}D_{KL}[\pi^*(\cdot|s)\|\piref(\cdot|s)]\notag\\
&\geq \inner{\tilde{Q}_h^{k,\sigma}(s,\cdot),\pi^*(\cdot|s)}-\frac{1}{\eta}D_{KL}[\pi^*(\cdot|s)\|\piref(\cdot|s)]\notag\\
&\geq\inner{\tilde{Q}_h^{*,\sigma}(s,\cdot),\pi^*(\cdot|s)}-\frac{1}{\eta}D_{KL}[\pi^*(\cdot|s)\|\piref(\cdot|s)]\notag\\
&=\tilde{V}_{h}^{*,\sigma}(s).\notag
\end{align}

\subsection{Proof of \Cref{lemma:iota_bound}}

\begin{proof}
Starting off from \Cref{lemma:d5}: for any fixed policy $\pi$, on event $\mathcal{E}$, for all $(s,a,h,k)\in S\backslash\{s_f\}\times A\times[H]\times[K]$ we have
\begin{align}
\langle\bm{\phi}(s,a),\bm{\theta}_h+\bm{\nu}_h^{\rho,k}\rangle-\tilde{Q}_h^{\pi,\rho}(s,a)=&\infPh{h}[\PP_h\tilde{V}_{h+1}^{k,\rho}](s,a)\notag\\
&\qquad\qquad-\infPh{h}[\PP_h\tilde{V}_{h+1}^{\pi,\rho}](s,a)+\Delta_h^k(s,a).\notag
\end{align}
where $|\Delta_h^k(s,a)|\leq\Gamma_h^k(s,a)=\beta\sum_{i=1}^d\sqrt{\phi_i(s,a)\mathbf{1}_i^\top(\Lambda_h^k)^{-1}\phi_i(s,a)\mathbf{1}_i}$. Then because $r_h(s,a)=\bm{\phi}(s,a)^\top\bm{\theta}_h$ 
and due to the KL-regularized Robust Bellman equation in \Cref{proposition:robust_bellman_kl_drmdp}, $\tilde{Q}_h^{\pi,\rho}(s,a)=r_h(s,a)+\infPh{h}[\PP_h\tilde{V}_{h+1}^{\pi,\rho}](s,a)$, we have
\begin{align}
r_h(s,a)+\bm{\phi}_h^k(s,a)^\top\bm{\nu}_h^{\rho,k}\qquad\qquad\qquad\qquad\qquad\qquad\qquad&\notag\\
-r_h(s,a)-\infPh{h}[\PP_h\tilde{V}_{h+1}^{\pi,\rho}](s,a)\leq&\infPh{h}[\PP_h\tilde{V}_{h+1}^{k,\rho}](s,a) \notag\\
&-\infPh{h}[\PP_h\tilde{V}_{h+1}^{\pi,\rho}](s,a)+\Gamma_h^k(s,a)\notag\\
\bm{\phi}_h^k(s,a)^\top\bm{\nu}_h^{\rho,k}-\infPh{h}[\PP_h\tilde{V}_{h+1}^{k,\rho}](s,a)\leq&\Gamma_h^k(s,a).\notag
\end{align}
Then we can bound the model prediction error $\iota_h^{k,\rho}$ as defined in \eqref{eq:iota_def}. For the lower bound, $\forall(s,a)\in S\times A$ under the event $\mathcal{E}$ we have that
\begin{align}
-\iota_h^{k,\rho}(s,a)=&\tilde{Q}_h^{k,\rho}(s,a)-\Big(r_h(s,a)+\infPh{h}[\PP_h\tilde{V}_{h+1}^k](s,a)\Big)\notag\\
=&\bm{\phi}(s,a)^\top(\bm{\theta}_h^k+\bm{\nu}_h^{\rho,k})+\Gamma_h^k(s,a)-\Big(r_h(s,a)+\infPh{h}[\PP_h\tilde{V}_{h+1}^k](s,a)\Big)\notag\\
=&\bm{\phi}(s,a)^\top\bm{\nu}_h^{\rho,k}-\infPh{h}[\PP_h\tilde{V}_{h+1}^k](s,a)+\Gamma_h^k(s,a)\notag\\
\leq&2\Gamma_h^k(s,a)
.\notag
\end{align}
For the upper bound, again using \Cref{lemma:d5}, we have
\begin{align}
&r_h(s,a)+\bm{\phi}_h^k(s,a)^\top\bm{\nu}_h^{\rho,k}-r_h(s,a)-\infPh{h}[\PP_h\tilde{V}_{h+1}^{\pi,\rho}](s,a)\notag\\
&\qquad\qquad\geq\infPh{h}[\PP_h\tilde{V}_{h+1}^{k,\rho}](s,a) \notag -\infPh{h}[\PP_h\tilde{V}_{h+1}^{\pi,\rho}](s,a)-\Gamma_h^k(s,a)\notag\\
&\bm{\phi}_h^k(s,a)^\top\bm{\nu}_h^{\rho,k}-\infPh{h}[\PP_h\tilde{V}_{h+1}^{k,\rho}](s,a)\geq-\Gamma_h^k(s,a).\label{eq:bounded_by_bonus}
\end{align}
Therefore,
\begin{align}
\iota_h^{k,\rho}(s,a)=&\Big(r_h(s,a)+\infPh{h}P_h\tilde{V}_{h+1}^k\Big)(s,a)-Q_h^k(s,a)\notag\\
\leq&\infPh{h}P_h\tilde{V}_{h+1}^k(s,a)-\min\{\bm{\phi}(s,a)^\top\bm{\nu}_h^{\rho,k}+\Gamma_h^k(s,a),H-h\}\notag\\
\leq&\max\Big\{\infPh{h}P_h\tilde{V}_{h+1}^k(s,a)-\bm{\phi}(s,a)^\top\bm{\nu}_h^{\rho,k}-\Gamma_h^k(s,a),\notag \\
&\qquad \quad \infPh{h}P_h\tilde{V}_{h+1}^k(s,a)-H+h\Big\}.\notag
\end{align}
Equation \eqref{eq:bounded_by_bonus} implies that $\infPh{h}P_hV_{h+1}^k(s,a)-\bm{\phi}(s,a)^\top\bm{\nu}_h^{\rho,k}-\Gamma_h^k(s,a)\leq0$, and by definition $\parallel\tilde{V}_{h+1}^k\parallel_\infty\leq H-h$, so $\iota_h^{k,\rho}(s,a)\leq0$.
\end{proof}

\subsection{Proof of \Cref{lemma:iota_bound1}}

\begin{proof}
Starting off from \Cref{lemma:d5_1}: for any fixed policy $\pi$, on event $\mathcal{E}$, for all $(s,a,h,k)\in S\backslash\{s_f\}\times A\times[H]\times[K]$ we have
\begin{align}
\langle\bm{\phi}(s,a),\bm{\theta}_h+\bm{\nu}_h^{\sigma,k}\rangle-\tilde{Q}_h^{\pi,\sigma}(s,a)&=\psi(\tilde{V}_{h+1}^{k,\sigma},s,a)-\psi(\tilde{V}_{h+1}^{\pi,\sigma},s,a)+\Delta_h^k(s,a).\notag
\end{align}
where $|\Delta_h^k(s,a)|\leq\Gamma_h^k=\beta\sum_{i=1}^d\sqrt{\phi_i(s,a)\mathbf{1}_i^\top(\Lambda_h^k)^{-1}\phi_i(s,a)\mathbf{1}_i}$. Then because $r_h(s,a)=\bm{\phi}(s,a)^\top\bm{\theta}_h$ and due to the KL-regularized Robust Bellman equation in \Cref{proposition:robust_bellman_kl_rrmdp}, $\tilde{Q}_h^{\pi,\sigma}(s,a)=r_h(s,a)+\psi(\tilde{V}_{h+1}^{\pi,\sigma},s,a)$, we have
\begin{align}
r_h(s,a)+\bm{\phi}_h^k(s,a)^\top\bm{\nu}_h^{\sigma,k}-r_h(s,a)-\psi(\tilde{V}_{h+1}^{\pi,\sigma},s,a)&\leq\psi(\tilde{V}_{h+1}^{k,\sigma},s,a)-\psi(\tilde{V}_{h+1}^{\pi,\sigma},s,a)+\Gamma_h^k(s,a) \notag\\
\bm{\phi}_h^k(s,a)^\top\bm{\nu}_h^{\sigma,k}-\psi(\tilde{V}_{h+1}^{k,\sigma},s,a)&\leq\Gamma_h^k(s,a). \notag
\end{align}
Then we can bound the model prediction error $\iota_h^{k,\sigma}$ as defined in \eqref{eq:iota_def1}. For the lower bound, $\forall(s,a)\in S\times A$ under the event $\mathcal{E}$ we have that
\begin{align}
-\iota_h^{k,\sigma}(s,a)=&\tilde{Q}_h^{k,\sigma}(s,a)-(r_h(s,a)+\psi(\tilde{V}_{h+1}^{k,\sigma},s,a))\notag \\
=&\bm{\phi}(s,a)^\top(\bm{\theta}_h^k+\bm{\nu}_h^{\sigma,k})+\Gamma_h^k(s,a)-(r_h(s,a)+\psi(\tilde{V}_{h+1}^{k,\sigma},s,a))\notag\\
=&\bm{\phi}(s,a)^\top\bm{\nu}_h^{\sigma,k}-\psi(\tilde{V}_{h+1}^{k,\sigma},s,a)+\Gamma_h^k(s,a) \notag\\
\leq&2\Gamma_h^k(s,a). \notag
\end{align}
For the upper bound, again using \Cref{lemma:d5_1}, we have
\begin{align}
r_h(s,a)+\bm{\phi}_h^k(s,a)^\top\bm{\nu}_h^{\sigma,k}-r_h(s,a)-\psi(\tilde{V}_{h+1}^{\pi,\sigma},s,a)&\geq\psi(\tilde{V}_{h+1}^{k,\sigma},s,a)-\psi(\tilde{V}_{h+1}^{\pi,\sigma},s,a)-\Gamma_h^k(s,a) \notag\\
\bm{\phi}_h^k(s,a)^\top\bm{\nu}_h^{\sigma,k}-\psi(\tilde{V}_{h+1}^{k,\sigma},s,a)&\geq-\Gamma_h^k(s,a). \label{eq:bounded_by_bonus1}
\end{align}
Therefore,
\begin{align}
\iota_h^{k,\sigma}(s,a)=&(r_h(s,a)+\psi(\tilde{V}_{h+1}^{k,\sigma},s,a))-Q_h^k(s,a)\notag\\
\leq&\psi(\tilde{V}_{h+1}^{k,\sigma},s,a)-\min\{\bm{\phi}(s,a)^\top\bm{\nu}_h^{\sigma,k}+\Gamma_h^k(s,a),H-h\}\notag\\
\leq&\max\big\{\psi(\tilde{V}_{h+1}^{k,\sigma},s,a)-\bm{\phi}(s,a)^\top\bm{\nu}_h^{\sigma,k}-\Gamma_h^k(s,a),\psi(\tilde{V}_{h+1}^{k,\sigma},s,a)-H+h\big\}.\notag
\end{align}
Equation \eqref{eq:bounded_by_bonus1} implies that $\psi(\tilde{V}_{h+1}^{k,\sigma},s,a)-\bm{\phi}(s,a)^\top\bm{\nu}_h^{\sigma,k}-\Gamma_h^k(s,a)\leq0$, and by definition $\|\tilde{V}_{h+1}^k\|_\infty\leq H-h$, so $\iota_h^{k,\sigma}(s,a)\leq0$.
\end{proof}

\section{Proof of Supporting Lemmas}

\subsection{Proof of \Cref{lemma:d5}}
\begin{proof}
We will utilize the following supporting lemma:
\begin{lemma}\label{lemma:C1}
There exists an absolute constant $C$ independent of $c_\beta$ such that $\forall p\in[0,1]$, the event $\mathcal{E}$ that for all $(k,h)\in[K]\times[H]$,
\begin{align}
\Bigg\Vert\sum_{\tau=1}^{k-1}\bm{\phi}_h^\tau\Big[[\tilde{V}_{h+1}^{k,\rho}(s_{h+1}^\tau)]_\alpha-[\PP_h^0[\tilde{V}_{h+1}^{k,\rho}]_\alpha](s_{h+1}^\tau,a_{h+1}^\tau)\Big]\Bigg\Vert^2_{(\Lambda_h^k)^{-1}}\leq Cd^2H^2\log[3(c_\beta+1)dKH\vA/p]\notag
\end{align}
has probability at least $1-p/3$.
\end{lemma}

Recall that by definition, $\bm{\nu}_{h,i}^{k,\rho}=\max_{\alpha\in[0,H]}\{((\Lambda_h^k)^{-1}[\sum_{\tau=1}^{k-1}\bm{\phi}_h^\tau[\tilde{V}_{h+1}^{k,\rho}(s_{h+1}^\tau)]_\alpha])_i-\rho\alpha\}$ and $\bnu_{h,i}^{\pi,\rho}=\inf_{\bmu\in U_h^\rho(\bmu_h^0)}\EE_{s\sim\bmu}\tilde{V}_{h+1}^{\pi,\rho}(s)$ $=\max_{\alpha\in[0,H]}\{\EE_{s\sim\bmu_{h,i}^0}[\tilde{V}_{h+1}^{\pi,\rho}(s)]_\alpha-\rho\alpha\}$. To help bound their difference, we introduce a middle term $\bar{\bnu}_{h,i}^{k,\rho}=\inf_{\bmu_{h,i}\in U_h^\rho(\bmu_{h,i}^0)}\EE_{s\sim\bmu_{h,i}}\tilde{V}_{h+1}^{k,\rho}(s)=\max_{\alpha\in[0,H]}\{\EE_{s\sim\bmu_{h,i}^0}[\tilde{V}_{h+1}^{k,\rho}(s)]_\alpha-\rho\alpha\}$, so that $\bnu_h^{k,\rho}$ and $\bar{\bnu}_h^{k,\rho}$ depend on the same value function $\tilde{V}_{h+1}^{k,\rho}$ and $\bar{\bnu}_h^{k,\rho}-\bnu_h^{\pi,\rho}$ relates to the RHS of the equation we want to prove.
For any fixed policy $\pi$, decompose
\begin{align}
(\bm{\theta}_h+\bm{\nu}_h^{k,\rho})-(\bm{\theta}_h+\bm{\nu}_h^{\pi,\rho})=\bm{\nu}_h^{k,\rho}-\bm{\nu}_h^{\pi,\rho}=\underbrace{\bm{\nu}_h^{k,\rho}-\bar{\bm{\nu}}_h^{k,\rho}}_{\text{(I)}}+\underbrace{\bar{\bm{\nu}}_h^{k,\rho}-\bm{\nu}_h^{\pi,\rho}}_{\text{(II)}}.\notag
\end{align}
To bound term (I), we have
\begin{align}
\bm{\nu}_{h,i}^{k,\rho}-\bar{\bm{\nu}}_{h,i}^{k,\rho}&=\max_{\alpha\in[0,H]}\Bigg\{\Bigg((\Lambda_h^k)^{-1}\Bigg[\sum_{\tau=1}^{k-1}\bm{\phi}_h^\tau[\tilde{V}_{h+1}^{k,\rho}(s_{h+1}^\tau)]_\alpha\Bigg]\Bigg)_i-\rho\alpha\Bigg\}-\max_{\alpha\in[0,H]}\Big\{\EE_{s\sim\bmu_{h,i}^0}[\tilde{V}_{h+1}^{k,\rho}(s)]_\alpha-\rho\alpha\Big\}\notag\\
&\leq\max_{\alpha\in[0,H]}\Bigg\{\Bigg((\Lambda_h^k)^{-1}\Bigg[\sum_{\tau=1}^{k-1}\bm{\phi}_h^\tau[\tilde{V}_{h+1}^{k,\rho}(s_{h+1}^\tau)]_\alpha\Bigg]\Bigg)_i-\EE_{s\sim\bmu_{h,i}^0}[\tilde{V}_{h+1}^{k,\rho}(s)]_\alpha\Bigg\},\notag
\end{align}
where $\bphi_h^\tau$ is a shorthand for $\bphi(s_h^\tau,a_h^\tau)$. Denote $\alpha_i^k=\arg\max_{\alpha\in[0,H]}\big\{\big((\Lambda_h^k)^{-1}\big[\sum_{\tau=1}^{k-1}\bm{\phi}_h^\tau[\tilde{V}_{h+1}^{k,\rho}(s_{h+1}^\tau)]_\alpha\big]\big)_i-\EE_{s\sim\bmu_{h,i}^0}[\tilde{V}_{h+1}^{k,\rho}(s)]_\alpha\big\}$. Then we have 
\begin{align}
&\bm{\nu}_{h,i}^{k,\rho}-\bar{\bm{\nu}}_{h,i}^{k,\rho}\notag \\
\leq&\Bigg((\Lambda_h^k)^{-1}\Bigg[\sum_{\tau=1}^{k-1}\bm{\phi}_h^\tau[\tilde{V}_{h+1}^{k,\rho}(s_{h+1}^\tau)]_{\alpha_i^k}\Bigg]\Bigg)_i-\bigg((\Lambda_h^k)^{-1}(\Lambda_h^k)\EE_{s\sim\bmu_{h}^0}[\tilde{V}_{h+1}^{k,\rho}(s)]_{\alpha_i^k}\bigg)_i
\notag\\
=&\Bigg((\Lambda_h^k)^{-1}\Bigg[\sum_{\tau=1}^{k-1}\bm{\phi}_h^\tau[\tilde{V}_{h+1}^{k,\rho}(s_{h+1}^\tau)]_{\alpha_i^k}\Bigg]\Bigg)_i-\Bigg((\Lambda_h^k)^{-1}\bigg(\lambda I+\sum_{\tau=1}^{k-1}\bphi_h^\tau(\bphi_h^\tau)^\top\bigg)\EE_{s\sim\bmu_{h}^0}[\tilde{V}_{h+1}^{k,\rho}(s)]_{\alpha_i^k}\Bigg)_i\notag\\
=&\bigg(-\lambda(\Lambda_h^k)^{-1}\EE_{s\sim\bmu_{h}^0}[\tilde{V}_{h+1}^{k,\rho}(s)]_{\alpha_i^k}\bigg)_i+\Bigg((\Lambda_h^k)^{-1}\Bigg[\sum_{\tau=1}^{k-1}\bm{\phi}_h^\tau[\tilde{V}_{h+1}^{k,\rho}(s_{h+1}^\tau)]_{\alpha_i^k}\Bigg]\Bigg)_i\notag\\
&-\Bigg((\Lambda_h^k)^{-1}\Bigg(\sum_{\tau=1}^{k-1}\bphi_h^\tau(\bphi_h^\tau)^\top\Bigg)\EE_{s\sim\bmu_{h}^0}[\tilde{V}_{h+1}^{k,\rho}(s)]_{\alpha_i^k}\Bigg)_i\notag\\
=&\bigg(-\lambda(\Lambda_h^k)^{-1}\EE_{s\sim\bmu_{h}^0}[\tilde{V}_{h+1}^{k,\rho}(s)]_{\alpha_i^k}\bigg)_i+\Bigg((\Lambda_h^k)^{-1}\Bigg[\sum_{\tau=1}^{k-1}\bm{\phi}_h^\tau[\tilde{V}_{h+1}^{k,\rho}(s_{h+1}^\tau)]_{\alpha_i^k}\Bigg]\Bigg)_i \notag\\
&\qquad-\Bigg((\Lambda_h^k)^{-1}\sum_{\tau=1}^{k-1}\bphi_h^\tau\EE_{s\sim P_h^0(\cdot|s_h^\tau,a_h^\tau)}[\tilde{V}_{h+1}^{k,\rho}(s)]_{\alpha_i^k}\Bigg)_i
\notag\\
=&\underbrace{\Bigg(-\lambda(\Lambda_h^k)^{-1}\EE_{s\sim\bmu_{h}^0}[\tilde{V}_{h+1}^{k,\rho}(s)]_{\alpha_i^k}\Bigg)_i}_{\text{(III)}}+\underbrace{\Bigg((\Lambda_h^k)^{-1}\bigg[\sum_{\tau=1}^{k-1}\bm{\phi}_h^\tau[\tilde{V}_{h+1}^{k,\rho}(s_{h+1}^\tau)]_{\alpha_i^k}-\Big[\PP_{h}^0[\tilde{V}_{h+1}^{k,\rho}]_{\alpha_i^k}\Big](s_h^\tau,a_h^\tau)\bigg]\Bigg)_i}_{\text{(IV)}}, \label{eq:decomp_3-4}
\end{align}
where the first equality is from the definition that $\Lambda_h^k=\lambda I+\sum_{\tau=1}^{k-1}\bphi_h^\tau(\bphi_h^\tau)^\top$, the third equality is from the linear structure of transition model in \Cref{assumption:linearMDP}.

For term (III),
\begin{align}
\Big|\inner{\bm{\phi}(s,a),[\text{(III)}]_{i\in[d]}}\Big|&=\Bigg|\sum_{i=1}^d\phi_i(s,a)\mathbf{1}_i^\top(-\lambda)\big(\Lambda_h^k\big)^{-1}\EE_{\bmu_{h}^0}[\tilde{V}_{h+1}^{k,\rho}(s)]_{\alpha_i^k}\Bigg|\notag\\
&\leq\lambda\sum_{i=1}^d\bigg|\phi_i(s,a)\mathbf{1}_i^\top\big(\Lambda_h^k\big)^{-1/2}\big(\Lambda_h^k\big)^{-1/2}\EE_{\bmu_{h}^0}[\tilde{V}_{h+1}^{k,\rho}(s)]_{\alpha_i^k}\bigg|\notag\\
&\leq\lambda\sum_{i=1}^d\big\Vert\phi_i(s,a)\mathbf{1}_i\big\Vert_{\big(\Lambda_h^k\big)^{-1}}\Big\Vert\EE_{\bmu_{h}^0}[\tilde{V}_{h+1}^{k,\rho}(s)]_{\alpha_i^k}\Big\Vert_{\big(\Lambda_h^k\big)^{-1}}\label{eq:4_13}\\
&\leq\sqrt{\lambda}H\sum_{i=1}^d\big\Vert\phi_i(s,a)\mathbf{1}_i\big\Vert_{\big(\Lambda_h^k\big)^{-1}},\label{eq:4_14}
\end{align}
where \eqref{eq:4_13} holds due to Cauchy-Schwartz and \eqref{eq:4_14} is because $\big\Vert\EE_{\bmu_{h,i}^0}[V_{h+1}^{k,\rho}(s)]_{\alpha_i^k}\big\Vert_\infty\leq H$ and $\Lambda_h^k\succeq\lambda$.
For term (IV),
\begin{align}
&\big|\inner{\bm{\phi}(s,a),[\text{(IV)}]_{i\in[d]}}\big|\notag\\
&=\Bigg|\sum_{i=1}^d\phi_i(s,a)\mathbf{1}_i^\top(-\lambda)\big(\Lambda_h^k\big)^{-1}\Bigg[\sum_{\tau=1}^{k-1}\bm{\phi}_h^\tau[\tilde{V}_{h+1}^{k,\rho}(s_{h+1}^\tau)]_{\alpha_i^k}-\Big[\PP_{h}^0[\tilde{V}_{h+1}^{k,\rho}]_{\alpha_i^k}\Big](s_h^\tau,a_h^\tau)\Bigg]\Bigg|\notag\\
&\leq\lambda\sum_{i=1}^d\Bigg|\phi_i(s,a)\mathbf{1}_i^\top\big(\Lambda_h^k\big)^{-1/2}\big(\Lambda_h^k\big)^{-1/2}\Bigg[\sum_{\tau=1}^{k-1}\bm{\phi}_h^\tau[\tilde{V}_{h+1}^{k,\rho}(s_{h+1}^\tau)]_{\alpha_i^k}-\Big[\PP_{h}^0[\tilde{V}_{h+1}^{k,\rho}]_{\alpha_i^k}\Big](s_h^\tau,a_h^\tau)\Bigg]\Bigg|\notag\\
&\leq\lambda\sum_{i=1}^d\big\Vert\phi_i(s,a)\mathbf{1}_i\big\Vert_{\big(\Lambda_h^k\big)^{-1}}\Bigg\Vert\sum_{\tau=1}^{k-1}\bm{\phi}_h^\tau[\tilde{V}_{h+1}^{k,\rho}(s_{h+1}^\tau)]_{\alpha_i^k}-\Big[\PP_{h}^0[\tilde{V}_{h+1}^{k,\rho}]_{\alpha_i^k}\Big](s_h^\tau,a_h^\tau)\Bigg\Vert_{\big(\Lambda_h^k\big)^{-1}}\notag\\
&\leq CdH\sqrt{\chi}\sum_{i=1}^d\big\Vert\phi_i(s,a)\mathbf{1}_i\big\Vert_{\big(\Lambda_h^k\big)^{-1}},\notag
\end{align}
where the bound on $\big\Vert\sum_{\tau=1}^{k-1}\bm{\phi}_h^\tau[\tilde{V}_{h+1}^{k,\rho}(s_{h+1}^\tau)]_{\alpha_i^k}-\big[\PP_{h}^0[\tilde{V}_{h+1}^{k,\rho}]_{\alpha_i^k}\big](s_h^\tau,a_h^\tau)\big\Vert_{(\Lambda_h^k)^{-1}}$ is due to \Cref{lemma:C1} and $\chi$ denotes $\log[3(c_\beta+1)dKH\vA/p]$. For term (II),
\begin{align}
\inner{\bm{\phi}(s,a),\bar{\bm{\nu}}_h^{k,\rho}-\bm{\nu}_h^{\pi,\rho}}=&\inner{\bm{\phi}(s,a),\inf_{\bmu\in \cU_h^\rho(\bmu_h^0)}\EE_{s\sim\bmu}\tilde{V}_{h+1}^{k,\rho}(s)-\inf_{\bmu\in \cU_h^\rho(\bmu_h^0)}\EE_{s\sim\bmu}\tilde{V}_{h+1}^{\pi,\rho}(s)}\notag\\
=&\infPh{h}[\PP_h\tilde{V}_{h+1}^{k,\rho}](s,a)-\infPh{h}[\PP_h\tilde{V}_{h+1}^{\pi,\rho}](s,a).\notag
\end{align}
Then combining these terms, there exists some constant $c$ such that by setting $\lambda=1,\beta=cdH\sqrt{\chi}$ we have
\begin{align}
&\bigg|\inner{\bm{\phi}(s,a),(\bm{\theta}_h+\bm{\nu}_h^{k,\rho})}-Q_h^{\pi,\rho}(s,a)\notag\\
&\qquad\qquad-\bigg(\infPh{h}[\PP_h\tilde{V}_{h+1}^{k,\rho}](s,a)-\infPh{h}[\PP_h\tilde{V}_{h+1}^{\pi,\rho}](s,a)\bigg)\bigg|\notag\\
=&
\big|\la\bm{\phi}(s,a),(\bm{\theta}_h+\bm{\nu}_h^{k,\rho})\ra-\inner{\bm{\phi}(s,a),(\bm{\theta}_h+\bm{\nu}_h^{\pi,\rho})}-\la\bm{\phi}(s,a),\bar{\bm{\nu}}_h^{k,\rho}-\bm{\nu}_h^{\pi,\rho}\ra\big|\notag\\
=&
\big|\la\bm{\phi}(s,a),(\bm{\nu}_h^{k,\rho}-\bar{\bm{\nu}}_h^{k,\rho}\ra\big|\notag\\
\leq&\inner{\bm{\phi}(s,a),[\text{(III)}]_{i\in[d]}}+\inner{\bm{\phi}(s,a),[\text{(IV)}]_{i\in[d]}}
\label{eq:use_decomp_3-4}\\
\leq&\sqrt{\lambda}H\sum_{i=1}^d\big\Vert\phi_i(s,a)\mathbf{1}_i\big\Vert_{\big(\Lambda_h^k\big)^{-1}}+{CdH}\sqrt{\chi}\sum_{i=1}^d\big\Vert\phi_i(s,a)\mathbf{1}_i\big\Vert_{\big(\Lambda_h^k\big)^{-1}}\notag\\
\leq&\beta\sum_{i=1}^d\big\Vert\phi_i(s,a)\mathbf{1}_i\big\Vert_{\big(\Lambda_h^k\big)^{-1}}=\Gamma_h^k(s,a),\notag
\end{align}
where \eqref{eq:use_decomp_3-4} is due to the decomposition of term (I) in \eqref{eq:decomp_3-4} and the triangle inequality.
\end{proof}

\subsection{Proof of \Cref{lemma:d5_1}}

\begin{proof}

We will utilize the following supporting lemma:
\begin{lemma}\label{lemma:C1_1}
There exists an absolute constant $C$ independent of $c_\beta$ such that $\forall p\in[0,1]$, the event $\mathcal{E}$ that for all $(k,h)\in[K]\times[H]$,
\begin{align}
\Bigg\Vert\sum_{\tau=1}^{k-1}\bm{\phi}_h^\tau\Big[[\tilde{V}_{h+1}^{k,\sigma}(s_{h+1}^\tau)]_\sigma-[\PP_h^0[\tilde{V}_{h+1}^{k,\sigma}]_\sigma](s_{h+1}^\tau,a_{h+1}^\tau)\Big]\Bigg\Vert^2_{(\Lambda_h^k)^{-1}}\leq Cd^2H^2\log[3(c_\beta+1)dKH\vA/p]\notag
\end{align}
has probability at least $1-p/3$.
\end{lemma}
Recall that $\bm{\nu}_{h,i}^{k,\sigma}=\big((\Lambda_h^k)^{-1}\big[\sum_{\tau=1}^{k-1}\bm{\phi}_h^\tau\big[\tilde{V}_{h+1}^{k,\sigma}(s_{h+1}^\tau)\big]_{\sigma}\big]\big)_i$ and $\bm{\nu}_{h,i}^{\pi,\sigma}=\inf_{\bmu\in\Delta(\cS)^d}[\EE_{s\sim\bmu}\tilde{V}_{h+1}^{\pi,\sigma}(s)+\sigma D(\bmu\|\bmu^0)]=\EE_{s\sim\bmu_i^0}[\tilde{V}_{h+1}^{\pi,\sigma}(s)]_\sigma$. To help bound their difference, we introduce a middle term $\bar{\bm{\nu}}_{h,i}^{k,\sigma}=\inf_{\bmu\in\Delta(\cS)^d}[\EE_{s\sim\bmu}\tilde{V}_{h+1}^{k,\sigma}(s)+\sigma D(\bmu\|\bmu^0)]=\EE_{s\sim\bmu_i^0}[\tilde{V}_{h+1}^{k,\sigma}(s)]_\sigma$, so that $\bnu_h^{k,\sigma}$ and $\bar{\bnu}_h^{k,\sigma}$ depend on the same value function $\tilde{V}_{h+1}^{k,\sigma}$ and $\bar{\bnu}_h^{k,\sigma}-\bnu_h^{\pi,\sigma}$ relates to the RHS of the equation we want to prove.
For any fixed policy $\pi$, decompose
\begin{align}
(\bm{\theta}_h+\bm{\nu}_h^{k,\sigma})-(\bm{\theta}_h+\bm{\nu}_h^{\pi,\sigma})=\bm{\nu}_h^{k,\sigma}-\bm{\nu}_h^{\pi,\sigma}=\underbrace{\bm{\nu}_h^{k,\sigma}-\bar{\bm{\nu}}_h^{k,\sigma}}_{\text{(I)}}+\underbrace{\bar{\bm{\nu}}_h^{k,\sigma}-\bm{\nu}_h^{\pi,\sigma}}_{\text{(II)}}.\notag
\end{align}
To bound term (I), we have
\begin{align}
&\bm{\nu}_{h,i}^{k,\sigma}-\bar{\bm{\nu}}_{h,i}^{k,\sigma}\notag\\
\leq&\Bigg((\Lambda_h^k)^{-1}\Bigg[\sum_{\tau=1}^{k-1}\bm{\phi}_h^\tau[\tilde{V}_{h+1}^{k,\sigma}(s_{h+1}^\tau)]_{\sigma}\Bigg]\Bigg)_i-\bigg((\Lambda_h^k)^{-1}(\Lambda_h^k)\EE_{s\sim\bmu_{h}^0}[\tilde{V}_{h+1}^{k,\sigma}(s)]_{\sigma}\bigg)_i
\notag
\\
=&\Bigg((\Lambda_h^k)^{-1}\Bigg[\sum_{\tau=1}^{k-1}\bm{\phi}_h^\tau[\tilde{V}_{h+1}^{k,\sigma}(s_{h+1}^\tau)]_{\sigma}\Bigg]\Bigg)_i-\Bigg((\Lambda_h^k)^{-1}\bigg(\lambda I+\sum_{\tau=1}^{k-1}\bphi_h^\tau(\bphi_h^\tau)^\top\bigg)\EE_{s\sim\bmu_{h}^0}[\tilde{V}_{h+1}^{k,\sigma}(s)]_{\sigma}\Bigg)_i\notag\\
=&\bigg(-\lambda(\Lambda_h^k)^{-1}\EE_{s\sim\bmu_{h}^0}[\tilde{V}_{h+1}^{k,\sigma}(s)]_{\sigma}\bigg)_i+\Bigg((\Lambda_h^k)^{-1}\Bigg[\sum_{\tau=1}^{k-1}\bm{\phi}_h^\tau[\tilde{V}_{h+1}^{k,\sigma}(s_{h+1}^\tau)]_{\sigma}\Bigg]\Bigg)_i\notag\\
&-\Bigg((\Lambda_h^k)^{-1}\Bigg(\sum_{\tau=1}^{k-1}\bphi_h^\tau(\bphi_h^\tau)^\top\Bigg)\EE_{s\sim\bmu_{h}^0}[\tilde{V}_{h+1}^{k,\sigma}(s)]_{\sigma}\Bigg)_i\notag\\
=&\bigg(-\lambda(\Lambda_h^k)^{-1}\EE_{s\sim\bmu_{h}^0}[\tilde{V}_{h+1}^{k,\sigma}(s)]_{\sigma}\bigg)_i+\Bigg((\Lambda_h^k)^{-1}\Bigg[\sum_{\tau=1}^{k-1}\bm{\phi}_h^\tau[\tilde{V}_{h+1}^{k,\sigma}(s_{h+1}^\tau)]_{\sigma}\Bigg]\Bigg)_i \notag\\
&\qquad-\Bigg((\Lambda_h^k)^{-1}\sum_{\tau=1}^{k-1}\bphi_h^\tau\EE_{s\sim P_h^0(\cdot|s_h^\tau,a_h^\tau)}[\tilde{V}_{h+1}^{k,\sigma}(s)]_{\sigma}\Bigg)_i
\notag\\
=&\underbrace{\Bigg(-\lambda(\Lambda_h^k)^{-1}\EE_{s\sim\bmu_{h}^0}[\tilde{V}_{h+1}^{k,\sigma}(s)]_{\sigma}\Bigg)_i}_{\text{(III)}}+\underbrace{\Bigg((\Lambda_h^k)^{-1}\bigg[\sum_{\tau=1}^{k-1}\bm{\phi}_h^\tau[\tilde{V}_{h+1}^{k,\sigma}(s_{h+1}^\tau)]_{\sigma}-\Big[\PP_{h}^0[\tilde{V}_{h+1}^{k,\sigma}]_{\sigma}\Big](s_h^\tau,a_h^\tau)\bigg]\Bigg)_i}_{\text{(IV)}}, \label{eq:decomp_3-4_1}
\end{align}
where the first equality is from the definition that $\Lambda_h^k=\lambda I+\sum_{\tau=1}^{k-1}\bphi_h^\tau(\bphi_h^\tau)^\top$, the third equality is from the linear structure of transition model in \Cref{assumption:linearMDP}.

For term (III),
\begin{align}
\Big|\inner{\bm{\phi}(s,a),[\text{(III)}]_{i\in[d]}}\Big|&=\Bigg|\sum_{i=1}^d\phi_i(s,a)\mathbf{1}_i^\top(-\lambda)\big(\Lambda_h^k\big)^{-1}\EE_{\bmu_{h,i}^0}[\tilde{V}_{h+1}^{k,\sigma}(s)]_{\sigma}\Bigg|\notag\\
&\leq\lambda\sum_{i=1}^d\Bigg|\phi_i(s,a)\mathbf{1}_i^\top\big(\Lambda_h^k\big)^{-1/2}\big(\Lambda_h^k\big)^{-1/2}\EE_{\bmu_{h,i}^0}[\tilde{V}_{h+1}^{k,\sigma}(s)]_{\sigma}\Bigg|\notag\\
&\leq\lambda\sum_{i=1}^d\big\Vert\phi_i(s,a)\mathbf{1}_i\big\Vert_{\big(\Lambda_h^k\big)^{-1}}\Big\Vert\EE_{\bmu_{h,i}^0}[\tilde{V}_{h+1}^{k,\sigma}(s)]_{\sigma}\Big\Vert_{\big(\Lambda_h^k\big)^{-1}}\label{eq:4_13_1}\\
&\leq\sqrt{\lambda}H\sum_{i=1}^d\big\Vert\phi_i(s,a)\mathbf{1}_i\big\Vert_{\big(\Lambda_h^k\big)^{-1}},\label{eq:4_14_1}
\end{align}
where \eqref{eq:4_13_1} holds due to Cauchy-Schwartz and \eqref{eq:4_14_1} is because $\big\Vert\EE_{\bmu_{h,i}^0}[V_{h+1}^{k,\sigma}(s)]_{\sigma}\big\Vert_\infty\leq H$ and $\Lambda_h^k\succeq\lambda$.
For term (IV),
\begin{align}
&\big|\la\bm{\phi}(s,a),[\text{(IV)}]_{i\in[d]}\ra\big|\notag\\
&=\Bigg|\sum_{i=1}^d\phi_i(s,a)\mathbf{1}_i^\top(-\lambda)\big(\Lambda_h^k\big)^{-1}\Bigg[\sum_{\tau=1}^{k-1}\bm{\phi}_h^\tau[\tilde{V}_{h+1}^{k,\sigma}(s_{h+1}^\tau)]_{\sigma}-\Big[\PP_{h}^0[\tilde{V}_{h+1}^{k,\sigma}]_{\sigma}(s_h^\tau,a_h^\tau)\Big]\Bigg]\Bigg|\notag\\
&\leq\lambda\sum_{i=1}^d\Bigg|\phi_i(s,a)\mathbf{1}_i^\top\big(\Lambda_h^k\big)^{-1/2}\big(\Lambda_h^k\big)^{-1/2}\Bigg[\sum_{\tau=1}^{k-1}\bm{\phi}_h^\tau[\tilde{V}_{h+1}^{k,\sigma}(s_{h+1}^\tau)]_{\sigma}-\Big[\PP_{h}^0[\tilde{V}_{h+1}^{k,\sigma}]_{\sigma}(s_h^\tau,a_h^\tau)\Big]\Bigg]\Bigg|\notag\\
&\leq\lambda\sum_{i=1}^d\big\Vert\phi_i(s,a)\mathbf{1}_i\big\Vert_{\big(\Lambda_h^k\big)^{-1}}\Bigg\Vert\sum_{\tau=1}^{k-1}\bm{\phi}_h^\tau[\tilde{V}_{h+1}^{k,\sigma}(s_{h+1}^\tau)]_{\sigma}-\Big[\PP_{h}^0[\tilde{V}_{h+1}^{k,\sigma}]_{\sigma}(s_h^\tau,a_h^\tau)\Big]\Bigg\Vert_{\big(\Lambda_h^k\big)^{-1}}\notag\\
&\leq CdH\sqrt{\chi}\sum_{i=1}^d\big\Vert\phi_i(s,a)\mathbf{1}_i\big\Vert_{\big(\Lambda_h^k\big)^{-1}},\notag
\end{align}
where the bound on $\big\Vert\sum_{\tau=1}^{k-1}\bm{\phi}_h^\tau[\tilde{V}_{h+1}^{k,\sigma}(s_{h+1}^\tau)]_{\sigma}-\big[\PP_{h}^0[\tilde{V}_{h+1}^{k,\sigma}]_{\sigma}(s_h^\tau,a_h^\tau)\big]\big\Vert_{(\Lambda_h^k)^{-1}}$ is due to \Cref{lemma:C1_1}, where $\chi$ denotes $\log[3(c_\beta+1)dKH\vA/p]$. For term (II),
\begin{align}
\inner{\bm{\phi}(s,a),\bar{\bm{\nu}}_h^{k,\sigma}-\bm{\nu}_h^{\pi,\sigma}}
=&\Big\la\bm{\phi}(s,a),\inf_{\bmu\in \Delta(\cS)^d,P=\inner{\bphi,\bmu}}\EE_{s\sim P(\cdot|s,a)}[\tilde{V}_{h+1}^{k,\sigma}(s)+\sigma D(\bmu\|\bmu^0)]\notag\\
&\qquad\qquad\qquad-\inf_{\bmu\in \Delta(\cS)^d,P=\inner{\bphi,\bmu}}[\EE_{s\sim P(\cdot|s,a)}\tilde{V}_{h+1}^{\pi,\sigma}(s)+\sigma D(\bmu\|\bmu^0)]\Big\ra\notag\\
=&\psi(\tilde{V}_{h+1}^{k,\sigma},s,a)-\psi(\tilde{V}_{h+1}^{\pi,\sigma},s,a).\notag
\end{align}
Then combining these terms, there exists some constant $c$ such that by setting $\lambda=1,\beta=cdH\sqrt{\chi}$ we have
\begin{align}
&\bigg|\inner{\bm{\phi}(s,a),(\bm{\theta}_h+\bm{\nu}_h^{k,\sigma})}-Q_h^{\pi,\sigma}(s,a)\notag\\
&\qquad\qquad\qquad-\bigg(\inf_{\bmu\in \Delta(\cS)^d,P_h=\inner{\bphi,\bmu}}[\PP_h\tilde{V}_{h+1}^{k,\sigma}](s,a)-\inf_{\bmu\in \Delta(\cS)^d,P_h=\inner{\bphi,\bmu}}[\PP_h\tilde{V}_{h+1}^{\pi,\sigma}](s,a)\bigg)\bigg|\notag\\
=&
\big|\la\bm{\phi}(s,a),(\bm{\theta}_h+\bm{\nu}_h^{k,\sigma})\ra-\la\bm{\phi}(s,a),(\bm{\theta}_h+\bm{\nu}_h^{\pi,\sigma})\ra-\la\bm{\phi}(s,a),\bar{\bm{\nu}}_h^{k,\sigma}-\bm{\nu}_h^{\pi,\sigma}\ra\big|\notag\\
=&
\big|\la\bm{\phi}(s,a),(\bm{\nu}_h^{k,\sigma}-\bar{\bm{\nu}}_h^{k,\sigma}\ra\big|\notag\\
\leq&\inner{\bm{\phi}(s,a),[\text{(III)}]_{i\in[d]}}+\inner{\bm{\phi}(s,a),[\text{(IV)}]_{i\in[d]}}
\label{eq:use_decomp_3-4_1}\\
\leq&\sqrt{\lambda}H\sum_{i=1}^d\big\Vert\phi_i(s,a)\mathbf{1}_i\big\Vert_{\big(\Lambda_h^k\big)^{-1}}+{cdH}\sqrt{\chi}\sum_{i=1}^d\big\Vert\phi_i(s,a)\mathbf{1}_i\big\Vert_{\big(\Lambda_h^k\big)^{-1}}\notag\\
\leq&\beta\sum_{i=1}^d\big\Vert\phi_i(s,a)\mathbf{1}_i\big\Vert_{\big(\Lambda_h^k\big)^{-1}}=\Gamma_h^k(s,a),\notag
\end{align}
where \eqref{eq:use_decomp_3-4_1} is due to the decomposition of term (I) in \eqref{eq:decomp_3-4_1} and the triangle inequality.
\end{proof}

\subsection{Proof of \Cref{lemma:C1}}

\begin{proof}

To facilitate the proof, we will rely on the following lemmas:

\begin{lemma}\label{Lemma:sum_quadratic_form}
\citep[Lemma D.1]{jin2020} If $\Lambda_t=\lambda I+\sum_{i=1}^t\phi_i\phi_i^\top$ where $\phi_i\in\mathbb{R}^d,\lambda>0$, then $\sum_{i=1}^t\phi_i^\top(\Lambda_t)^{-1}\phi_i\leq d.$
\end{lemma}

\begin{lemma}\label{Lemma:abbasi_yadkori}
(Concentration of Self-Normalized Processes \citep[Theorem 1]{abbasiyadkori})let $\{\epsilon_t\}_{t=1}^\infty$ be a real-valued stochastic process with corresponding filtration $\{\mathcal{F}_t\}_{t=0}^\infty$ where $\epsilon_t|\mathcal{F}_{t-1}$ is zero-mean and $\sigma$-sub-Gaussian, and let $\{\bm{\phi}_t\}_{t=1}^\infty$ be an $\mathbb{R}^d$-valued stochastic process where $\bm{\phi}_t$ is $\mathcal{F}_{t-1}$ measurable. $\Lambda_0\in\mathbb{R}^{d\times d}$ is a positive definite matrix and $\Lambda_t=\Lambda_0+\sum_{s=1}^t\bm{\phi}_s\bm{\phi}_s^\top$. Then $\forall \delta\in[0,1]$, with probability at least $1-\delta$ we have $\forall t\geq0$,
\begin{align}
\Bigg\Vert\sum_{s=1}^t\bm{\phi}_s\epsilon_s\Bigg\Vert^2_{\Lambda_t^{-1}}\leq2\sigma^2\log\Bigg[\frac{\det(\Lambda_t)^{1/2}\det(\Lambda_0)^{-1/2}}{\delta}\Bigg].\notag
\end{align}
\end{lemma}

\begin{lemma}\label{Lemma:covering_ball}
(Covering Number of Euclidean Ball \citep[Lemma D.5]{jin2020}) For all $\epsilon>0$, the $\epsilon$-covering number of a Euclidean ball in $\RR^d$ with radius $R>0$ is upper-bounded by $(1+2R/\epsilon)^d$.
\end{lemma}

\begin{lemma}\label{Lemma:covering_interval}
(Covering Number of Interval \citep{vershynin2018}) For all $\epsilon>0$ and real numbers $a<b$, the $\epsilon$-covering number of the closed interval $[a,b]$ is upper-bounded by $3(b-a)/\epsilon$.
\end{lemma}

\begin{lemma}\label{Lemma:vec_norm}
$\forall k,h, \Vert\bm{\theta}_h+\bm{\nu}_h^{\rho,k}\Vert_2\leq3H\sqrt{\frac{dk}{\lambda}}$.
\end{lemma}

\begin{lemma}\label{Lemma:union_bound}
Let $\{x_\tau\}_{\tau=1}^\infty$ be a stochastic process on $S$ with filtration $\{\mathcal{F}_\tau\}_{\tau=0}^\infty$, and let $\{\bm{\phi}_\tau\}_{\tau=1}^\infty$ be an $\mathbb{R}^d$-valued stochastic process with filtration $\bm{\phi}_\tau\in\mathcal{F}_{\tau-1}$ and $\Vert\bm{\phi}_\tau\Vert\leq1$. Let $\Lambda_k=\lambda I+\sum_{\tau=1}^{k-1}\bm{\phi}_\tau\bm{\phi}_\tau^\top$. Then for any $\delta>0$, with probability at least $1-\delta$ we have $\forall k\geq0,\alpha\in[0,H],\tilde{V}\in\mathcal{V}$ such that $\sup_x|V(x)|\leq H$ that
\begin{align}
\Bigg\Vert\sum_{\tau=1}^k\bm{\phi}_\tau\Big[[\tilde{V}(x_\tau)]_\alpha-\EE[[\tilde{V}(x_\tau)]_\alpha|\mathcal{F}_{\tau-1}]\Big]\Bigg\Vert^2_{(\Lambda_h^k)^{-1}}\leq&8H^2\Bigg[\frac{d}{2}\log\frac{k+\lambda}{\lambda}+\log\frac{\mathcal{N}_{\epsilon_1}}{\delta}+\log\frac{\mathcal{N}_{\epsilon_2}}{\delta}\Bigg]\notag\\
&\qquad\qquad+\frac{16k^2\epsilon_1^2}{\lambda}+\frac{8k^2\epsilon_2^2}{\lambda}.\notag
\end{align}
\end{lemma}

\begin{lemma}\label{Lemma:eps_covering}
Let $\mathcal{V}$ be a class of functions $S\rightarrow\mathbb{R}$ with the form
\begin{align}
V(\cdot)=\min\Bigg\{\max_\pi\Bigg\{\innerA{\pi(\cdot,a),w^\top\bm{\phi}(\cdot,a)+\beta\sum_{i=1}^d\Vert\phi_i(\cdot,a)\mathbf{1}_i\Vert_{\Lambda^{-1}}}-\frac{1}{\eta}D_{KL}[\pi\Vert\pi_0]\Bigg\},H\Bigg\}.\notag
\end{align}
where $\Vert w\Vert\leq L, \beta\in[0,B], \lambda_{\min}(\Lambda)\geq\lambda$. If $\forall s,a, \Vert\bm{\phi}(s,a)\Vert\leq1$, let $\mathcal{N}_\epsilon$ be the $\epsilon$-covering number of $\mathcal{V}$ with respect to the distance $dist(V_1,V_2)=\sup_x|V_1(x)-V_2(x)|$. Then
\begin{align}
\log\mathcal{N}_\epsilon\leq d\log(1+4L|\mathcal{A}|/\epsilon)+d^2\log\bigg[1+\frac{8\sqrt{d}B^2\vA^2}{\lambda\epsilon^2}\bigg].\notag
\end{align}
\end{lemma}
Combining \Cref{Lemma:covering_interval,Lemma:vec_norm,Lemma:union_bound}, and \Cref{Lemma:eps_covering}, set $\epsilon_1=\epsilon_2=\epsilon$ and we have
\begin{align}
&\Bigg\Vert\sum_{\tau=1}^{k-1}\bm{\phi}_h^\tau\Big[[\tilde{V}_{h+1}^{k,\rho}(s_{h+1}^\tau)]_\alpha-[\PP_h^0[\tilde{V}_{h+1}^{k,\rho}]_\alpha](s_{h+1}^\tau,a_{h+1}^\tau)\Big]\Bigg\Vert^2_{(\Lambda_h^k)^{-1}}\notag\\
\leq&8H^2\bigg[\frac{d}{2}\log\frac{k+\lambda}{\lambda}+\log\frac{\mathcal{N}_{\epsilon_1}}{\delta}+\log\frac{\mathcal{N}_{\epsilon_2}}{\delta}\bigg]+\frac{16k^2\epsilon^2}{\lambda}+\frac{8k^2\epsilon^2}{\lambda}\notag\\
\leq&8H^2\Bigg[\frac{d}{2}\log\frac{k+\lambda}{\lambda}+d\log\bigg(1+\frac{4L\vA}{\epsilon}\bigg)+d^2\log\bigg(1+\frac{8\sqrt{d}\beta^2\vA^2}{\lambda\epsilon^2}\bigg)+\log\frac{3H}{\delta^2\epsilon}\Bigg]+\frac{24k^2\epsilon^2}{\lambda}\notag\\
\leq&8H^2\Bigg[\frac{d}{2}\log\frac{k+\lambda}{\lambda}+d\log\bigg(1+\frac{12H\sqrt{dk}\vA}{\sqrt{\lambda}\epsilon}\bigg)+d^2\log\bigg(1+\frac{8\sqrt{d}\beta^2\vA^2}{\lambda\epsilon^2}\bigg)+\log\frac{3H}{\delta^2\epsilon}\Bigg]+\frac{24k^2\epsilon^2}{\lambda}.\notag
\end{align}
Setting $\lambda=1,\beta=c_\beta dH\iota,\epsilon=dH/k$, we get
\begin{align}
\Bigg\Vert\sum_{\tau=1}^{k-1}\bm{\phi}_h^\tau\Big[[\tilde{V}_{h+1}^{k,\rho}(s_{h+1}^\tau)]_\alpha-[\PP_h^0[\tilde{V}_{h+1}^{k,\rho}]_\alpha](s_{h+1}^\tau,a_{h+1}^\tau)\Big]\Bigg\Vert^2_{(\Lambda_h^k)^{-1}}\leq&Cd^2H^2\log\frac{3(c_\beta+1)dKH\vA}{\delta}.\notag
\end{align}
\end{proof}

\subsection{Proof of \Cref{lemma:C1_1}}

\begin{proof}

To facilitate the proof, we will rely on the following lemmas:
\begin{lemma}\label{Lemma:vec_norm1}
$\forall k,h, \Vert\bm{\theta}_h+\bm{\nu}_h^{\sigma,k}\Vert_2\leq2H\sqrt{\frac{dk}{\lambda}}$.
\end{lemma}

\begin{lemma}\label{Lemma:union_bound1}
Let $\{x_\tau\}_{\tau=1}^\infty$ be a stochastic process on $S$ with filtration $\{\mathcal{F}_\tau\}_{\tau=0}^\infty$, and let $\{\bm{\phi}_\tau\}_{\tau=1}^\infty$ be an $\mathbb{R}^d$-valued stochastic process with filtration $\bm{\phi}_\tau\in\mathcal{F}_{\tau-1}$ and $\Vert\bm{\phi}_\tau\Vert\leq1$. Let $\Lambda_k=\lambda I+\sum_{\tau=1}^{k-1}\bm{\phi}_\tau\bm{\phi}_\tau^\top$. Then for any $\delta>0$, with probability at least $1-\delta$ we have $\forall k\geq0,\sigma\in[0,H],\tilde{V}\in\mathcal{V}$ such that $\sup_x|V(x)|\leq H$ that
\begin{align}
\Bigg\Vert\sum_{\tau=1}^k\bm{\phi}_\tau\Big[[\tilde{V}(x_\tau)]_\sigma-\EE[[\tilde{V}(x_\tau)]_\sigma|\mathcal{F}_{\tau-1}]\Big]\Bigg\Vert^2_{(\Lambda_h^k)^{-1}}\leq8H^2\Bigg[\frac{d}{2}\log\frac{k+\lambda}{\lambda}+\log\frac{\mathcal{N}_{\epsilon}}{\delta}\Bigg]+\frac{8k^2\epsilon^2}{\lambda}.\notag
\end{align}
\end{lemma}

Combining \Cref{Lemma:vec_norm1}, \Cref{Lemma:union_bound1}, and \Cref{Lemma:eps_covering}, we have
\begin{align}
&\Bigg\Vert\sum_{\tau=1}^{k-1}\bm{\phi}_h^\tau\Big[[\tilde{V}_{h+1}^{k,\sigma}(s_{h+1}^\tau)]_\sigma-[\PP_h^0[\tilde{V}_{h+1}^{k,\sigma}]_\sigma](s_{h+1}^\tau,a_{h+1}^\tau)\Big]\Bigg\Vert^2_{(\Lambda_h^k)^{-1}}\notag\\
\leq&8H^2\bigg[\frac{d}{2}\log\frac{k+\lambda}{\lambda}+\log\frac{\mathcal{N}_{\epsilon}}{\delta}\bigg]+\frac{8k^2\epsilon^2}{\lambda}\notag\\
\leq&8H^2\Bigg[\frac{d}{2}\log\frac{k+\lambda}{\lambda}+\log\frac{1}{\delta}+d\log\bigg(1+\frac{4L\vA}{\epsilon}\bigg)+d^2\log\bigg(1+\frac{8\sqrt{d}\beta^2\vA}{\lambda\epsilon^2}\bigg)\Bigg]+\frac{8k^2\epsilon^2}{\lambda}\notag\\
\leq&8H^2\Bigg[\frac{d}{2}\log\frac{k+\lambda}{\lambda}+\log\frac{1}{\delta}+d\log\bigg(1+\frac{8H\sqrt{dk}\vA}{\sqrt{\lambda}\epsilon}\bigg)+d^2\log\bigg(1+\frac{8\sqrt{d}\beta^2\vA}{\lambda\epsilon^2}\bigg)\Bigg]+\frac{8k^2\epsilon^2}{\lambda}.\notag
\end{align}
Setting $\lambda=1,\beta=c_\beta dH\iota,\epsilon=dH/k$, we get
\begin{align}
\Bigg\Vert\sum_{\tau=1}^{k-1}\bm{\phi}_h^\tau\Big[[\tilde{V}_{h+1}^{k,\sigma}(s_{h+1}^\tau)]_\sigma-[\PP_h^0[\tilde{V}_{h+1}^{k,\sigma}]_\sigma](s_{h+1}^\tau,a_{h+1}^\tau)\Big]\Bigg\Vert^2_{(\Lambda_h^k)^{-1}}\leq&Cd^2H^2\log\frac{3(c_\beta+1)dKH\vA}{\delta}.\notag
\end{align}
\end{proof}

\section{Proof of Additional Lemmas}

\subsection{Proof of \Cref{Lemma:vec_norm}}
\begin{proof}

Denote $\alpha_i=\arg\max_{\alpha\in[0,H]}\big\{z_{h,i}^k(\alpha)-\rho\alpha\big\}$. Then $\forall \bm{x}\in\mathbb{R}^d$, we have
\begin{align}
\bigg|\inner{\bm{x},\bm{\theta}_h+\bm{\nu}_h^{\rho,k}}\bigg|&=\Bigg|\inner{\bm{x},\bm{\theta}_h}+\inner{\bm{x},\bigg[
\max_{\alpha\in[0,H]}\{z_{h,i}^k(\alpha)-\rho\alpha\}\bigg]_{i\in[d]}}\Bigg|\notag\\
&\leq\big|\inner{\bm{x},\bm{\theta}_h}\big|+\Big|\inner{\bm{x},\big[z_{h,i}^k(\alpha_i)-\rho\alpha_i\big]_{i\in[d]}}\Big|\notag\\
&\leq\big|\inner{\bm{x},\bm{\theta}_h}\big|+\Big|\inner{\bm{x},\big[z_{h,i}^k(\alpha_i)\big]_{i\in[d]}}\Big|+\Big|\inner{\bm{x},\big[\rho\alpha_i\big]_{i\in[d]}}\Big|.\notag
\end{align}
Because d-rectangular DRMDP assumes $\Vert\bm{\theta}_h\Vert_2\leq\sqrt{d}$, for the first term we have $|\inner{\bm{x},\bm{\theta}_h}|\leq\Vert \bm{x}\Vert_2\Vert\bm{\theta}_h\Vert_2=\sqrt{d}\Vert \bm{x}\Vert_2$. Since $\alpha_i\in[0,H]$ and $\rho\leq1$, for the last term we have $\big|\la\bm{x},\big[\rho\alpha_i\big]_{i\in[d]}\ra\big|\leq H\sqrt{d}\Vert \bm{x}\Vert_2$. For the middle term, we have
\begin{align}
\bigg|\inner{\bm{x},\big[z_{h,i}^k(\alpha_i)\big]_{i\in[d]}}\bigg|&=\Bigg|\inner{\bm{x},\Bigg[\Bigg((\Lambda_h^k)^{-1}\bigg[\sum_{\tau=1}^{k-1}\bm{\phi}_h^\tau[\tilde{V}_{h+1}^{k,\rho}(s_{h+1}^\tau)]_{\alpha_i}\bigg]\Bigg)_i\Bigg]_{i\in[d]}}\Bigg|\notag\\
&\leq H\Bigg|\bm{x}^\top(\Lambda_h^k)^{-1}\sum_{\tau=1}^{k-1}\bm{\phi}_h^\tau\Bigg|\notag\\
&\leq H\sqrt{\sum_{\tau=1}^{k-1}\bm{x}^\top(\Lambda_h^k)^{-1}\bm{x}}\sqrt{\sum_{\tau=1}^{k-1}{\bm{\phi}_h^\tau}^\top(\Lambda_h^k)^{-1}\bm{\phi}_h^\tau}\notag\\
&\leq H\sqrt{d}\sqrt{\sum_{\tau=1}^{k-1}\bm{x}^\top(\Lambda_h^k)^{-1}\bm{x}} \label{eq:use_jin_d1}\\
&\leq H\sqrt{dk/\lambda}\Vert \bm{x}\Vert_2,\label{eq:bound_quadratic_form}
\end{align}
where \eqref{eq:use_jin_d1} is due to \Cref{Lemma:sum_quadratic_form}, and \eqref{eq:bound_quadratic_form} uses the fact that $\Lambda_h^k\succeq\lambda I$.
Combining the three terms, we have $\forall \bm{x}\in\mathbb{R}^d, \big|\la\bm{x},\bm{\theta}_h+\bm{\nu}_h^{\rho,k}\ra\big|\leq3H\sqrt{dk/\lambda}\Vert \bm{x}\Vert_2$, which implies $\Vert\bm{\theta}_h+\bm{\nu}_h^{\rho,k}\Vert_2\leq3H\sqrt{dk/\lambda}\Vert x\Vert_2$.
\end{proof}

\subsection{Proof of \Cref{Lemma:union_bound}}

\begin{proof}
Construct a $\epsilon_1$-covering of $\mathcal{V}$ and a $\epsilon_2$-covering of $[0,H]$ such that $\forall \tilde{V}\in\mathcal{V},\alpha\in[0,H]$, there exists $\tilde{V}_0,\alpha_0$ in the covering with $\Vert\tilde{V}-\tilde{V}_0\Vert_\infty\leq\epsilon_1, |\alpha-\alpha_0|\leq\epsilon_2$. Then
\begin{align}
&\Bigg\Vert\sum_{\tau=1}^k\bm{\phi}_\tau\Big[[\tilde{V}(x_\tau)]_\alpha-\EE[[\tilde{V}(x_\tau)]_\alpha|\mathcal{F}_{\tau-1}]\Big]\Bigg\Vert^2_{(\Lambda_h^k)^{-1}}\notag\\
=&\Bigg\Vert\sum_{\tau=1}^k\bm{\phi}_\tau\Big[[\tilde{V}_0(x_\tau)]_\alpha-\EE[[\tilde{V}_0(x_\tau)]_\alpha|\mathcal{F}_{\tau-1}]\Big]\notag\\
&\qquad\qquad+\sum_{\tau=1}^k\bm{\phi}_\tau\Big[[\tilde{V}(x_\tau)]_\alpha-[\tilde{V}_0(x_\tau)]_\alpha-\EE[[\tilde{V}_0(x_\tau)]_\alpha-[\tilde{V}_0(x_\tau)]_\alpha|\mathcal{F}_{\tau-1}]\Big]\Bigg\Vert^2_{(\Lambda_h^k)^{-1}}\notag\\
\leq&2\Bigg\Vert\sum_{\tau=1}^k\bm{\phi}_\tau\Big[[\tilde{V}_0(x_\tau)]_\alpha-\EE[[\tilde{V}_0(x_\tau)]_\alpha|\mathcal{F}_{\tau-1}]\Big]\Bigg\Vert^2_{(\Lambda_h^k)^{-1}}\notag\\
&\qquad\qquad +2\Bigg\Vert\sum_{\tau=1}^k\bm{\phi}_\tau\Big[[\tilde{V}(x_\tau)]_\alpha-[\tilde{V}_0(x_\tau)]_\alpha-\EE[[\tilde{V}_0(x_\tau)]_\alpha-[\tilde{V}_0(x_\tau)]_\alpha|\mathcal{F}_{\tau-1}]\Big]\Bigg\Vert^2_{(\Lambda_h^k)^{-1}}\notag\\
=&2\Bigg\Vert\sum_{\tau=1}^k\bm{\phi}_\tau\Big[[\tilde{V}_0(x_\tau)]_{\alpha_0}-\EE[[\tilde{V}_0(x_\tau)]_{\alpha_0}|\mathcal{F}_{\tau-1}]\Big]
\notag\\
&\qquad\qquad +\sum_{\tau=1}^k\bm{\phi}_\tau\Big[[\tilde{V}_0(x_\tau)]_{\alpha}-[\tilde{V}_0(x_\tau)]_{\alpha_0}-\EE[[\tilde{V}_0(x_\tau)]_{\alpha}-[\tilde{V}_0(x_\tau)]_{\alpha_0}|\mathcal{F}_{\tau-1}]\Big]\Bigg\Vert^2_{(\Lambda_h^k)^{-1}}\notag\\
&\qquad\qquad+2\Bigg\Vert\sum_{\tau=1}^k\bm{\phi}_\tau\Big[[\tilde{V}(x_\tau)]_\alpha-[\tilde{V}_0(x_\tau)]_\alpha-\EE[[\tilde{V}_0(x_\tau)]_\alpha-[\tilde{V}_0(x_\tau)]_\alpha|\mathcal{F}_{\tau-1}]\Big]\Bigg\Vert^2_{(\Lambda_h^k)^{-1}}\notag\\
=&4\Bigg\Vert\sum_{\tau=1}^k\bm{\phi}_\tau\Big[[\tilde{V}_0(x_\tau)]_{\alpha_0}-\EE[[\tilde{V}_0(x_\tau)]_{\alpha_0}|\mathcal{F}_{\tau-1}]\Big]\Bigg\Vert^2_{(\Lambda_h^k)^{-1}}\notag\\
&\qquad\qquad+4\Bigg\Vert\sum_{\tau=1}^k\bm{\phi}_\tau\Big[[\tilde{V}_0(x_\tau)]_{\alpha}-[\tilde{V}_0(x_\tau)]_{\alpha_0}-\EE[[\tilde{V}_0(x_\tau)]_{\alpha}-[\tilde{V}_0(x_\tau)]_{\alpha_0}|\mathcal{F}_{\tau-1}]\Big]\Bigg\Vert^2_{(\Lambda_h^k)^{-1}}\notag\\
&\qquad\qquad+2\Bigg\Vert\sum_{\tau=1}^k\bm{\phi}_\tau\Big[[\tilde{V}(x_\tau)]_\alpha-[\tilde{V}_0(x_\tau)]_\alpha-\EE[[\tilde{V}_0(x_\tau)]_\alpha-[\tilde{V}_0(x_\tau)]_\alpha|\mathcal{F}_{\tau-1}]\Big]\Bigg\Vert^2_{(\Lambda_h^k)^{-1}}.\notag
\end{align}
For the first term, since $[\tilde{V}_0(x_\tau)]_{\alpha_0}-\EE[[\tilde{V}_0(x_\tau)]_{\alpha_0}|\mathcal{F}_{\tau-1}]$ is zero-mean and $H$-sub-Gaussian, by \Cref{Lemma:abbasi_yadkori} and union-bound we have
\begin{align}
4\Bigg\Vert\sum_{\tau=1}^k\bm{\phi}_\tau\Big[[\tilde{V}_0(x_\tau)]_{\alpha_0}-\EE[[\tilde{V}_0(x_\tau)]_{\alpha_0}|\mathcal{F}_{\tau-1}]\Big]\Bigg\Vert^2_{(\Lambda_h^k)^{-1}}
\leq&8H^2\log\Bigg[\frac{\mathcal{N}_{\epsilon_1}\mathcal{N}_{\epsilon_2}}{\delta}\frac{\det(\Lambda_h^k)^{1/2}}{\det(\Lambda_h^0)^{1/2}}\Bigg]\notag\\
\leq&8H^2\Bigg[\log\frac{\mathcal{N}_{\epsilon_1}}{\delta}+\log\frac{\mathcal{N}_{\epsilon_2}}{\delta}+\log\sqrt{\frac{(\lambda+k)^d}{\lambda^d}}\Bigg]\notag\\
=&8H^2\Bigg[\log\frac{\mathcal{N}_{\epsilon_1}}{\delta}+\log\frac{\mathcal{N}_{\epsilon_2}}{\delta}+\frac{d}{2}\log\frac{\lambda+k}{\lambda}\Bigg].\notag
\end{align}
For the second term, we have
\begin{align}
4\Bigg\Vert\sum_{\tau=1}^k\bm{\phi}_\tau\Big[[\tilde{V}_0(x_\tau)]_{\alpha}-[\tilde{V}_0(x_\tau)]_{\alpha_0}-\EE[[\tilde{V}_0(x_\tau)]_{\alpha}-[\tilde{V}_0(x_\tau)]_{\alpha_0}|\mathcal{F}_{\tau-1}]\Big]\Bigg\Vert^2_{(\Lambda_h^k)^{-1}}&\leq4\Bigg\Vert\sum_{\tau=1}^k\bm{\phi}_\tau|2\epsilon_1|\Bigg\Vert^2_{(\Lambda_h^k)^{-1}} \notag\\
&\leq\frac{16k^2\epsilon_1^2}{\lambda} \notag  
\end{align}
and
\begin{align}
2\Bigg\Vert\sum_{\tau=1}^k\bm{\phi}_\tau\Big[[\tilde{V}(x_\tau)]_\alpha-[\tilde{V}_0(x_\tau)]_\alpha-\EE[[\tilde{V}_0(x_\tau)]_\alpha-[\tilde{V}_0(x_\tau)]_\alpha|\mathcal{F}_{\tau-1}]\Big]\Bigg\Vert^2_{(\Lambda_h^k)^{-1}}&\leq2\Bigg\Vert\sum_{\tau=1}^k\bm{\phi}_\tau|2\epsilon_2|\Bigg\Vert^2_{(\Lambda_h^k)^{-1}} \notag\\
&\leq\frac{8k^2\epsilon_2^2}{\lambda}.\notag
\end{align}
This finishes the proof.
\end{proof}

\subsection{Proof of \Cref{Lemma:eps_covering}}

\begin{proof}
Let arbitrary $\tilde{V}_1,\tilde{V}_2\in\mathcal{V}$ be parameterized by $w_1,A_1=\beta\Lambda_1^{-1}$ and $w_2,A_2=\beta\Lambda_2^{-1}$. Then
\begin{align}
&\Vert\tilde{V}_1-\tilde{V}_2\Vert_\infty\notag\\
\leq&\sup_s\Bigg|\sup_\pi\Bigg(\innerA{\pi(a|s),w_1^\top\bm{\phi}(s,a)+\sum_{i=1}^d\Vert\phi_i(s,a)\mathbf{1}_i\Vert_{A_1}}-\frac{1}{\eta}D_{KL}[\pi\Vert\pi_0]\Bigg)\notag\\
-&\sup_\pi\Bigg(\innerA{\pi(a|s),w_2^\top\bm{\phi}(s,a)+\sum_{i=1}^d\Vert\phi_i(s,a)\mathbf{1}_i\Vert_{A_2}}-\frac{1}{\eta}D_{KL}[\pi\Vert\pi_0]\Bigg)\Bigg|\notag\\
\leq&\sup_{s,\pi}\Bigg|\innerA{\pi(a|s),w_1^\top\bm{\phi}(s,a)+\sum_{i=1}^d\Vert\phi_i(s,a)\mathbf{1}_i\Vert_{A_1}}-\innerA{\pi(\cdot|a),w_2^\top\bm{\phi}(s,a)+\sum_{i=1}^d\Vert\phi_i(s,a)\mathbf{1}_i\Vert_{A_2}}\Bigg|\notag\\
=&\sup_{s,\pi}\Bigg|\innerA{\pi(a|s),(w_1-w_2)^\top\bm{\phi}(s,a)+\sum_{i=1}^d\Big[\Vert\phi_i(s,a)\mathbf{1}_i\Vert_{A_1}-\Vert\phi_i(s,a)\mathbf{1}_i\Vert_{A_2}\Big]}\Bigg|\notag\\
\leq&\sup_{s,\pi}\Bigg|\innerA{\pi(a|s),(w_1-w_2)^\top\bm{\phi}(s,a)}\Bigg|+\sup_{s,\pi}\Bigg|\innerA{\pi(a|s),\sum_{i=1}^d\Big[\Vert\phi_i(s,a)\mathbf{1}_i\Vert_{A_1}-\Vert\phi_i(s,a)\mathbf{1}_i\Vert_{A_2}\Big]}\Bigg|.\label{eq:vdiff}
\end{align}
Without loss of generality, assume that $\Vert\phi_i(s,a)\mathbf{1}_i\Vert_{A_1}\geq\Vert\phi_i(s,a)\mathbf{1}_i\Vert_{A_2}$. We observe that
\begin{align}
\big|\Vert\phi_i(s,a)\mathbf{1}_i\Vert_{A_1}-\Vert\phi_i(s,a)\mathbf{1}_i\Vert_{A_2}\big|&=\big|\sqrt{\phi_i(s,a)\mathbf{1}_iA_1\phi_i(s,a)\mathbf{1}_i}-\sqrt{\phi_i(s,a)\mathbf{1}_iA_2\phi_i(s,a)\mathbf{1}_i}\big|\notag\\
&\leq\sqrt{\big|\phi_i(s,a)\mathbf{1}_iA_1\phi_i(s,a)\mathbf{1}_i-\phi_i(s,a)\mathbf{1}_iA_2\phi_i(s,a)\mathbf{1}_i\big|}\notag\\
&=\sqrt{\phi_i(s,a)\mathbf{1}_i(A_1-A_2)\phi_i(s,a)\mathbf{1}_i}\notag\\
&=\Vert\phi_i(s,a)\mathbf{1}_i\Vert_{A_1-A_2}.\notag
\end{align}
Therefore for all $\pi$, we have
\begin{align}
\Bigg|\sum_{i=1}^d\Big[\Vert\phi_i(s,a)\mathbf{1}_i\Vert_{A_1}-\Vert\phi_i(s,a)\mathbf{1}_i\Vert_{A_2}\Big]\Bigg|&\leq\sum_{i=1}^d\Big|\Vert\phi_i(s,a)\mathbf{1}_i\Vert_{A_1}-\Vert\phi_i(s,a)\mathbf{1}_i\Vert_{A_2}\Big|\notag\\
&\leq\sum_{i=1}^d\Vert\phi_i(s,a)\mathbf{1}_i\Vert_{A_1-A_2}.\notag
\end{align}
Consequently,
\begin{align}
\Bigg|\innerA{\pi(a|s),\sum_{i=1}^d\Big[\Vert\phi_i(s,a)\mathbf{1}_i\Vert_{A_1}-\Vert\phi_i(s,a)\mathbf{1}_i\Vert_{A_2}\Big]}\Bigg|
\leq&\innerA{\pi(a|s),\Vert\phi_i(s,a)\mathbf{1}_i\Vert_{A_1-A_2}}.\notag
\end{align}
Plugging this into \eqref{eq:vdiff}, we have
\begin{align}
\Vert\tilde{V}_1-\tilde{V}_2\Vert_\infty&\leq\sup_{s,\pi}\Big|\innerA{\pi(a|s),(w_1-w_2)^\top\bm{\phi}(s,a)}\Big|+\sup_{s,\pi}\innerA{\pi(a|s),\Vert\phi_i(s,a)\mathbf{1}_i\Vert_{A_1-A_2}}\notag\\
&=\sup_{s,\pi}\Bigg|\int_{a\in\mathcal{A}}(w_1-w_2)^\top\bm{\phi}(s,a)da\Bigg|+\sup_{s,\pi}\int_{a\in\mathcal{A}}\sum_{i=1}^d\Vert\bm{\phi}(a)_i\mathbf{1}_i\Vert_{A_1-A_2}da\notag\\
&\leq\vA\big(\Vert w_1-w_2\Vert_2+\sqrt{\Vert A_1-A_2\Vert_F}\big).\notag
\end{align}
Note that $\vA$ denotes the measure of $\mathcal{A}$. We assume the action space is a compact set, e.g. a box, so $\vA$ is finite. We can further normalize $\mathcal{A}$ so that it is a subset of the unit box, and thus bound $\vA\leq1$. Let $C_w$ be a $\epsilon/2\vA$ cover of $\{w\in\mathbb{R}^d:\Vert w\Vert_2\leq L\}$ with respect to the $L_2$ norm and $C_A$ be a $\epsilon^2/4\vA^2$ cover of $\{A\in\mathbb{R}^{d\times d}:\Vert A\Vert_F\leq \frac{1} {\lambda}\sqrt{d}B^2\}$ with respect to the Frobenius norm. By \Cref{Lemma:covering_ball}, we have that
\begin{align}
|C_w|\leq\Bigg(1+\frac{4L\vA}{\epsilon}\Bigg)^d,\quad|C_A|\leq\Bigg(1+\frac{8\sqrt{d}B^2\vA^2}{\lambda\epsilon^2}\Bigg)^{d^2}.\notag
\end{align}
Then $\forall \tilde{V}_1\in\mathcal{V}$, there exists $\tilde{V}_2\in\mathcal{V}$ parametrized by $w_2\in C_w,A_2\in C_A$ such that $\Vert\tilde{V}_1-\tilde{V}_2\Vert\leq\epsilon$. Therefore
\begin{align}
\log\mathcal{N}_\epsilon\leq\log(|C_w\Vert C_A|)\leq d\log\Bigg(1+\frac{4L\vA}{\epsilon}\Bigg)+d^2\log\Bigg(1+\frac{8\sqrt{d}B^2\vA^2}{\lambda\epsilon^2}\Bigg).\notag
\end{align}
\end{proof}

\subsection{Proof of \Cref{Lemma:vec_norm1}}

\begin{proof}
Let $\bm{x}\in\RR^d$ be an arbitrary vector. Because d-rectangular RRMDP assumes $\Vert\bm{\theta}_h\Vert_2\leq\sqrt{d}$, we have $|\inner{\bm{x},\bm{\theta}_h}|\leq\Vert \bm{x}\Vert_2\Vert\bm{\theta}_h\Vert_2=\sqrt{d}\Vert \bm{x}\Vert_2$. On the other hand,
\begin{align}
\bigg|\inner{\bm{x},\bm{\nu}_{h}^{\sigma,k}}\bigg|&=\Bigg|\inner{\bm{x},\Bigg[\Bigg((\Lambda_h^k)^{-1}\bigg[\sum_{\tau=1}^{k-1}\bm{\phi}_h^\tau[\tilde{V}_{h+1}^{k,\sigma}(s_{h+1}^\tau)]_{\sigma}\bigg]\Bigg)_i\Bigg]_{i\in[d]}}\Bigg|\notag\\
&\leq H\Bigg|\bm{x}^\top(\Lambda_h^k)^{-1}\sum_{\tau=1}^{k-1}\bm{\phi}_h^\tau\Bigg|\notag\\
&\leq H\sqrt{d}\sqrt{\sum_{\tau=1}^{k-1}x^\top(\Lambda_h^k)^{-1}\bm{x}} \label{eq:use_jin_d1_1}\\
&\leq H\sqrt{dk/\lambda}\Vert \bm{x}\Vert_2,\label{eq:bound_quadratic_form_1}
\end{align}
where \eqref{eq:use_jin_d1_1} is due to \Cref{Lemma:sum_quadratic_form}, and \eqref{eq:bound_quadratic_form_1} uses the fact that $\Lambda_h^k\succeq\lambda I$.
Combining the two terms, we have $\forall \bm{x}\in\mathbb{R}^d, \bigg|\inner{\bm{x},\bm{\theta}_h+\bm{\nu}_h^{\sigma,k}}\bigg|\leq2H\sqrt{dk/\lambda}\Vert \bm{x}\Vert_2$, which implies $\Vert\bm{\theta}_h+\bm{\nu}_h^{\sigma,k}\Vert_2\leq2H\sqrt{dk/\lambda}\Vert x\Vert_2$.
\end{proof}

\subsection{Proof of \Cref{Lemma:union_bound1}}

\begin{proof}
Construct a $\epsilon$-covering of $\mathcal{V}$ such that $\forall \tilde{V}\in\mathcal{V}$, there exists $\tilde{V}_0$ in the covering with $\Vert\tilde{V}-\tilde{V}_0\Vert_\infty\leq\epsilon$. Then
\begin{align}
&\Bigg\Vert\sum_{\tau=1}^k\bm{\phi}_\tau\Big[[\tilde{V}(x_\tau)]_\sigma-\EE[[\tilde{V}(x_\tau)]_\sigma|\mathcal{F}_{\tau-1}]\Big]\Bigg\Vert^2_{(\Lambda_h^k)^{-1}}\notag\\
=&\Bigg\Vert\sum_{\tau=1}^k\bm{\phi}_\tau\Big[[\tilde{V}_0(x_\tau)]_\sigma-\EE[[\tilde{V}_0(x_\tau)]_\sigma|\mathcal{F}_{\tau-1}]\Big]\notag\\
&\qquad\qquad+\sum_{\tau=1}^k\bm{\phi}_\tau\Big[[\tilde{V}(x_\tau)]_\sigma-[\tilde{V}_0(x_\tau)]_\sigma-\EE[[\tilde{V}_0(x_\tau)]_\sigma-[\tilde{V}_0(x_\tau)]_\sigma|\mathcal{F}_{\tau-1}]\Big]\Bigg\Vert^2_{(\Lambda_h^k)^{-1}}\notag\\
\leq&2\Bigg\Vert\sum_{\tau=1}^k\bm{\phi}_\tau\Big[[\tilde{V}_0(x_\tau)]_\sigma-\EE[[\tilde{V}_0(x_\tau)]_\sigma|\mathcal{F}_{\tau-1}]\Big]\Bigg\Vert^2_{(\Lambda_h^k)^{-1}}\notag\\
&\qquad\qquad+2\Bigg\Vert\sum_{\tau=1}^k\bm{\phi}_\tau\Big[[\tilde{V}(x_\tau)]_\sigma-[\tilde{V}_0(x_\tau)]_\sigma-\EE[[\tilde{V}_0(x_\tau)]_\sigma-[\tilde{V}_0(x_\tau)]_\sigma|\mathcal{F}_{\tau-1}]\Big]\Bigg\Vert^2_{(\Lambda_h^k)^{-1}}.\notag
\end{align}
For the first term, since $[\tilde{V}_0(x_\tau)]_{\alpha_0}-\EE[[\tilde{V}_0(x_\tau)]_{\alpha_0}|\mathcal{F}_{\tau-1}]$ is zero-mean and $H$-sub-Gaussian, by \Cref{Lemma:abbasi_yadkori} and union-bound we have
\begin{align}
4\Bigg\Vert\sum_{\tau=1}^k\bm{\phi}_\tau\Big[[\tilde{V}_0(x_\tau)]_{\alpha_0}-\EE[[\tilde{V}_0(x_\tau)]_{\alpha_0}|\mathcal{F}_{\tau-1}]\Big]\Bigg\Vert^2_{(\Lambda_h^k)^{-1}}
\leq&8H^2\log\Bigg[\frac{\mathcal{N}_{\epsilon}}{\delta}\frac{\det(\Lambda_h^k)^{1/2}}{\det(\Lambda_h^0)^{1/2}}\Bigg]\notag\\
\leq&8H^2\Bigg[\log\frac{\mathcal{N}_{\epsilon}}{\delta}+\log\sqrt{\frac{(\lambda+k)^d}{\lambda^d}}\Bigg]\notag\\
=&8H^2\Bigg[\log\frac{\mathcal{N}_{\epsilon}}{\delta}+\frac{d}{2}\log\frac{\lambda+k}{\lambda}\Bigg].\notag
\end{align}
For the second term, we have
\begin{align}
2\Bigg\Vert\sum_{\tau=1}^k\bm{\phi}_\tau\Big[[\tilde{V}(x_\tau)]_\sigma-[\tilde{V}_0(x_\tau)]_\sigma-\EE[[\tilde{V}_0(x_\tau)]_\sigma-[\tilde{V}_0(x_\tau)]_\sigma|\mathcal{F}_{\tau-1}]\Big]\Bigg\Vert^2_{(\Lambda_h^k)^{-1}}&\leq2\Bigg\Vert\sum_{\tau=1}^k\bm{\phi}_\tau|2\epsilon|\Bigg\Vert^2_{(\Lambda_h^k)^{-1}}\notag\\
&\leq\frac{8k^2\epsilon^2}{\lambda}.\notag
\end{align}
This finishes the proof.
\end{proof}

\bibliography{reference}
\bibliographystyle{ims}
\end{document}